\documentclass{article}

\usepackage{amsthm}
\usepackage{amsmath}
\usepackage{amssymb}
\usepackage{authblk}
\usepackage{breqn}
\usepackage{booktabs}
\usepackage{graphicx}
\usepackage{caption}
\usepackage{multicol}
\usepackage[margin=0.5in]{geometry}

\usepackage[backend=biber, style=numeric, sorting=none]{biblatex}

\graphicspath{./figures/}

\newcommand{\norm}[1]{\left\lVert#1\right\rVert}

\DeclareMathOperator*{\argmax}{argmax}

\newtheorem{theorem}{Theorem}[section]
\newtheorem*{theorem*}{Theorem}
\newtheorem{lemma}{Lemma}[section]
\newtheorem{proposition}{Proposition}[section]
\newtheorem{corollary}{Corollary}[section]
\newtheorem{definition}{Definition}[section]

\title{Informationally Compressive Anonymization: Non-Degrading Sensitive Input Protection for Privacy-Preserving Supervised Machine Learning}
\author{Jeremy J. Samuelson: EVP, Artificial Intelligence \& Innovation}
\affil{Integrated Quantum Technologies}
\date{May 2026}

\begin{document}

\maketitle

\begin{abstract}
    Modern machine learning systems increasingly rely on sensitive data, creating significant privacy, security, and regulatory risks that existing privacy-preserving machine learning (ppML) techniques—such as Differential Privacy (DP) and Homomorphic Encryption (HE)—address only at the cost of degraded performance, increased complexity, or prohibitive computational overhead. This paper introduces Informationally Compressive Anonymization (ICA) and the VEIL architecture, a privacy-preserving ML framework that achieves strong privacy guarantees through architectural and mathematical design rather than noise injection or cryptography. ICA embeds a supervised, multi-objective encoder within a trusted Source Environment to transform raw inputs into low-dimensional, task-aligned latent representations, ensuring that only irreversibly anonymized vectors are exported to untrusted training and inference environments. The paper rigorously proves that these encodings are structurally non-invertible using topological and information-theoretic arguments, showing that inversion is logically impossible—even under idealized attacker assumptions—and that, in realistic deployments, the attacker’s conditional entropy over the original data diverges, driving reconstruction probability to zero. Unlike prior autoencoder-based ppML approaches, ICA preserves predictive utility by aligning representation learning with downstream supervised objectives, enabling low-latency, high-performance ML without gradient clipping, noise budgets, or encryption at inference time. The VEIL architecture enforces strict trust boundaries, supports scalable multi-region deployment, and naturally aligns with privacy-by-design regulatory frameworks, establishing a new foundation for enterprise ML that is secure, performant, and safe by construction—even in the face of post-quantum threats.
\end{abstract}

\begin{multicols}{2}

\section{Introduction}

Today, ML applications are nearly ubiquitous in many industry sectors, and the speed with which ML is being adopted is continuing to increase. The successful implementation of ML often requires very large amounts of training data. However, the need for such volumes of data to develop highly effective models has raised serious privacy concerns among regulators. For certain application domains, such as financial services and healthcare, the data needed can be extremely sensitive in nature. Furthermore, bad actors have been quick to respond to the widespread use of ML, and have developed new methods for exploiting ML solutions \cite{7958568, 10.1145/2810103.2813677, 10.1145/3624010}. A given ML model may face multiple types of adversarial attacks depending on the type of access to the model an attacker gains. Among the most common ML attacks is the interception and exfiltration of sensitive input features in-flight from the data source to the deployed ML model's prediction API \cite{tramèr2016stealingmachinelearningmodels, liu2025datareconstructionattacksdefenses, ibm2024mlops, siposova2023dataexfiltration}.\\
In response to this emerging threat, many regulations and ethical data policies, such as the EU's General Data Protection Regulation (GDPR), the California Consumer Privacy Act (CCPA), the California Privacy Rights Act (CPRA), and the Health Insurance Portability and Accountability Act (HIPAA), were put in place to raise awareness and require ML practitioners to take precautions against attacks and data breaches that might occur in a given ML pipeline. Researchers have developed various techniques to defend against privacy attacks. Among the most common techniques are Differential Privacy (DP), Federated Learning (FL), and Homomorphic Encryption (HE) \cite{10.1145/3460427}. However, multiple studies have shown that these methods preserve data privacy at the expense of model performance \cite{DBLP:journals/corr/abs-1812-03224} and can add a great deal of additional computational burden \cite{s23031252, 9734024}, rendering such methods impractical, or even useless, for some industry applications that require low response times and high predictive reliability.\\
The autoencoder neural network architecture was introduced, and has long been used, as a methodology for nonlinear dimensionality reduction that produces a latent vector space embedding for use with other ML algorithms \cite{goodfellow2016}. Some researchers have shown that this latent vector space embedding can serve as a deep learning (DL) powered encoding for privacy-preserving data transfer over public channels when necessary \cite{sagar2019, quinteroossa2022privacypreserving}. However, information loss due to the Information Bottleneck Principle \cite{tishby2015deeplearninginformationbottleneck} often results in degradation of predictive performance in the downstream ML model.\\
The purpose of this paper is two-fold: first, to describe the current state-of-the-art in ppML methodologies, with a specific focus on the Multi-Objective Supervised Convolution-Residual AutoEncoder (SCRAE) architecture introduced by Ouaari et al.\ in 2023 \cite{ouaari2023robust}. This novel architecture aligns the representation produced by the encoder with the prediction objective of the downstream ML model. This supervised approach is combined with a multi-objective framework that enforces certain properties in the latent representation of the inputs to provide privacy protection by obscuring the original data dimensionality, format, distribution, and content in a way that prevents accurate reconstruction, while maintaining predictive utility.\\
This combined architecture is first developed in a way that is specialized for classification and is then extended to regression to produce a methodology that is generalized to all supervised learning.

\section{A Brief Introduction to AutoEncoders}

In this section, a brief introduction to and discussion of the autoencoder architecture, also known as an auto-associative neural network \cite{bishop2009}, is given before showing, in an upcoming section, how the inherent information loss and resulting ML model degradation have been mitigated in the framework introduced by Ouaari et al.\\
Principal Component Analysis (PCA) can be considered as a factor analysis that learns a linear mapping $\mathbf{x} \mapsto \mathbf{z}$, called the \textbf{encoder}, $f_e$, and another linear mapping $\mathbf{z} \mapsto \mathbf{x}$, called the \textbf{decoder}, $f_d$. The overall reconstruction function is of the form $\hat{\mathbf{x}} = f_d\left(f_e\left(\mathbf{x}\right)\right)$. The model is then trained to minimize the squared reconstruction error, $\mathcal{L}\left(\boldsymbol{\theta}\right) = \norm{\mathbf{x} - \hat{\mathbf{x}}}_2^2$ \cite{murphy2022}.\\
Now consider a feed-forward artificial neural network with a single hidden layer consisting of $E$ neurons, in which the hidden units are computed as $\mathbf{z} = W_1\mathbf{x}$, and the output is a reconstruction of the input given by $\hat{\mathbf{x}} = W_2\mathbf{z}$, where $\mathbf{x} \in \mathbb{R}^D$, $W_1$ is an $E \times D$ matrix, $W_2$ is a $D \times E$ matrix, and $E < D$. Thus, the output of the model is given by $\hat{\mathbf{x}} = W_2W_1\mathbf{x} = \hat{W}\mathbf{x}$. If this model is trained to minimize the squared reconstruction error, $\mathcal{L}\left(\hat{W}\right) = \sum_{n=1}^N\norm{\mathbf{x}_n - \hat{W}\mathbf{x}_n}_2^2$, it has been shown \cite{baldi198953, karhunen1995} that $\hat{W}$ is an orthogonal projection onto the first $E$ eigenvectors of the empirical covariance matrix of the data, $\Sigma_X \in \mathbb{R}^{D \times D}$. Therefore, this \textbf{linear autoencoder} is equivalent to PCA \cite{jolliffe2002}.\\
If the architecture of this autoencoder is extended to include nonlinear activations, this produces a model that has been proven \cite{japkowics2000} to be strictly \textit{more} powerful than PCA\@. Such a model is capable of learning very useful representations of data, even under significant compression to a relatively low-dimensional latent space \cite{murphy2023}.
\begin{center}
\includegraphics[scale=0.35]{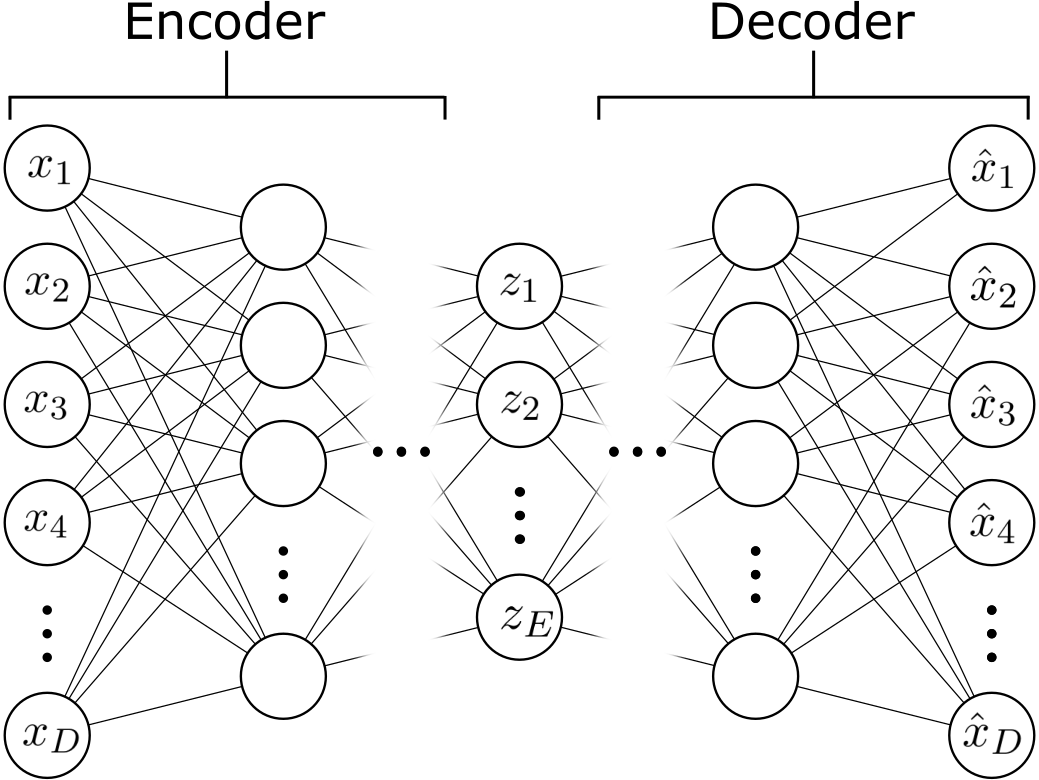}
\captionof{figure}{A typical ``bottleneck'' AutoEncoder, yielding an undercomplete representation of the input data}
\label{fig:autoencoder}
\end{center}
In general, autoencoders may have arbitrarily many hidden layers with nonlinear activations, yielding arbitrarily complex mappings. The hidden layers often form a bottleneck between the input and its reconstruction, as shown in Figure \ref{fig:autoencoder}.
\begin{center}
\includegraphics[scale=0.35]{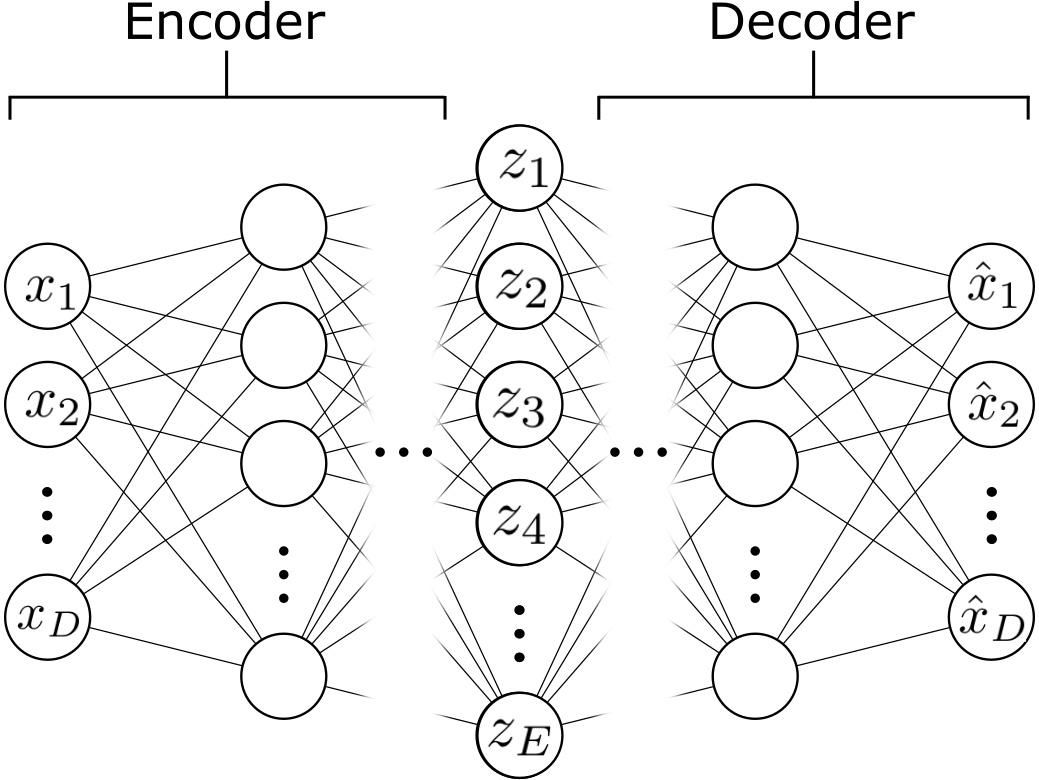}
\captionof{figure}{An AutoEncoder yielding an overcomplete representation of the input data}
\label{fig:autoencoder2}
\end{center}
Of course, if the hidden layers are wide enough, there is nothing to stop such a model from merely learning the identity function \cite{murphy2022}. To prevent this degenerate solution, the model must be restricted in some way. The simplest approach is to use a narrow bottleneck, where $E \ll D$. The resulting latent embedding is called an \textbf{undercomplete representation}. An alternative approach is to use $E \gg D$, as shown in Figure \ref{fig:autoencoder2}, which yields an \textbf{overcomplete representation}.
\begin{center}
\includegraphics[scale=0.35]{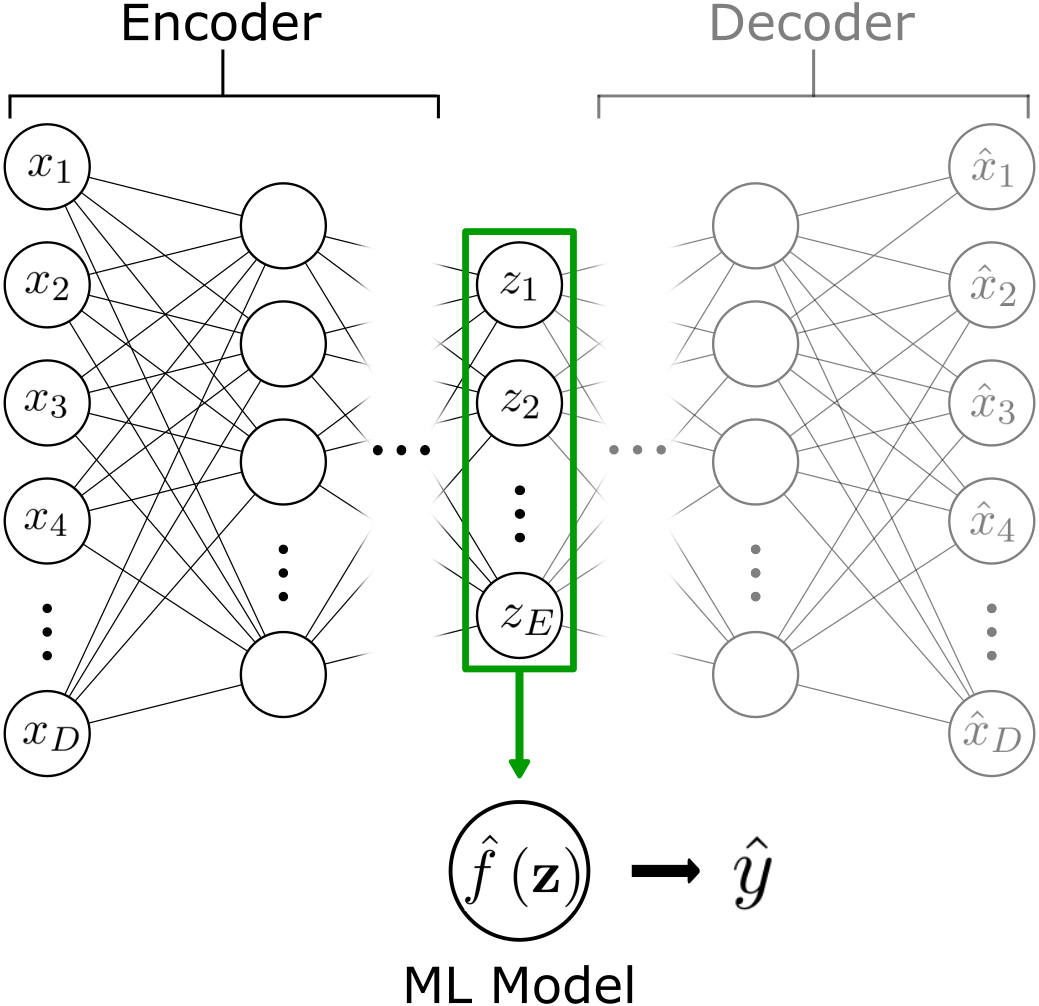}
\captionof{figure}{The latent representation of an AutoEncoder can be directly input into another ML model}
\label{fig:autoencoder3}
\end{center}
When using an overcomplete representation, additional regularization is needed to prevent the model from overfitting, such as adding noise to the inputs, forcing hidden activations to be sparse, or forcing weights to be small by adding a penalty term to the loss function. The autoencoder architecture is extremely flexible in its applications, as the densely-connected layers of the feed-forward architecture may also be substituted for convolutional, recurrent, or residual layers, to adapt the model to encode images, text, or other types of data, as needed.\\
As previously mentioned, the autoencoder architecture is often used as a method for reducing the dimensionality of input data for other ML models. This means the latent representations produced by autoencoders do not need to be decoded prior to model inference, as ML models are capable of learning patterns and mappings directly from the latent vector space $\mathbf{z} \mapsto \mathbf{y}$, as illustrated in Figure \ref{fig:autoencoder3}.\\
This is due to the fact that the autoencoder must simultaneously learn the encoder, $f_e$, and the decoder, $f_d$, ensuring that the latent representation has retained enough information from the original input to accurately reconstruct it. Therefore, enough information is present in the latent embedding to support accurate ML inferences, despite the fact that the data have been given a complex, nonlinear encoding \cite{10.3389/fninf.2023.1074653}.\\
As mentioned before, when training ML models on latent encodings from autoencoders, there is some amount of degradation in the performance of the ML model due to the obfuscation of the input data. This has long been accepted as an unavoidable trade off for reducing dimensionality to achieve a smaller model or to protect the input data from bad actors.

\section{Related Work and the Current State of AutoEncoder-Based Privacy Protection}

Earlier applications of autoencoders for secure transmission treated the latent encoding as a cipher. Sagar and Kumar (2019) proposed ``AutoEncoder Artificial Neural Network Public Key Cryptography'' (AANNPKC) for encrypting messages over unsecured public channels \cite{sagar2019}. This was done by training identical autoencoders at both the sender and receiver nodes of a communication channel, with the shared training data acting as the encryption key. They benchmarked throughput and latency against several classical ciphers (DES, 3DES, Blowfish, IDEA, AES, and RC6), reporting higher performance for small payloads and demonstrating feasibility. However, this method relied on maintaining the secrecy of both the encoder and decoder trained parameters, as discovery of these parameters would make decoding of the transmitted data possible. This methodology also suffered from degradation of data due to the information compression in undercomplete autoencoders.\\
More recent work by Ouaari et al. (2023) \cite{ouaari2023robust} explored combating this degradation and improving autoencoder-based prediction utility by training a Supervised Convolution-Residual AutoEncoder (SCRAE) architecture. In this framework, the training of the encoder is aligned to two objectives to simultaneously minimize the reconstruction error and the prediction error of a supervised classifier that ingests the encoding. This is expressed by the authors in the following loss function.
\begin{dmath}
\mathcal{L}\left(\mathbf{x}, y, \boldsymbol{\theta}_e, \boldsymbol{\theta}_d, \boldsymbol{\theta}_c\right) = \frac{1}{N}\sum_{n=1}^N \left[ \mathcal{L}_R\left(\mathbf{x}_n, f_d\left(f_e\left(\mathbf{x}_n, \boldsymbol{\theta}_e\right), \boldsymbol{\theta}_d\right)\right) + \mathcal{L}_M\left(y_n, \hat{f}\left(f_e\left(\mathbf{x}_n, \boldsymbol{\theta}_e\right), \boldsymbol{\theta}_c\right)\right)\right]
\end{dmath}
In the next section, this multi-objective approach is extended to include other objectives enforced on the latent representation. These additional objectives lead to beneficial geometric properties in the latent representation that are shown to further preserve, or enhance, the predictive utility of the encoding. This extension of the multi-objective training framework is also combined with aspects of a multi-level convolutional neural network architecture proposed by Xu and Duraisamy (2020) \cite{Xu_2020}, allowing the autoencoder to learn robust encodings across multiple levels, or dimensionalities.

\section{Multi-Level, Multi-Objective AutoEncoders for Classification}

To begin the extension of this multi-objective framework, the overall objective scaffolding is clarified by separating the individual objectives. It is then extended by adding the new objective terms and modified with coefficients for each term to enable setting or modifying the relative prioritization of each objective. This enhanced multi-objective function is expressed as follows.
\begin{dmath}
\mathcal{L} = \lambda_{\text{recon}}\mathcal{L}_{\text{recon}} + \lambda_{\text{repr}}\mathcal{L}_{\text{repr}} + \lambda_{\text{pred}}\mathcal{L}_{\text{pred}} + \lambda_{\text{reg}}\mathcal{L}_{\text{reg}}
\label{eq:multi-objective_loss}
\end{dmath}
It becomes clear that the overall loss has four components: the autoencoder reconstruction loss $\left(\mathcal{L}_{\text{recon}}\right)$, the loss function enforcing useful characteristics in the latent representation $\left(\mathcal{L}_{\text{repr}}\right)$, the prediction loss of the downstream ML model $\left(\mathcal{L}_{\text{pred}}\right)$, and a final term for any regularization on the latent representation $\left(\mathcal{L}_{\text{reg}}\right)$. The $\lambda$ values are the aforementioned hyperparameters used to specify the relative importance of each objective.\\
The most common objective function used to measure reconstruction loss is the Ordinary Least Squares loss, also known as the Mean Squared Error (MSE), which is adequate.
\begin{dmath}
\mathcal{L}_{\text{OLS}} = \frac{1}{2N}\sum_{i=1}^N \norm{\mathbf{x}_i - \hat{\mathbf{x}}_i}_2^2
\end{dmath}
Given that the goal is to create a ppML system, a desirable property of the latent encodings is that they cannot support reconstruction. In future sections, it will also be established that no corresponding decoder is deployed as part of the ppML system. For these reasons, the recommended default is $\lambda_{\text{recon}}=0$. Unlike other systems that use autoencoders, including the approach defined by Ouaari et al., under this new objective structure, minimizing the reconstruction error is demoted from a primary objective to an optional regularization term that can be added to resist latent collapse during training, which will be discussed in future sections. As such, $\lambda_{\text{recon}}$ should always be kept small relative to the other objective prioritization hyperparameters.\\
Starting with the SCRAE architecture, the densely-connected hidden layers of the encoder are concatenated into a single vector, $\Psi$, before being passed to the downstream ML model. This creates a multi-level autoencoder, as the encoder must learn to create optimal latent representations both at the last layer of the encoder, as well as at the level of this concatenation \cite{Xu_2020}. The representation objective function, $\mathcal{L}_{\text{repr}}$, is applied to $\Psi$ to ensure the produced latent representation has the desired properties.\\
In the case of classification, a model must have clearly discernible class clusters to accurately classify observations. Ideally, these class clusters will be as separable as possible to enable discrimination via the classifier's learned decision boundary. To achieve this, the objective applied to the concatenated latent representation is the Center loss \cite{10.1007/978-3-319-46478-7_31},
\begin{dmath}
\mathcal{L}_{\text{C}} = \frac{1}{2N}\sum_{i=1}^N \norm{\mu_{y_i} - \Psi_i}_2^2
\end{dmath}
where $\mu_{y_i} \in \mathbb{R}^E$ is the latent center of the class to which observation $i$ belongs. The introduction of the Center loss minimizes intra-class distance and improves target resolution for the deployed ML model as visualized in Figures \ref{fig:no_center_loss} and \ref{fig:center_loss}.
\begin{center}
\includegraphics[scale=0.65]{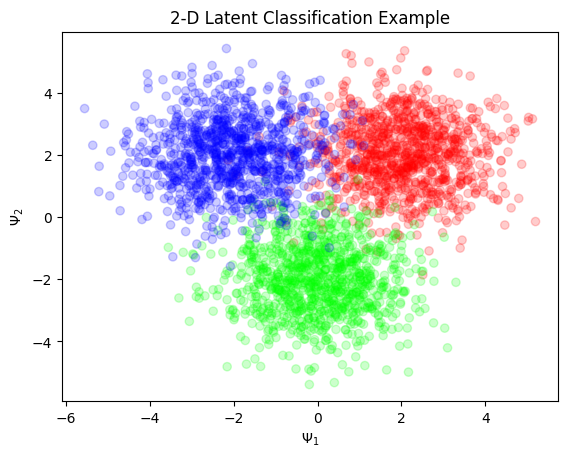}
\captionof{figure}{2-D latent representation without the Center loss}
\label{fig:no_center_loss}
\end{center}
\begin{center}
\includegraphics[scale=0.65]{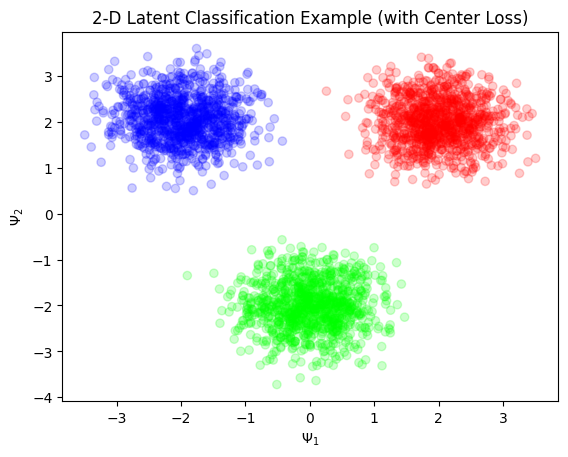}
\captionof{figure}{2-D latent representation after training with Center loss}
\label{fig:center_loss}
\end{center}
Any common prediction objective function used in classification is suitable to train the downstream model and to maximize the prediction utility of the encoder. Common choices include the Cross-Entropy loss function,
\begin{dmath}
\mathcal{L}_{\text{CE}} = -\frac{1}{N}\sum_{i=1}^N\sum_{k=1}^K y_{ik}\text{ln}p_{ik}
\end{dmath}
or the Hinge loss function, more commonly used to train perceptron \cite{minsky1988perceptrons} and Support Vector Machine \cite{cristianini2000svm} models.
\begin{dmath}
\mathcal{L}_{\text{Hinge}} = \frac{1}{N}\sum_{i=1}^N\sum_{k \neq y_i} \text{max}\left\{0, 1 + \langle W_k, \mathbf{x}_i\rangle - \langle W_{y_i}, \mathbf{x}_i\rangle\right\}
\end{dmath}
There are many options for model regularization within ML and DL\@. To achieve stability within the latent representation and to dampen the high training variance sometimes observed in neural networks, a PCA Cosine Similarity loss is employed to enforce alignment between the learned nonlinear manifold and the geometry of PCA, which is equivalent to a linear autoencoder, as established in Section 2.
\begin{dmath}
\mathcal{L}_{\text{PCA}} = \frac{1}{N}\sum_{i=1}^N 1 - \frac{f_2\left(\Psi_i\right)\cdot\left(\mathbf{x}_{\text{PCA}}\right)_i}{\norm{f_2\left(\Psi_i\right)}\norm{\left(\mathbf{x}_{\text{PCA}}\right)_i}}
\end{dmath}
Here, $f_2:\mathbb{R}^{\sum_{l=1}^Lm_l} \rightarrow \mathbb{R}^2$, $m_l$ is the dimensionality of densely-connected layer $l$ of the encoder, and $\left(\mathbf{x}_{\text{PCA}}\right)_i \in \mathbb{R}^2$ is the vector consisting of the first two principal components of the original input data, $\mathbf{x}_i$. This enforced PCA alignment also preserves the global geometric structure, as PCA does \cite{GOU2023364}, thereby placing some restrictions on the learned latent representation and reducing variance by imposing bias.\\
All of these components are combined to produce the final autoencoder architecture, visualized in Figure \ref{fig:crae}.
\begin{center}
\includegraphics[scale=0.4]{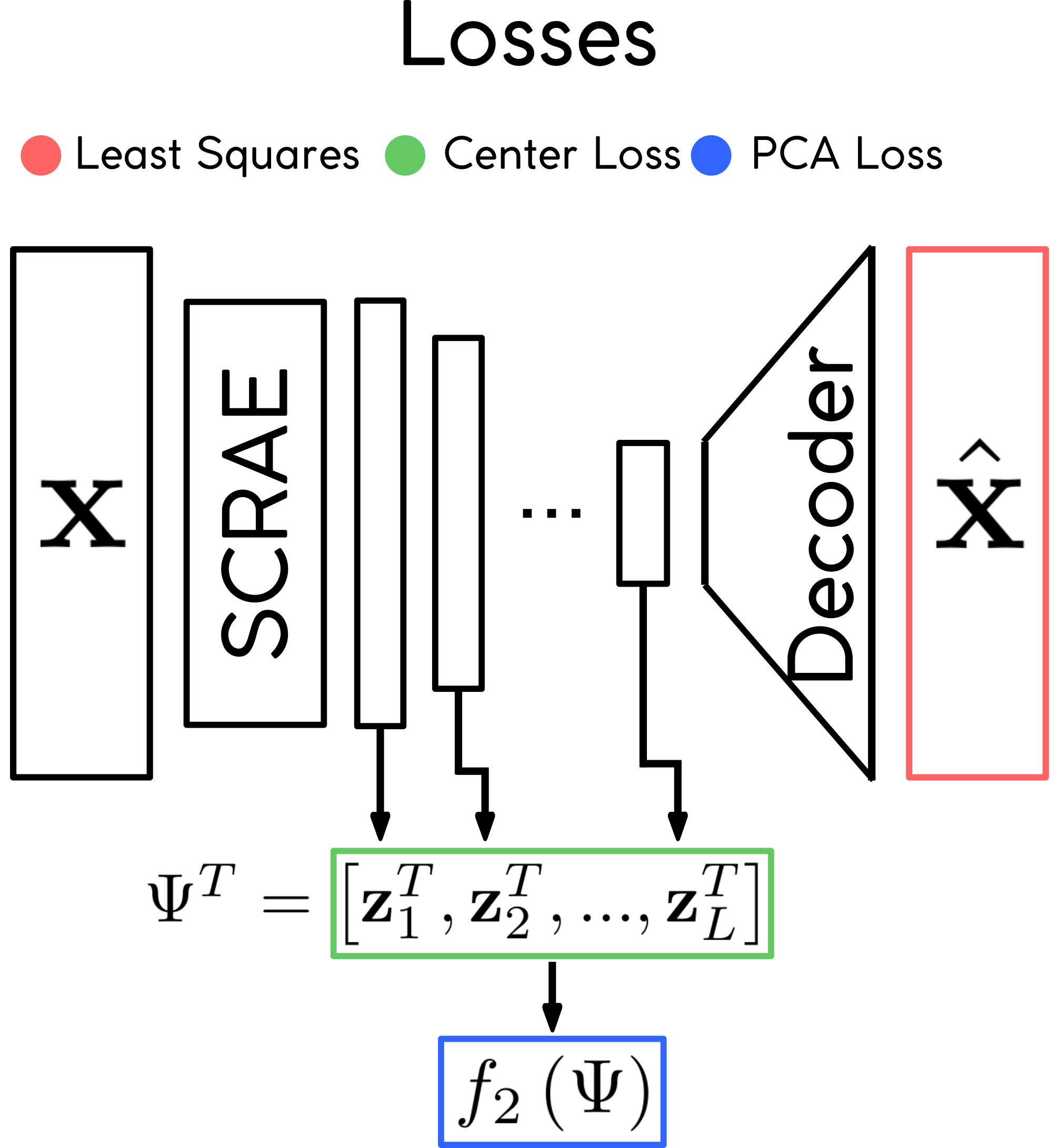}
\captionof{figure}{Multi-Level, Multi-Objective SCRAE architecture}
\label{fig:crae}
\end{center}
The combination of these objectives and architectural aspects yields a powerful privacy-preserving encoding with astonishing predictive utility. To demonstrate this, a performance benchmark study was conducted on the well-known MNIST dataset \cite{lecun1998gradient}. First, a Logistic Regression model was trained on the raw 784 dimensional input. For comparison, a Logistic Regression model was also trained under DP with $\epsilon=7.992$ and $\delta=2.78\times10^{-10}$. A third Logistic Regression was trained on 2-dimensional latent encodings produced by an unsupervised densely-connected autoencoder, as described in Section 2. Finally, a fourth Logistic Regression was trained with an SCRAE architecture, also emitting 2-dimensional latent encodings. Accuracy of each model pipeline was measured on a held-out test set, and predictive performance was compared against the raw baseline. The results are shown in Table \ref{tab:class_results}.
\begin{center}
\begin{tabular}{lcc}
\toprule
\textbf{Model Pipeline} & \textbf{Test Accuracy} & \textbf{vs. Raw} \\
\midrule
Raw Data & $92.42\%$ & $+0.00\%$ \\
DP & $91.78\%$ & $-0.64\%$ \\
Dense AE & $19.69\%$ & $-72.73\%$ \\
SCRAE & $98.61\%$ & $+6.19\%$ \\
\bottomrule
\end{tabular}
\captionof{table}{MNIST Test Accuracy Comparison}
\label{tab:class_results}
\end{center}
The degradation of the densely connected autoencoder is evident, as the test accuracy of the model trained on this encoding dropped by 72.73 percentage points compared with the model trained on the raw data. A small amount of predictive degradation was also observed in the DP pipeline as an expected privacy-utility trade-off. In contrast, the model trained on the SCRAE encoding not only maintained test accuracy relative to the raw data but also exceeded it by 6.19 percentage points, outperforming the densely connected autoencoder pipeline by 78.92 percentage points and the DP pipeline by 6.83 percentage points in out-of-sample accuracy. This is a significant breakthrough in ppML, as the degradation in downstream ML model performance is no longer a necessary trade-off to be accepted for the sake of security.\\
The informational compression achieved by the SCRAE encoder in this experiment also warrants explicit quantification. Each raw MNIST image is a $28\times28$ pixel grayscale image in which each pixel is represented as an 8-bit integer, yielding a raw representation of $784 \times 8\text{ bits}=6,272\text{ bits}$ per image. The SCRAE encoding reduces each image to a 2-dimensional latent vector, with each dimension stored as a 32-bit floating-point value, yielding an encoded representation of $2 \times 32\text{ bits}=64\text{ bits}$. This corresponds to a compression ratio of $64/6,272 \approx 0.0102$, reflecting a reduction in storage and transmission requirements given by the following.
$$\left(1-\frac{64}{6,272}\right)\times 100\% \approx 98.98\%$$
Critically, this compression is not a degradation of the input signal. As Table \ref{tab:class_results} demonstrates, it is the direct result of the Information Bottleneck Principle \cite{tishby2015deeplearninginformationbottleneck} being deliberately harnessed by the multi-objective encoder to discard variation that is not predictively useful while concentrating task-relevant structure within the latent geometry. The SCRAE does not produce a lossy approximation of the raw input; it produces a representation that is, by the metric of downstream predictive accuracy, strictly more informative.\\
The contrast between this outcome and the representation overhead introduced by HE is stark. Ciphertext expansion is a structural and unavoidable property of lattice-based HE schemes, arising from the requirement that plaintext values be embedded within high-degree polynomial rings to support arithmetic operations in the encrypted domain \cite{tang2025hefdivs}. Rather than compressing inputs, HE necessarily expands them. This has been studied empirically in the context of neural network inference. In the CryptoNets system \cite{gilad2016cryptonets}, a well-known HE-based framework applied to MNIST classification, each pixel of the input image is encoded as a separate ciphertext polynomial. Under the original parameterization of that system, a single MNIST image occupying 784 bytes in raw form requires approximately 367.5 megabytes of ciphertext prior to batching---an expansion factor of 468,750.\\
Even when 4,096 images are packed together using Single Instruction, Multiple Data (SIMD) batching to amortize the polynomial overhead, the per-image ciphertext size remains approximately 91.9 kilobytes, representing a roughly 117-fold expansion relative to the raw image \cite{gilad2016cryptonets}. More recent empirical analyses of HE applied to high-dimensional feature spaces have reported ciphertext expansion factors reaching as high as 500 times the plaintext size \cite{tang2025hefdivs}, reflecting a penalty that grows with input dimensionality. In addition to this data expansion, HE-based inference incurs latencies of tens to hundreds of seconds per inference on standard hardware, even with optimized implementations \cite{gilad2016cryptonets, brutzkus2019lola}, rendering it incompatible with the low-latency production deployments described in Section 8 of this paper and demanded by most industry applications.\\
DP does not modify the structure or raw size of input data and therefore does not impose ciphertext expansion in the sense described above. Instead, under the canonical DP-SGD algorithm \cite{abadi2016deep}, privacy guarantees are enforced at training time through the clipping and noise injection of per-example gradients before model parameter updates are applied. This mechanism introduces its own efficiency trade-offs, however. Because DP-SGD must compute gradients on a per-example rather than per-batch basis prior to clipping, the matrix-algebra parallelism exploited by GPU-based training pipelines is substantially undermined, introducing computational overhead that has been identified as a practical deployment obstacle \cite{nist2023dp}.\\
Beyond computational cost, the injection of Gaussian noise into per-example gradients degrades predictive accuracy in proportion to the strength of the privacy guarantee selected, as a fundamental consequence of the privacy-accuracy trade-off inherent to DP \cite{DBLP:journals/corr/abs-1812-03224}. This trade-off must be managed through a carefully chosen privacy budget expressed as an $\left(\epsilon, \delta\right)$-pair, parameters whose selection requires domain expertise and which must be tracked across the entire training lifecycle \cite{abadi2016deep, nist2023dp}.\\
ICA via the SCRAE, and the VEIL architecture introduced in Section 8 of this paper, avoid both categories of overhead entirely. Because raw data never cross the trust boundary, no noise injection is required at training time, gradient clipping is unnecessary, and there are no privacy budgets to maintain. The 98.98\% reduction in input size demonstrated in this experiment is therefore not merely an incidental efficiency gain; it is an integral consequence of a privacy mechanism that achieves its guarantees through architectural separation rather than through the imposition of representational or computational overhead on the training or inference pipeline.\\
While this section covers the utility and efficiency-preserving properties of this encoding, it does not address the privacy-preserving properties. A deeper theoretical analysis of the privacy-preserving properties of the latent encodings is found in Section 9 of this paper, and Section 10 describes supporting empirical evidence for these properties through experiments involving multiple simulated privacy attacks, including attempted input reconstruction on both the DP and SCRAE modeling pipelines.

\section{Adapting the Multi-Level, Multi-Objective AutoEncoder to Regression}

With the multi-level architecture and multi-objective loss function scaffolding developed for classification, the task of adapting this framework to regression amounts to finding objectives that achieve the desired effects in the context of target variable spaces of infinite cardinality. Beginning with the multi-objective loss function given by Equation \ref{eq:multi-objective_loss}, each component is considered, and suitable substitutes are found when necessary.\\
The reconstruction loss is agnostic to whether the downstream ML task is classification or regression, as is the PCA similarity loss. The most obvious change needed is the prediction loss of the downstream ML model $\left(\mathcal{L}_{\text{pred}}\right)$. The prediction loss used in classification, such as Cross-Entropy or Hinge loss, can be replaced with any objective function suited to regression. Examples of commonly used objective functions are the Ordinary Least Squares (OLS) loss
\begin{dmath}
\mathcal{L}_{\text{OLS}} = \frac{1}{2N}\sum_{i=1}^N \norm{\mathbf{y}_i - \hat{\mathbf{y}}_i}_2^2
\end{dmath}
and the Mean Absolute Error (MAE) loss.
\begin{dmath}
\mathcal{L}_{\text{MAE}} = \frac{1}{N}\sum_{i=1}^N \norm{\mathbf{y}_i - \hat{\mathbf{y}}_i}_1
\end{dmath}
For a robust regression model, which has the benefits of both the OLS and MAE loss functions, the Huber loss function \cite{10.1214/aoms/1177703732, meyer2020alternativeprobabilisticinterpretationhuber} is recommended.
\begin{dmath}
\mathcal{L}_{\delta} = 
    \begin{cases}
        \frac{1}{2}\left(y_i - \hat{y}_i\right)^2 & \text{for} \left|y_i - \hat{y}_i\right| < \delta\\
        \left|y_i - \hat{y}_i\right|\delta - \frac{1}{2}\delta^2 & \text{for} \left|y_i - \hat{y}_i\right| \geq \delta
    \end{cases}
\end{dmath}
This leaves only the representation loss $\left(\mathcal{L}_{\text{repr}}\right)$ to consider. The Center loss function relies on discrete target values and is therefore also specific to the case of classification. Here, an objective function is needed to enforce the desired properties in the latent representation---smoothness and similar latent representations for observations with similar target values---but that is suitable for regression target variable spaces of infinite cardinality.\\
The first proposed objective function constructs a densely-connected, weighted graph in which the nodes are the latent representations of the observations, $\Psi$, and the edge weights, $\gamma_{ij}$, are determined by the similarity of the corresponding target values as follows.
\begin{dmath}
\gamma_{ij} = e^{-\frac{1}{\sigma^2}\norm{y_i-y_j}^2}
\label{eq:reg_similarity_kernel}
\end{dmath}
Here, $\sigma$ is a hyperparameter that controls how quickly the edge weight decreases as the distance between targets becomes larger. This can either be automatically tuned to the distribution of $y$ as,
$$\sigma = \frac{1}{2}\text{IQR}\left(\frac{y - \mu_y}{\sqrt{\sigma_y}}\right)$$
or it can be found via a hyperparameter search in training. An objective function that then enforces similar latent representations for observations with similar target values is given by the following.
\begin{dmath}
\mathcal{L}_{\text{Lap}} = \frac{1}{2N\left(N-1\right)}\sum_{i=1}^N\sum_{j\neq i} \gamma_{ij}\norm{\Psi_i-\Psi_j}_2^2
\end{dmath}
This loss is the Dirichlet energy of $\Psi$ on the above-described latent representation graph \cite{zhou2021dirichletenergyconstrainedlearning}. Therefore, minimizing it pulls together latent encodings whose targets are close, and doesn't enforce proximity when targets are far apart, yielding a smooth, target-aware embedding, rather than discrete class clusters, which is ideal for regression models. Thus, this objective function is a suitable replacement for the Center loss.\\
For an alternative perspective on this loss function, define $\Gamma = \left[\gamma_{ij}\right]$ and $D = \text{diag}\left(e_i\right)$ where $e_i=\sum_{j \neq i}\gamma_{ij}$. Then, the unnormalized graph Laplacian is $L=D-\Gamma$. Stack the latent encodings as a matrix $\mathbf{\Psi} \in \mathbb{R}^{N \times E}$ where the $i$-th row is $\Psi_i$. Then, the following equivalence emerges.
\begin{dmath}
\frac{1}{2N}\sum_{i=1}^N\sum_{j \neq i}\gamma_{ij}\norm{\Psi_i-\Psi_j}_2^2 = \frac{1}{N}\text{tr}\left(\mathbf{\Psi}^TL\mathbf{\Psi}\right)
\end{dmath}
Thus, the above-described loss is the graph-smoothness penalty familiar from Laplacian Eigenmaps \cite{choi2025geometricmachinelearningeeg}. Minimizing $\text{tr}\left(\mathbf{\Psi}^TL\mathbf{\Psi}\right)$ makes $\mathbf{\Psi}$ vary slowly across edges that connect similar targets, while remaining stationary across edges that connect dissimilar targets. For this reason, this loss is referred to as the Graph-Laplacian loss function, or $\mathcal{L}_{\text{Lap}}$.\\
The Graph-Laplacian loss preserves a continuum; that is, observations with similar targets are given similar latent representations. It preserves order, or rank, because $\gamma_{ij}$ decays with $\norm{\mathbf{y}_i - \mathbf{y}_j}$. Therefore, larger target gaps allow larger latent distances, thus encoding the ordering in $\mathbf{y}$. When combined with $\mathcal{L}_{\text{recon}}$ and $\mathcal{L}_{\text{pred}}$, the Graph-Laplacian loss is also robust to trivial collapse, in which all latent encodings, $\Psi_i$, are equal. The Graph-Laplacian then shapes the latent geometry, rather than defining it alone.\\
The gradient with respect to one encoding, $\Psi_i$, is as follows.
\begin{dmath}
\frac{\partial\mathcal{L}_{\text{Lap}}}{\partial\Psi_i} = \frac{1}{N\left(N-1\right)}\sum_{j \neq i}\gamma_{ij}\left(\Psi_i - \Psi_j\right)
\end{dmath}
Therefore, each $\Psi_i$ is pushed towards the weighted average of its target neighbors, further indicating a direct local smoothness force.\\
A study analogous to the one described in Section 4 was conducted using the Ames Housing dataset \cite{decock2011ames}. In its raw form, this dataset has 80 features. After feature engineering and selection, the design space had 282 dimensions. In this study, the term ``raw data'' will refer to the engineered inputs from the 282-dimensional design space. The downstream model for all pipelines in this study was Linear Regression, which was used to predict house sale prices. The optimized settings for the DP pipeline were $\left(\epsilon, \delta\right)=\left(7.999, 1.82\times10^{-7}\right)$. The latent dimensionality of the two autoencoder pipelines was $E=14$, the lowest encoding dimensionality at which the SCRAE could maintain baseline raw model performance. The results are shown in Tables \ref{tab:reg_results_r2} and \ref{tab:reg_results_rmsle}.
\begin{center}
\begin{tabular}{lcc}
\toprule
\textbf{Model Pipeline} & \textbf{Test $R^2$} & \textbf{vs. Raw} \\
\midrule
Raw Data & $0.9315$ & $+0.0000$ \\
DP & $0.8464$ & $-0.0850$ \\
Dense AE & $0.8993$ & $-0.0321$ \\
SCRAE & $0.9421$ & $+0.0106$ \\
\bottomrule
\end{tabular}
\captionof{table}{Ames Housing Test $R^2$ Comparison}
\label{tab:reg_results_r2}
\end{center}
\begin{center}
\begin{tabular}{lcc}
\toprule
\textbf{Model Pipeline} & \textbf{Test RMSLE} & \textbf{vs. Raw} \\
\midrule
Raw Data & $0.1279$ & $+0.0000$ \\
DP & $0.1642$ & $+0.0362$ \\
Dense AE & $0.1496$ & $+0.0217$\\
SCRAE & $0.1199$ & $-0.0081$ \\
\bottomrule
\end{tabular}
\captionof{table}{Ames Housing Test RMSLE Comparison}
\label{tab:reg_results_rmsle}
\end{center}
These results show that the adaptation of the multi-level, multi-objective autoencoder to regression has been successful, as the out-of-sample performance of the downstream ML model consuming the SCRAE encoding once again shows no performance degradation, unlike the DP and standard autoencoder pipelines, and is on par with the performance of the model trained on the raw data.\\
This demonstrates that the extent to which a multi-level, multi-objective autoencoder can compress information and data while maintaining the downstream model's predictive utility depends on the specific data and modeling task. Therefore, the dimensionality of the latent encoding is a hyperparameter that should be optimized during training and validation. In this case, the raw inputs were represented as 32-bit floating-point numbers. Thus, the encoded model inputs had a compression ratio of $14/282 \approx 0.0496$, resulting in reduced storage and transmission requirements as follows.
$$\left(1-\frac{14}{282}\right)\times 100\% \approx 95.04\%$$
This regression-adapted multi-level, multi-objective autoencoder exhibits the same ability as the classification variant to compress information, thereby protecting sensitive inputs from bad actors, while improving efficiency and preserving the predictive utility of supervised ML models.

\section{Practical Considerations for Training with the Graph-Laplacian Loss}

Having been shown that predictive utility may be preserved under substantial informational compression, the practical conditions under which that result may be trained reliably must now be considered. In particular, the introduction of the Graph-Laplacian loss, while beneficial for enforcing a smooth and target-aware latent geometry, is accompanied by computational and optimization considerations that become material as batch size, latent dimensionality, and target noise increase. For this reason, the behavior of the training objective must be examined not only from a theoretical standpoint, but also from the standpoint of scalability, memory efficiency, and stability under realistic learning conditions.

\subsection{Managing Computational Complexity}

For a batch of size $B$, the Graph-Laplacian loss function computes $B\left(B-1\right) \approx B^2$ edge weights, leading to a time and memory complexity of approximately $\mathcal{O}\left(B^2\right)$. Assuming single precision (float32) tensors, for batches above $B \approx 1,024$, tensors of size $B^2$ start to dominate GPU RAM\@. In cases where larger batches, $B \geq 1,024$, are needed, switching to a sparsified variant is recommended.\\
A $k$-NN sparsification of the Graph-Laplacian loss function is proposed in which only the $k$ strongest target neighbors are kept for each node. For scalar targets, the $k$ nearest neighbors are simply the $k$ nodes with the smallest absolute difference. These can be found efficiently with a single sort on the target, $y$. This avoids the computation of a $B \times B$ distance matrix, and instead locates all nearest neighbor candidates in $\mathcal{O}\left(B\text{ ln}B\right)$. Then, trimming to the top $k$ per row results in a total complexity of $\mathcal{O}\left(B\text{ ln}B + kB\right)$.\\
For vector targets, the sort-window logic would need to be replaced by a $k$-NN search in the target space, which means tolerating a one-time $\mathcal{O}\left(B^2\right)$ computation, if done naively. This may be tolerable on batches of moderate size. With optimized software, such as FAISS or ScaNN, it should be possible to construct a $k$-NN graph in the target space in sub-quadratic time.\\
To further control the balance between complexity and smoothness, different schemes for symmetrization within the $k$-NN graph may be used. The most straightforward case is to use no symmetrization. In this case, the edges of the $k$-NN graph are directed. That is, the set of neighbor edges, $\kappa$, is defined as follows.
\begin{equation}
\kappa = \left\{ \left(i \rightarrow j\right): j \in \mathcal{N}_k\left(i\right) \right\}
\label{eq:knn_none}
\end{equation}
Here, each ordered pair contributes separately. The set $\kappa$ may contain the edge $i \rightarrow j$, but not $j \rightarrow i$. This method preserves asymmetries naturally present in the $k$-NN graph. For a batch of size $B$ and a latent encoding dimensionality of $E$, the edge computation with no symmetrization has complexity $\mathcal{O}\left(kBE\right)$. In this scenario, the Graph-Laplacian loss induces directed smoothness, since it is implemented as a sum over directed edges.\\
The $k$-NN graph may be symmetrized to either increase smoothness at the cost of additional complexity, or to reduce complexity at the risk of oversparsity. In the first case, the set of neighbor edges is constructed as an undirected union that includes all edges that appear in either direction.
\begin{equation}
\kappa = \left\{ \left\{i, j\right\}: j \in \mathcal{N}_k\left(i\right) \vee i \in \mathcal{N}_k\left(j\right)\right\}
\label{eq:knn_union}
\end{equation}
This can increase the number of edges up to $2kB$, leading to much better connectivity in the $k$-NN graph. This has the effect of stabilizing training and removing potential ``one-way'' artifacts, thereby improving the richness of the latent representation. For this reason, this method is proposed as the default, despite the increased complexity, which is $\mathcal{O}\left(2kBE\right)$.\\
The second symmetrization case involves keeping only the edges that appear in both directions, leading to the following set of neighbor edges.
\begin{equation}
\kappa = \left\{ \left\{i, j\right\}: j \in \mathcal{N}_k\left(i\right) \wedge i \in \mathcal{N}_k\left(j\right)\right\}
\label{eq:knn_intersection}
\end{equation}
This leads to fewer, higher-confidence edges, as the neighborhoods are mutually ``agreed upon'', yielding a very clean local structure in the latent representation that is robust to outliers and asymmetric noise in the target space. The reduced number of edges can reduce complexity, which is bounded above by $\mathcal{O}\left(kBE\right)$. However, this can lead to a $k$-NN graph that is too sparse, or even empty when $k$ is small, $B$ is small, or when noise in the target space is large. In any case, this method results in the following sparsified Graph-Laplacian loss $\left(\widetilde{\mathcal{L}}_{\text{Lap}}\right)$.
\begin{dmath}
\widetilde{\mathcal{L}}_{\text{Lap}} = \frac{1}{2kB}\sum_{i=1}^B\sum_{j\in\mathcal{N}_k\left(i\right)}\gamma_{ij}\norm{\Psi_i-\Psi_j}_2^2
\end{dmath}
Here, $\mathcal{N}_k\left(i\right)$ is the set of the $k$ nearest target neighbors of $\Psi_i$. The gradients of this sparsified function still flow through $\Psi$, giving the following gradient with the same local smoothing effect.
\begin{dmath}
\frac{\partial\widetilde{\mathcal{L}}_{\text{Lap}}}{\partial\Psi_i} = \frac{1}{kB}\sum_{j\in\mathcal{N}_k\left(i\right)}\gamma_{ij}\left(\Psi_i-\Psi_j\right)
\end{dmath}
During training, the dense Graph-Laplacian should be used whenever possible for the simplest implementation and the smoothest behavior in the latent representation. For large batches ($B \geq 1,024$), higher-dimensional latents, or situations where prohibitive memory pressure is observed, the $k$-NN sparsified Graph-Laplacian should be used instead.\\
To test the predictive performance of these $k$-NN sparsified variants of the Graph-Laplacian loss, a study was conducted using the YearPredictionMSD dataset \cite{bertinmahieux2011millionsong}, which is a regression dataset containing approximately 515,000 instances, each consisting of 90 features, which makes it large enough that the dense neighborhood construction of the Graph-Laplacian loss becomes visibly expensive. Three variants of the SCRAE architecture were trained, one for each of the proposed symmetrization schemes, with a $k$-NN sparsified Graph-Laplacian loss, using $B=1,024$ and $k=8$.\\
This study follows the same structure as the studies described in Sections 4 and 5. Initially, the downstream model was a Linear Regression, due to compatibility constraints with the opacus package. However, as shown in Tables \ref{tab:lr_knn_r2_results}, \ref{tab:lr_knn_mae_results}, and \ref{tab:lr_knn_rmse_results}, the raw baseline performance was not very strong. The raw baseline model was trained on the raw 90-dimensional inputs. The DP pipeline settings were $\left(\epsilon, \delta\right)=\left(8.006, 4.65\times10^{-12}\right)$. Both autoencoder pipelines produced 16-dimensional latent encodings.
\begin{center}
\begin{tabular}{lcc}
\toprule
\textbf{Model Pipeline} & \textbf{Test $R^2$} & \textbf{vs. Raw} \\
\midrule
Raw Data & $0.2320$ & $+0.0000$ \\
DP & $0.1870$ & $-0.0450$ \\
Dense AE & $0.0955$ & $-0.1365$ \\
$k$-NN (none) & $0.3285$ & $+0.0965$ \\
$k$-NN (union) & $0.3279$ & $+0.0959$ \\
$k$-NN (intersection) & $0.3277$ & $+0.0957$ \\
\bottomrule
\end{tabular}
\captionof{table}{LR + $k$-NN Sparsified Test $R^2$ Comparison}
\label{tab:lr_knn_r2_results}
\end{center}
\begin{center}
\begin{tabular}{lcc}
\toprule
\textbf{Model Pipeline} & \textbf{Test MAE} & \textbf{vs. Raw} \\
\midrule
Raw Data & $6.5875$ & $+0.0000$ \\
DP & $6.8005$ & $+0.2130$ \\
Dense AE & $7.6206$ & $+1.0331$ \\
$k$-NN (none) & $5.8302$ & $-0.7573$ \\
$k$-NN (union) & $5.8221$ & $-0.7654$ \\
$k$-NN (intersection) & $5.8976$ & $-0.6889$ \\
\bottomrule
\end{tabular}
\captionof{table}{LR + $k$-NN Sparsified Test MAE Comparison}
\label{tab:lr_knn_mae_results}
\end{center}
\begin{center}
\begin{tabular}{lcc}
\toprule
\textbf{Model Pipeline} & \textbf{Test RMSE} & \textbf{vs. Raw} \\
\midrule
Raw Data & $9.5102$ & $+0.0000$ \\
DP & $9.7846$ & $+0.2744$ \\
Dense AE & $9.8688$ & $+1.0070$ \\
$k$-NN (none) & $8.8924$ & $-0.6178$ \\
$k$-NN (union) & $8.8964$ & $-0.6138$ \\
$k$-NN (intersection) & $8.8977$ & $-0.6125$ \\
\bottomrule
\end{tabular}
\captionof{table}{LR + $k$-NN Sparsified Test RMSE Comparison}
\label{tab:lr_knn_rmse_results}
\end{center}
In the above study, the SCRAE was observed to outperform every other pipeline across all metrics, while both DP and the densely connected autoencoder were observed to degrade model performance.\\
Given the poor performance of the raw baseline, the downstream model was upgraded to an XGBoost \cite{chen2016xgboost} regression model to represent a more realistic model deployment scenario. To implement DP with an XGBoost model, the DP-XGBoost package was used. To keep all four pipelines as uniform as possible, DP-XGBoost was used to implement the downstream regression model for all four pipelines. However, privacy mode was turned off for all non-DP pipelines. This ensured that any difference in observed performance was due to whether DP was used, and not due to any differences in implementation across different packages. The settings for the DP pipeline were $\epsilon=0.479199$ per tree; the $\delta$ parameter is not used by DP-XGBoost. Both autoencoder pipelines still emitted 16-dimensional encodings to the downstream model. The results are shown in Tables \ref{tab:xgb_knn_r2_results}, \ref{tab:xgb_knn_mae_results}, and \ref{tab:xgb_knn_rmse_results}.
\begin{center}
\begin{tabular}{lcc}
\toprule
\textbf{Model Pipeline} & \textbf{Test $R^2$} & \textbf{vs. Raw} \\
\midrule
Raw Data & $0.3331$ & $+0.0000$ \\
DP & $-0.0041$ & $-0.3372$ \\
Dense AE & $0.1714$ & $-0.1617$ \\
$k$-NN (none) & $0.3387$ & $+0.0056$ \\
$k$-NN (union) & $0.3393$ & $+0.0062$ \\
$k$-NN (intersection) & $0.3337$ & $+0.0006$ \\
\bottomrule
\end{tabular}
\captionof{table}{XGB + $k$-NN Sparsified Test $R^2$ Comparison}
\label{tab:xgb_knn_r2_results}
\end{center}
\begin{center}
\begin{tabular}{lcc}
\toprule
\textbf{Model Pipeline} & \textbf{Test MAE} & \textbf{vs. Raw} \\
\midrule
Raw Data & $6.1619$ & $+0.0000$ \\
DP & $8.1131$ & $+1.9512$ \\
Dense AE & $7.1880$ & $+1.0261$ \\
$k$-NN (none) & $5.9355$ & $-0.2264$ \\
$k$-NN (union) & $5.9332$ & $-0.2287$ \\
$k$-NN (intersection) & $6.0358$ & $-0.1261$ \\
\bottomrule
\end{tabular}
\captionof{table}{XGB + $k$-NN Sparsified Test MAE Comparison}
\label{tab:xgb_knn_mae_results}
\end{center}
\begin{center}
\begin{tabular}{lcc}
\toprule
\textbf{Model Pipeline} & \textbf{Test RMSE} & \textbf{vs. Raw} \\
\midrule
Raw Data & $8.8618$ & $+0.0000$ \\
DP & $10.8741$ & $+2.0123$ \\
Dense AE & $9.8781$ & $+1.0163$ \\
$k$-NN (none) & $8.8245$ & $-0.0373$ \\
$k$-NN (union) & $8.8207$ & $-0.0411$ \\
$k$-NN (intersection) & $8.8579$ & $-0.0039$ \\
\bottomrule
\end{tabular}
\captionof{table}{XGB + $k$-NN Sparsified Test RMSE Comparison}
\label{tab:xgb_knn_rmse_results}
\end{center}
In this more realistic study the raw baseline is now on par with every variant of the $k$-NN sparsified SCRAE across all performance metrics. Notably, the predictive performance of the DP pipeline collapsed completely after switching to the XGBoost model, even with a XGBoost-specialized implementation.\\
To quantify the computational benefit of each variant of the sparsified Graph-Laplacian loss, a separate benchmark study was run on the same data, using the same SCRAE architecture and the same number of nearest target neighbors. Four modes were compared: the unsparsified dense Graph-Laplacian and the three sparsified variants with no symmetrization, union symmetrization, and intersection symmetrization. The benchmark was executed on a single NVIDIA GeForce GTX 1650 GPU with 4 GB of VRAM, with batch sizes starting at 128 and doubling until reaching 4,096. For each batch size, median timings were computed over 8 measured iterations after 3 warm-up iterations. Notably, the dense objective failed entirely due to an Out-of-Memory (OOM) error at $B=4,096$ on the benchmark hardware, while all three sparsified variants remained tractable.\\
Immediately, it was observed that the $k$-NN sparsification had the intended effect of reducing the number of edges in the latent target-neighbor graph, thereby reducing computational complexity compared with the dense Graph-Laplacian. This is shown in Figure \ref{fig:edge_count}.
\begin{center}
\includegraphics[scale=0.45]{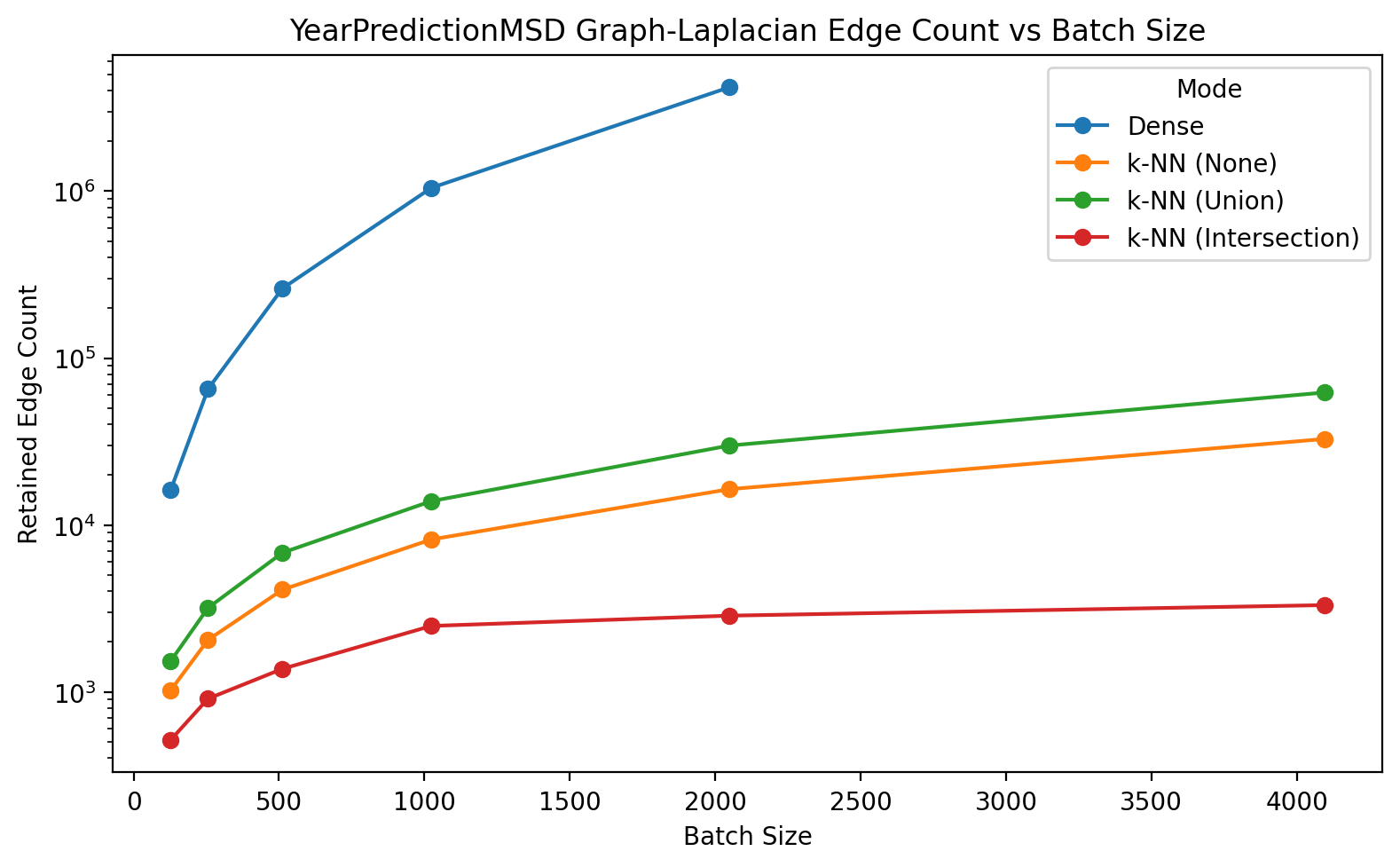}
\captionof{figure}{Latent Graph Edge count vs. Batch Size}
\label{fig:edge_count}
\end{center}
The end-to-end training step benchmark included the encoder forward pass, all loss function calculations, backward pass, and Adam update. At batch size $B=1,024$, the dense objective required $175.6$ ms per step and $2.53$ GB of peak allocated GPU memory, whereas the sparsified variants required only $43.1$--$62.8$ ms and $40$--$80$ MB\@. At $B=2,048$, the dense objective became prohibitive, requiring $5.55$ s per step and approximately $10.0$ GB of peak allocated memory, whereas the sparsified variants required only $18.4$ to $42.8$ ms and $47$ to $135$ MB\@. This corresponds to training-step speedups ranging from $129.7\times$ to $301.8\times$. Figures \ref{fig:train_step_time} through \ref{fig:train_step_speedup} visualize specific aspects of the observed scaling behavior.
\begin{center}
\includegraphics[scale=0.45]{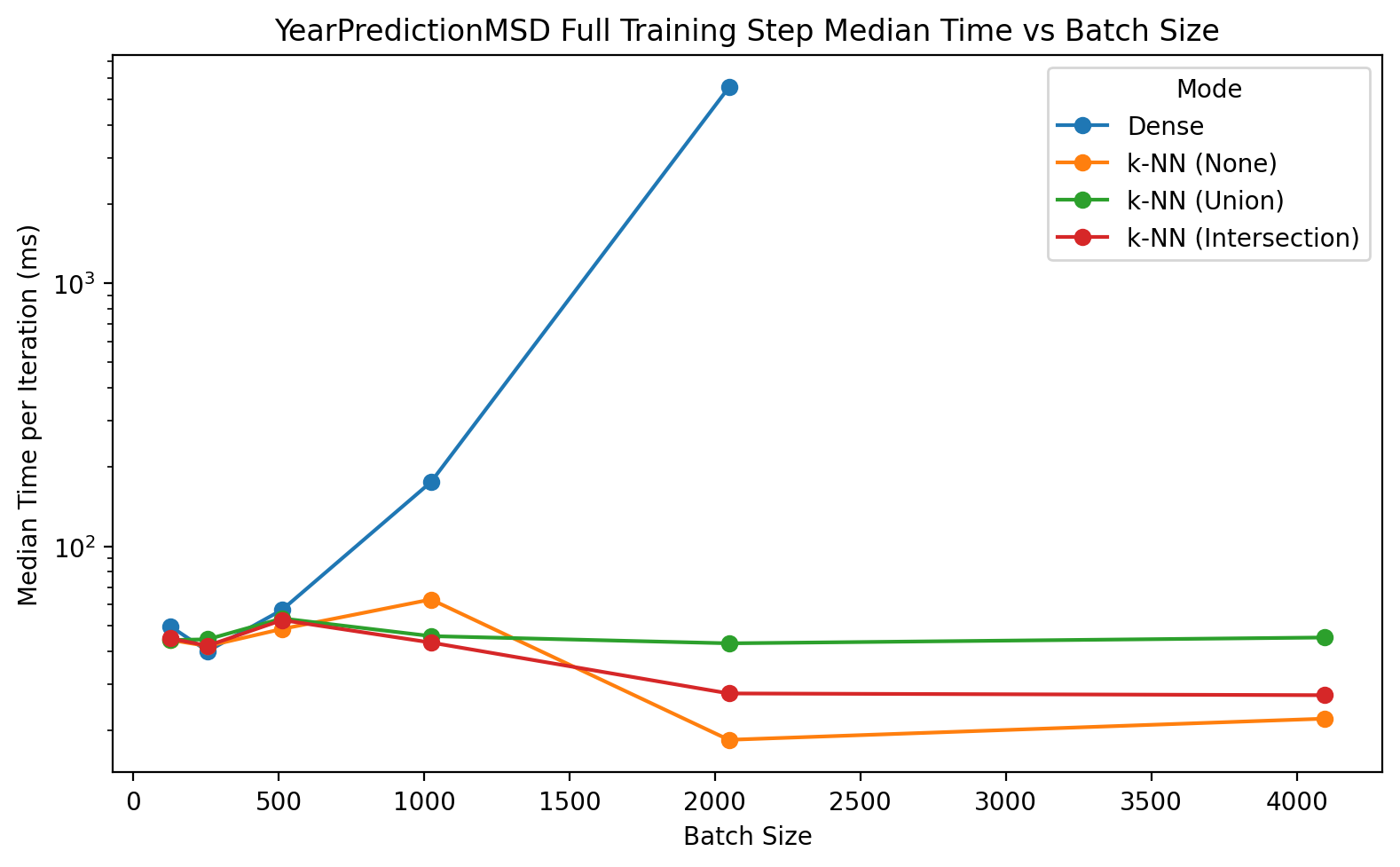}
\captionof{figure}{Train Step Time vs. Batch Size}
\label{fig:train_step_time}
\end{center}
\begin{center}
\includegraphics[scale=0.45]{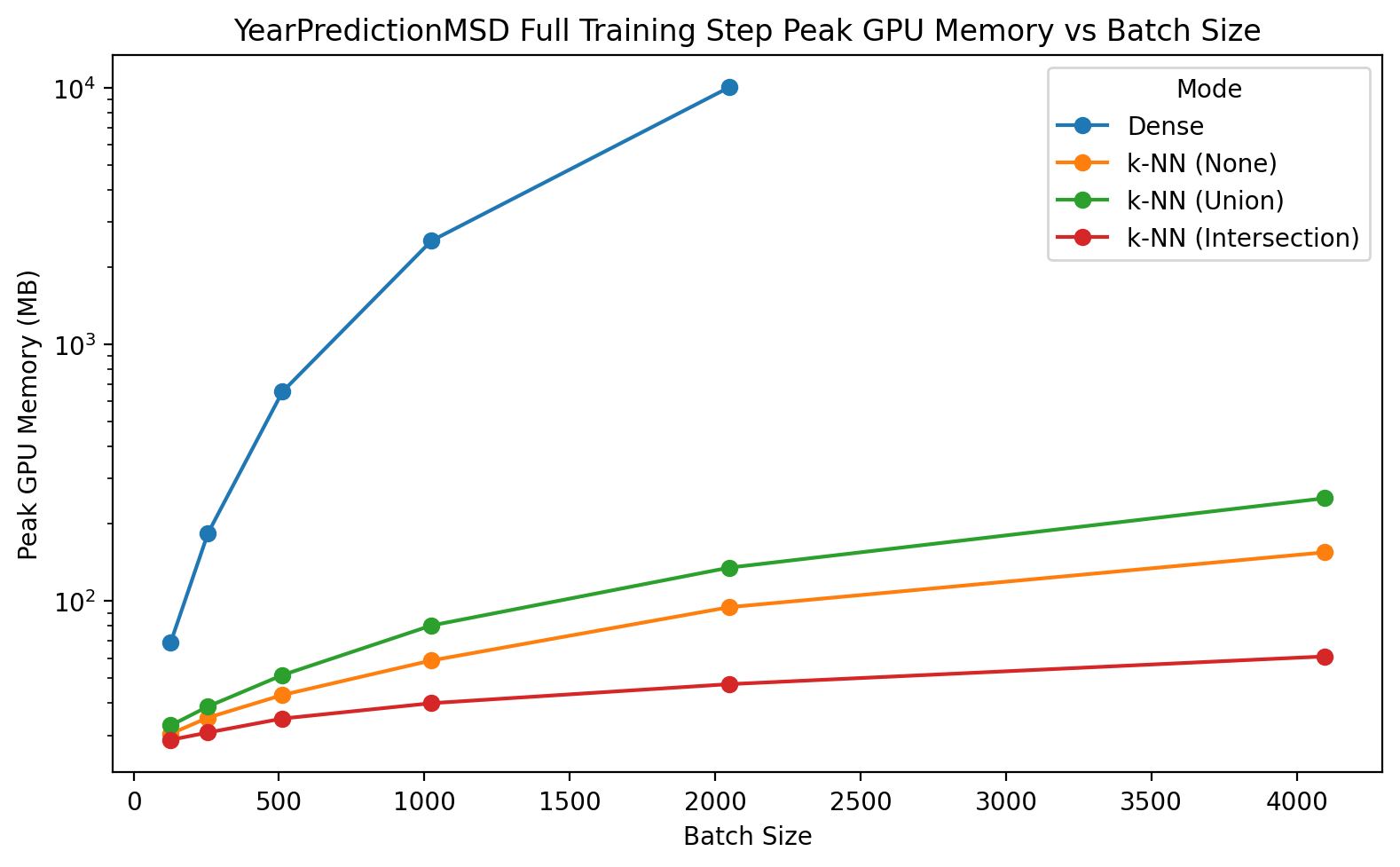}
\captionof{figure}{Train Step Peak Memory vs. Batch Size}
\label{fig:train_step_peak_mem}
\end{center}
\begin{center}
\includegraphics[scale=0.45]{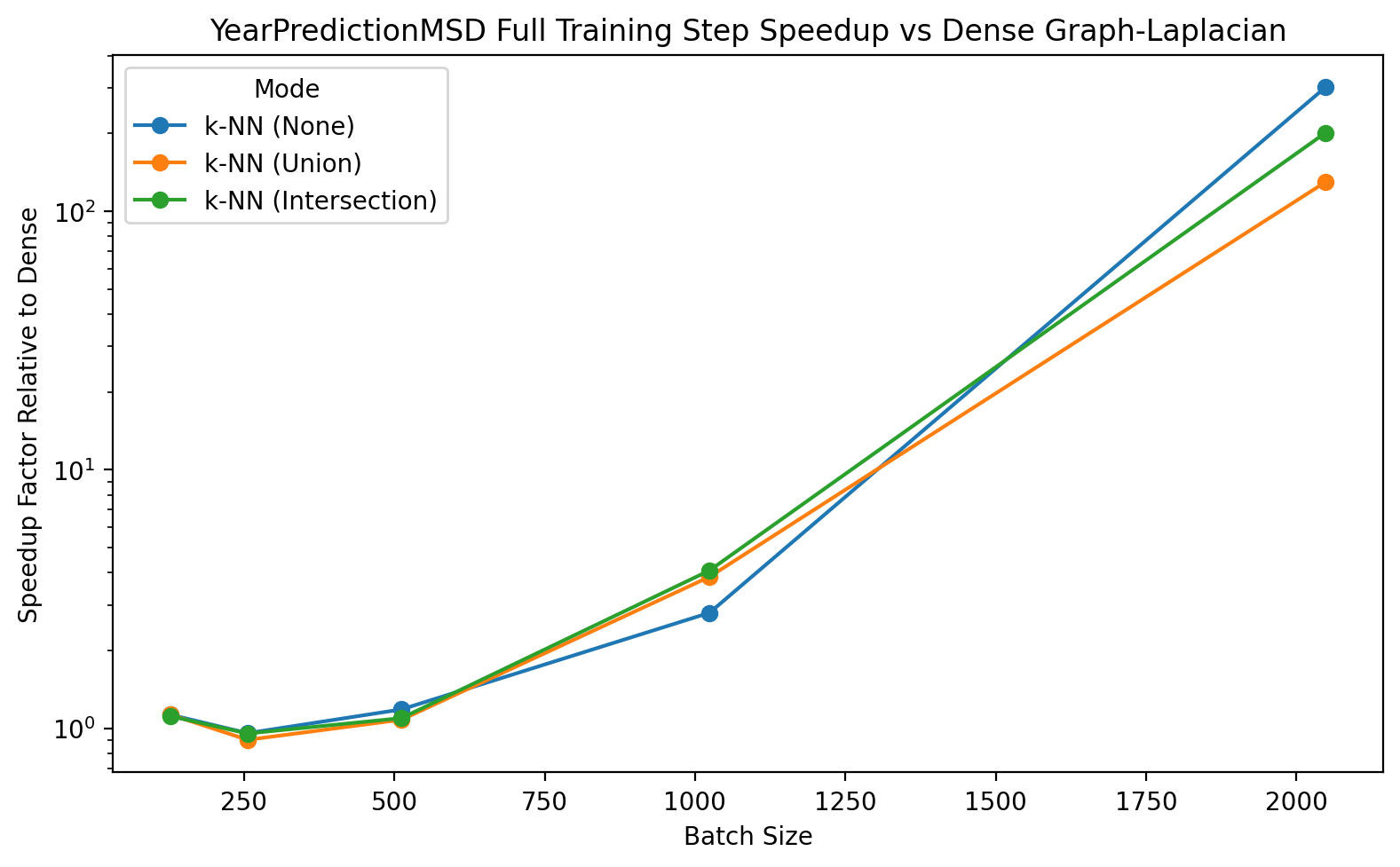}
\captionof{figure}{Speedup Factor Relative to Dense vs. Batch Size}
\label{fig:train_step_speedup}
\end{center}

\subsection{Training Stability with Noisy Targets}

For noisy targets, a slightly larger $\sigma$ should be used to increase smoothing. When using $\widetilde{\mathcal{L}}_{\text{Lap}}$, use a smaller value for $k$, as incorporating more nearest target neighbors allows more target noise to affect the smoothing process.\\
In any case, the size of $\sigma$ should be limited. If $\sigma$ is too large, then edge weights are allowed to be large for too many edges. This may lead to representation collapse in which all $\Psi_i$ are either identical or too similar to retain predictive utility. It should be noted that representation collapse is resisted and prevented by the presence of reconstruction and prediction losses in the multi-objective loss function, as it would render both reconstruction and prediction impossible. In cases where targets are particularly noisy, a hyperparameter search over all combinations of $\lambda$ and $\sigma$ may be used to achieve the desired result.\\
In difficult training scenarios with noisy targets, the Spearman rank correlation, $\rho$, is a useful diagnostic to track monotonicity in the similarity of latent representations, $\norm{\Psi_i-\Psi_j}$, and target similarity, $\norm{\mathbf{y}_i - \mathbf{y}_j}$. This should increase reliably as training progresses, as seen in Figure \ref{fig:spearman}.
\begin{center}
\includegraphics[scale=0.45]{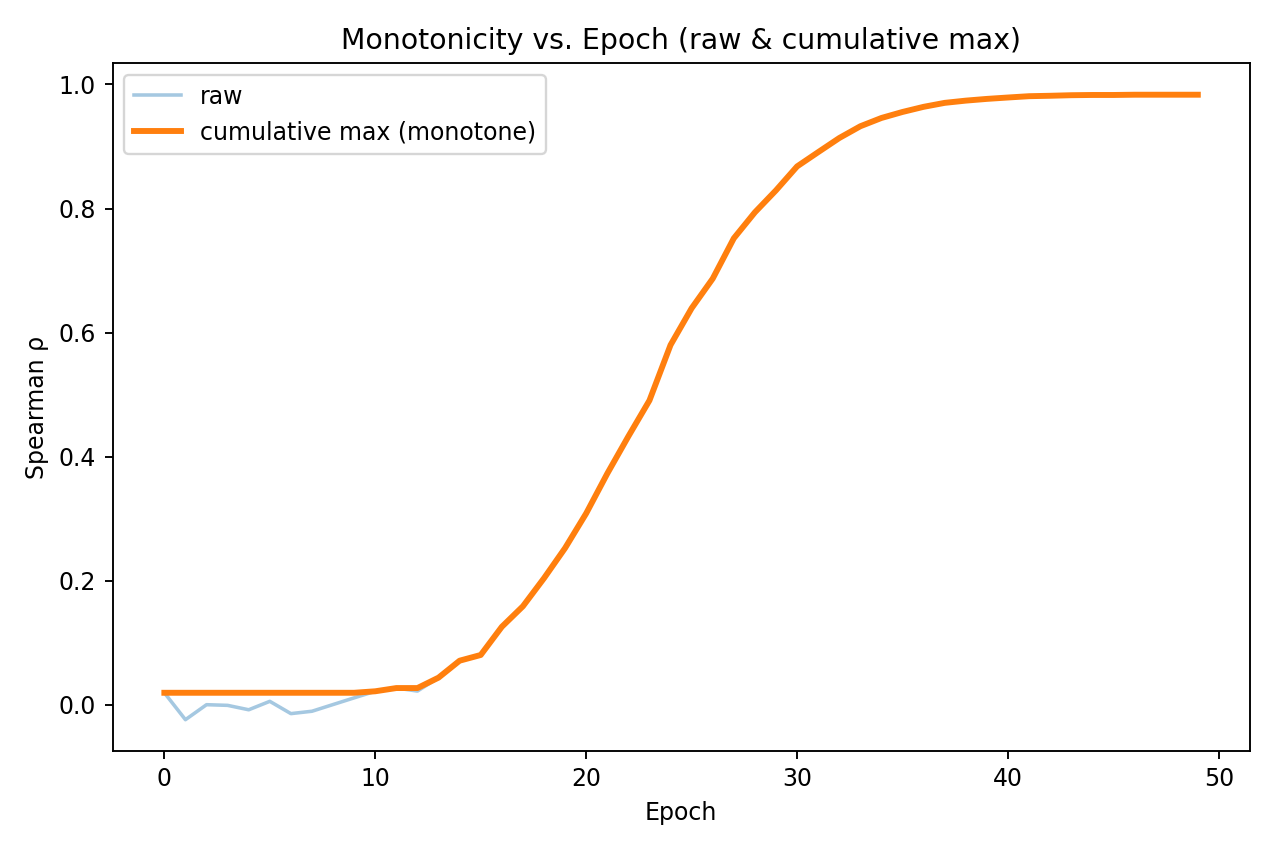}
\captionof{figure}{Spearman rank correlation, $\rho$, vs.\ training epochs}
\label{fig:spearman}
\end{center}
Some initial instability may be observed in early training epochs. However, if the algorithm is working as intended, the Spearman rank correlation is expected to become strictly increasing as training continues.\\
A similar diagnostic tool involves fitting a $k$-NN model in the latent space and measuring the out-of-sample $R_{k\text{-NN}}^2$ statistic. Two behaviors are expected from $R_{k\text{-NN}}^2$. First, this is also expected to become strictly increasing as training progresses, as shown in Figure \ref{fig:knn_r2}.
\begin{center}
\includegraphics[scale=0.45]{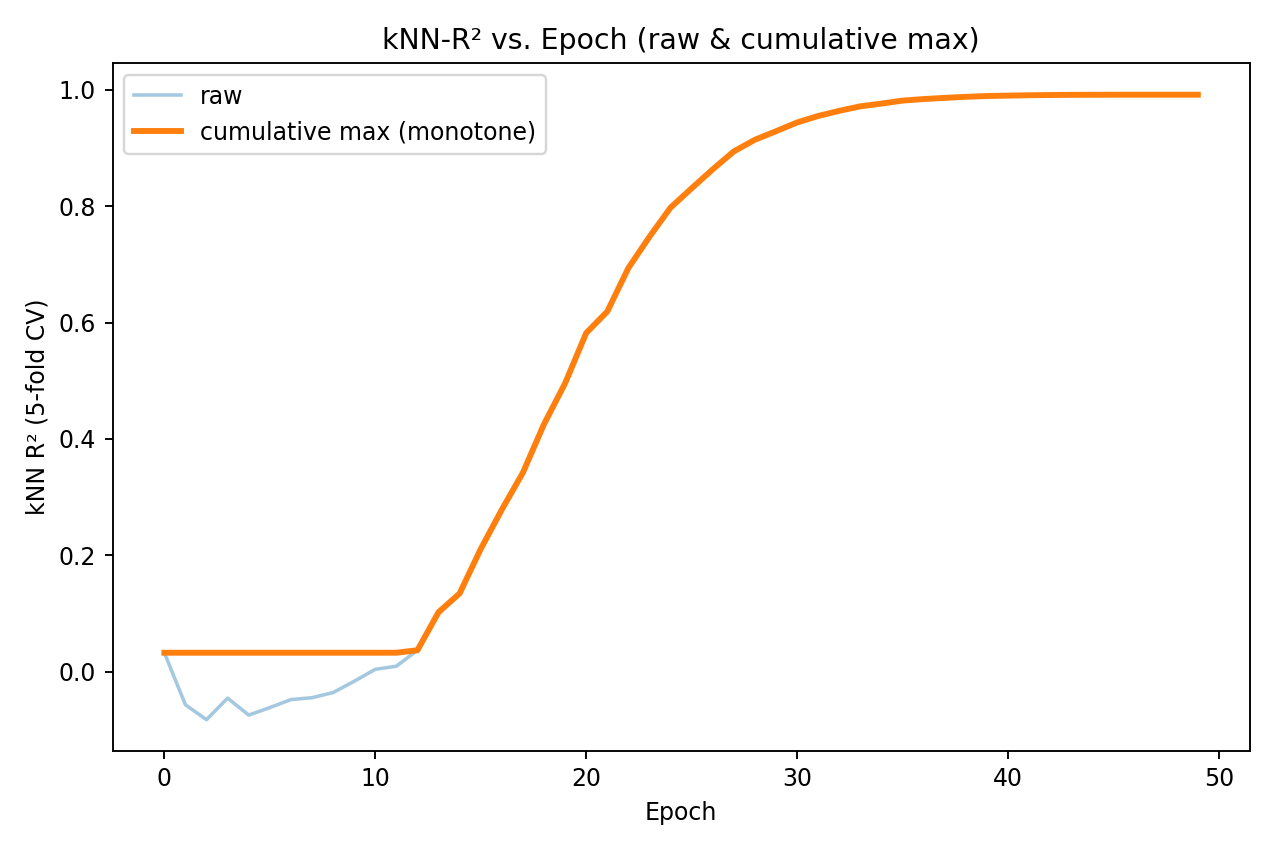}
\captionof{figure}{5-fold CV $R_{k\text{-NN}}^2$ vs.\ training epochs}
\label{fig:knn_r2}
\end{center}
Second, as training progresses, $R_{k\text{-NN}}^2$ is expected to approach the $R^2$ statistic of the downstream ML model, as seen in Figure \ref{fig:knn_r2_vs_hub_r2}, which was generated using different training data. In this case, the $R^2$ statistics were calculated on a fixed validation set.
\begin{center}
\includegraphics[scale=0.45]{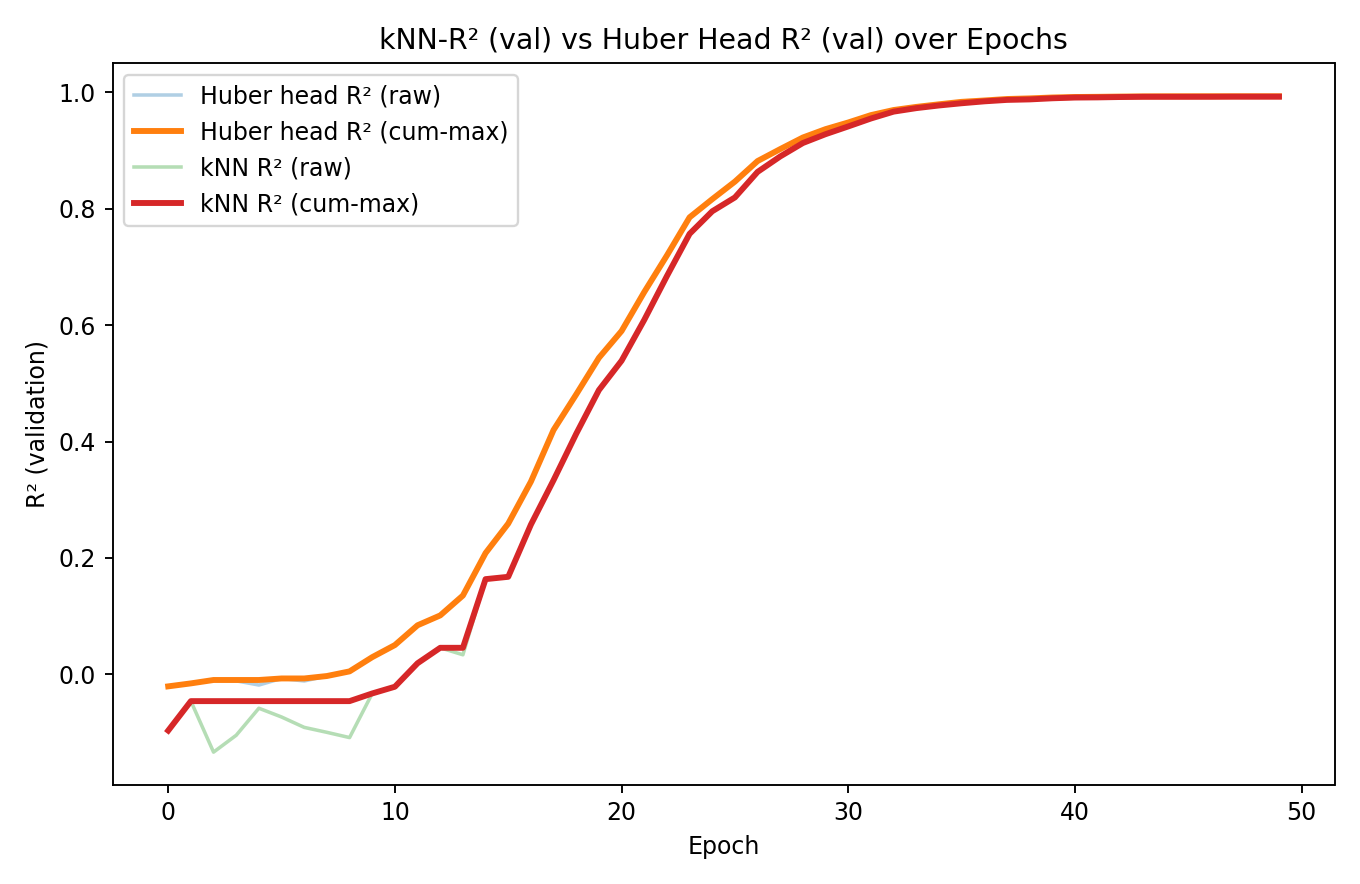}
\captionof{figure}{Convergence of $R_{k\text{-NN}}^2$ and $R_{\text{Huber}}^2$ vs.\ training epochs}
\label{fig:knn_r2_vs_hub_r2}
\end{center}
Another useful training diagnostic that verifies the encoder is behaving as intended is to check regressor calibration in the downstream model trained on the encoder's output. A calibrated regressor should have the property that, for any predicted value, $\hat{y}$, the average true target value among all cases with predictions near $\hat{y}$ is approximately $\hat{y}$. More compactly, the diagnostic checks that the following holds for the downstream regressor.
$$\mathbb{E}\left[Y \text{ }|\text{ } \hat{Y} \approx \hat{y}\right] \approx \hat{y}$$
Given that $|\text{rng }Y| \geq \aleph_0$, conditioning on a single prediction value, $\hat{y}$, is impossible. Therefore, the above expectation is approximated by binning. In the case of the example visualized in Figure \ref{fig:cal}, predictions were binned in deciles to arrive at ten bins of equal frequency. Then, for each bin, $b$, the mean was computed for both predicted and true target values. The two means for each bin were then compared to visualize whether they were roughly equal. Ideally, the points, $\left(\bar{\hat{y}}_b, \bar{y}_b\right)$, should lie on the identity line. The results of this calibration diagnostic are shown in Figure \ref{fig:cal}.
\begin{center}
\includegraphics[scale=0.55]{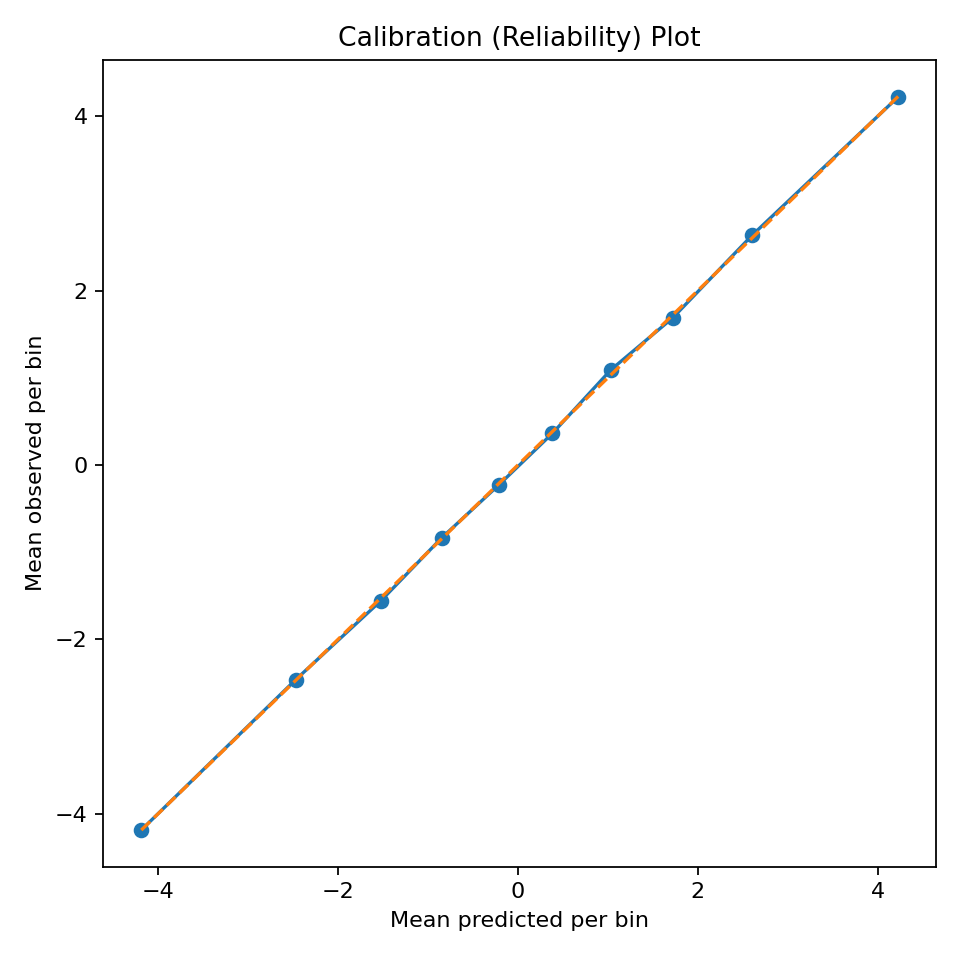}
\captionof{figure}{Observed vs. Predicted Means}
\label{fig:cal}
\end{center}
In some extreme cases, collapse may still occur. In these rare cases, an alternative representation loss function that is resistant to this effect may be used. One such alternative is discussed in the next section.

\section{The Target-Aware Contrastive Loss Function: InfoNCE Adapted to Continuous Targets}

The minute risk of latent collapse with the Graph-Laplacian loss proposed in the previous section arises because this loss function enforces a similar latent representation for observations with similar targets, but imposes no particular behavior for observations with dissimilar targets. Thus, the only force applied within the latent space is an attractive force. It should be noted that this is also true of the Center loss described in section 4.\\
In cases where latent collapse is occurring, it is avoided by applying a representation loss which enforces similar latent representations for observations with similar target values, while also enforcing dissimilar latent representations for observations with dissimilar target values. In the case of classification, such an objective function is found in the Information Noise-Contrastive Estimation (InfoNCE) loss \cite{oord2019representationlearningcontrastivepredictive}, given by the following formula.
\begin{equation}
\mathcal{L}_{\text{NCE}}=-\frac{1}{N}\sum_{i=1}^N \text{ln}\frac{\sum_{j \neq i}\text{ind}\left[y_i=y_j\right]e^{s_{ij}}}{\sum_{k \neq i}e^{s_{ik}}}
\end{equation}
Here, $\text{ind}\left[\cdot\right]$ is the indicator function, defined as,
\begin{dmath}
\text{ind}\left[\text{ }p\text{ }\right] = 
    \begin{cases}
        1 & p \text{ is true}\\
        0 & p \text{ is false}
    \end{cases}
\end{dmath}
and
\begin{dmath}
s_{ij} = \frac{\Psi_i^T\Psi_j}{\tau\norm{\Psi_i}\norm{\Psi_j}}
\end{dmath}
is the cosine similarity of the latent representations $\Psi_i$ and $\Psi_j$ with temperature $\tau > 0$. This acts as a softmax classifier over neighbors, as the normalized exponentiated similarities form a categorical distribution. For a given observation, $\mathbf{x}_i$, it is clear that $\mathcal{L}_{\text{NCE}}$ is the average of log-probabilities over the set of observations with the same class label. It is worth noting that this loss may be substituted for the less computationally expensive Center loss in classification in the event of latent collapse. However, the requirement of defining a discrete set of same-class positives is ill-suited to the continuous targets of regression. Just as with the Graph-Laplacian loss, some continuous notion of ``soft matching'' or similarity is needed to identify observations with similar target values without requiring that their target values be equal. Using hard bins to group observations with similar targets is not ideal, as hard bins would destroy resolution while injecting the need for discretization hyperparameters.\\
The desired notion of similarity is achieved by defining a similarity kernel for any two observations, $\gamma_{ij}$. This can be done with the same formula given in Equation \ref{eq:reg_similarity_kernel}. Replacing the discretized positive-match indicator with the continuous similarity kernel yields the following regression-adapted InfoNCE loss function.
\begin{dmath}
\mathcal{L}_{\text{R-NCE}} = -\frac{1}{N}\sum_{i=1}^N\text{ln }\frac{\sum_{j \neq i} \gamma_{ij}e^{s_{ij}}}{\sum_{k \neq i} e^{s_{ik}}}
\end{dmath}
Due to the increased complexity and computational requirements, these contrastive loss functions are offered as a fallback strategy in the event of latent collapse with the Center loss or Graph-Laplacian loss. All diagnostic and calibration tests described in the previous section can also be used to monitor training with the regression-adapted InfoNCE loss.\\
To demonstrate the effectiveness of these contrastive loss functions, two studies were run, one with a classification model and one with a regression model. These studies are analogous to those found in Sections 4 and 5 of this paper. For the classification study, the Fashion-MNIST dataset \cite{xiao2017fashionmnist} was used, as it is widely regarded as presenting a more challenging classification task than the original MNIST dataset. The downstream classifier was a Logistic Regression. The settings for the DP pipeline were $\left(\epsilon, \delta\right)=\left(7.992, 2.78\times10^{-10}\right)$. The dimensionality of the raw dataset was 784, and the encoding was 2-dimensional for the case of both the SCRAE and the densely connected autoencoder. The results are shown in Tables \ref{tab:contrast_class_acc_results} and \ref{tab:contrast_class_f1_results}.
\begin{center}
\begin{tabular}{lcc}
\toprule
\textbf{Model Pipeline} & \textbf{Test Accuracy} & \textbf{vs. Raw} \\
\midrule
Raw Data & $83.77\%$ & $+0.00\%$ \\
DP & $83.63\%$ & $-0.14\%$ \\
Dense AE & $61.50\%$ & $-22.13\%$ \\
SCRAE & $88.21\%$ & $+4.58\%$ \\
\bottomrule
\end{tabular}
\captionof{table}{Fashion-MNIST Test Accuracy Comparison}
\label{tab:contrast_class_acc_results}
\end{center}
\begin{center}
\begin{tabular}{lcc}
\toprule
\textbf{Model Pipeline} & \textbf{Test Macro-F1} & \textbf{vs. Raw} \\
\midrule
Raw Data & $0.8364$ & $+0.0000$ \\
DP & $0.8354$ & $-0.0010$ \\
Dense AE & $0.6110$ & $-0.2254$ \\
SCRAE & $0.8822$ & $+0.0468$ \\
\bottomrule
\end{tabular}
\captionof{table}{Fashion-MNIST Test Macro-F1 Comparison}
\label{tab:contrast_class_f1_results}
\end{center}
The SCRAE encoder was again observed to outperform the DP, densely connected autoencoder, and raw baseline modeling pipelines, while achieving the same level of compression observed in the study the original MNIST dataset from Section 4.\\
For the regression study, the E2006-tfidf dataset \cite{Kogan2009PredictingRF} was used. This dataset consists of TF-IDF \cite{sparckjones1972tfidf} representations of SEC 10-K financial reports from publicly traded companies. The regression task is then to predict volatility in each company's stock price based on information in these filings. These TF-IDF representations are extremely high-dimensional and sparse, with a raw dimensionality of 150,360. This can cause the distances between paired latent vectors to appear uniform, leading to excessive latent collapse under the Graph-Laplacian loss.\\
The downstream regressor was a Linear Regression model. The raw baseline model was trained on the 150,360-dimensional TF-IDF representations. The settings for the DP pipeline were $\left(\epsilon, \delta\right)=\left(7.998, 3.86\times10^{-9}\right)$. Both autoencoder pipelines emitted latent representations with dimensionality 64. The results are shown in Tables \ref{tab:contrast_reg_r2_results} and \ref{tab:contrast_reg_mae_results}.
\begin{center}
\begin{tabular}{lcc}
\toprule
\textbf{Model Pipeline} & \textbf{Test $R^2$} & \textbf{vs. Raw} \\
\midrule
Raw Data & $0.5218$ & $+0.0000$ \\
DP & $0.3578$ & $-0.1640$ \\
Dense AE & $0.2562$ & $-0.2656$ \\
SCRAE & $0.5226$ & $+0.0008$ \\
\bottomrule
\end{tabular}
\captionof{table}{E2006 Test $R^2$ Comparison}
\label{tab:contrast_reg_r2_results}
\end{center}
\begin{center}
\begin{tabular}{lcc}
\toprule
\textbf{Model Pipeline} & \textbf{Test MAE} & \textbf{vs. Raw} \\
\midrule
Raw Data & $0.2280$ & $+0.0000$ \\
DP & $0.2517$ & $+0.0237$ \\
Dense AE & $0.3154$ & $+0.0874$ \\
SCRAE & $0.2288$ & $+0.0008$ \\
\bottomrule
\end{tabular}
\captionof{table}{E2006 Test MAE Comparison}
\label{tab:contrast_reg_mae_results}
\end{center}
In this study, the observed degradation of the DP pipeline is quite severe. In a real-world ML project, this level of predictive degradation would greatly reduce, if not completely nullify, the model's value. In contrast, the SCRAE preserved baseline predictive utility while achieving a compression ratio of approximately 0.00043, reducing storage and transmission requirements by approximately 99.96\%. Thus, the SCRAE pipeline would remain viable and deliver positive ROI for a real-world project.\\
Together, the classification and regression experiments presented in this section confirm that the contrastive loss functions introduced here are effective fallback mechanisms when the Center loss or Graph-Laplacian loss fail to prevent latent collapse. More broadly, they demonstrate that the multi-objective framework introduced in Section 4 is sufficiently extensible to accommodate a family of representation objectives, each enforcing the same geometric intent---task-aligned separation of the latent space---through different mechanisms. The practitioner therefore has a principled hierarchy of representation losses to draw from: the Center loss and Graph-Laplacian loss are recommended as defaults for classification and regression, respectively, given their lower computational cost and simpler optimization landscape; the contrastive losses, $\mathcal{L}_{\text{NCE}}$ and $\mathcal{L}_{\text{R-NCE}}$, are reserved for cases where collapse persists under those defaults. In either case, the downstream effect is the same: a latent representation that is geometrically organized by the supervised objective, informationally compressed to the point of non-invertibility, and sufficiently rich to support accurate ML inference. With this full objective scaffolding now established for both classification and regression, the following section turns from the encoder itself to the broader system architecture within which it operates.

\section{The VEIL Architecture}

This section outlines the reference architecture and operational deployment model for an ICA ppML system. The architecture is deliberately designed to enforce strict separation between trusted environments---where sensitive data reside---and untrusted or semi-trusted external compute environments, where large-scale model training, optimization, and inference occur.\\
By ensuring that all sensitive signals remain confined to trusted, controlled infrastructure, the proposed deployment pattern provides robust privacy guarantees while allowing organizations to leverage the scalability of modern cloud platforms for downstream model training. The VEIL architecture is built around a three-tier trust model.\\
First is the \textbf{Trusted Tier} (Source Environment), which contains raw data and sensitive attributes. The encoder, here referred to as the Vector-Encoded Information Layer (VEIL), and its parameters are deployed and hosted within this tier at the relevant data source. All feature engineering, extraction, compression, and anonymization are performed at this layer. Within this tier, strict governance, residency, auditing, and compliance controls are enforced.\\
Second is the \textbf{Semi-Trusted Tier} (Training Environment), which receives only non-invertible, informationally-compressed and anonymized latent representations. This tier is where training of the downstream ML model occurs using cloud-scale compute resources. This tier never sees raw data, gradients over sensitive features, or identifiable attributes.\\
The third and final tier is the \textbf{Application Tier} (Inference Environment), which houses the inference runtime. This is the production environment in which model predictions are executed using the same latent-only interface from the semi-trusted tier. This environment can run in the cloud or on-premises and maintains the same trust boundaries established in training.\\
This structure ensures that sensitive information is \textit{never} present outside the Source Environment, ensuring privacy regardless of whether the cloud is compromised, misconfigured, or externally attacked.

\subsection{The Source Environment (Trusted)}

This environment contains the systems and data repositories already governed by the client organization's internal policies, and may be a Virtual Private Cloud (VPC), on-premises datacenter, secure enclave, private subnet, or database-specific execution environment. This tier houses all raw data repositories, such as any databases, data lakes, or data warehouses storing PII, financial data, PCI, PHI, operational data, internal logs, etc. Under this architecture, these raw records are never exported outside of this environment.\\
The multi-level, multi-objective autoencoder serves as the Vector-Encoded Information Layer (VEIL) and is deployed adjacent to the data. At the time of a call to send model inputs to a deployed ML model, raw records are engineered into the raw ML features and then passed through the VEIL\@. The VEIL removes any sensitive or identity-linked features by informationally compressing and anonymizing the raw feature vectors into non-invertible, informationally minimal latent vectors. The latent representation constraints ensure non-invertibility while maximizing predictive utility. All encoding operations are logged for compliance and auditing.\\
This environment includes an access control and governance layer that enforces Role-Based Access Control (RBAC), Attribute-Based Access Control (ABAC), token-scoped credentials, encryption at rest, audit trails, and data residency restrictions. In addition to these traditional access and governance components, this layer would strictly enforce that only ICA-encoded latent vectors may cross the trust boundary into other environments.
\begin{center}
\includegraphics[scale=0.16]{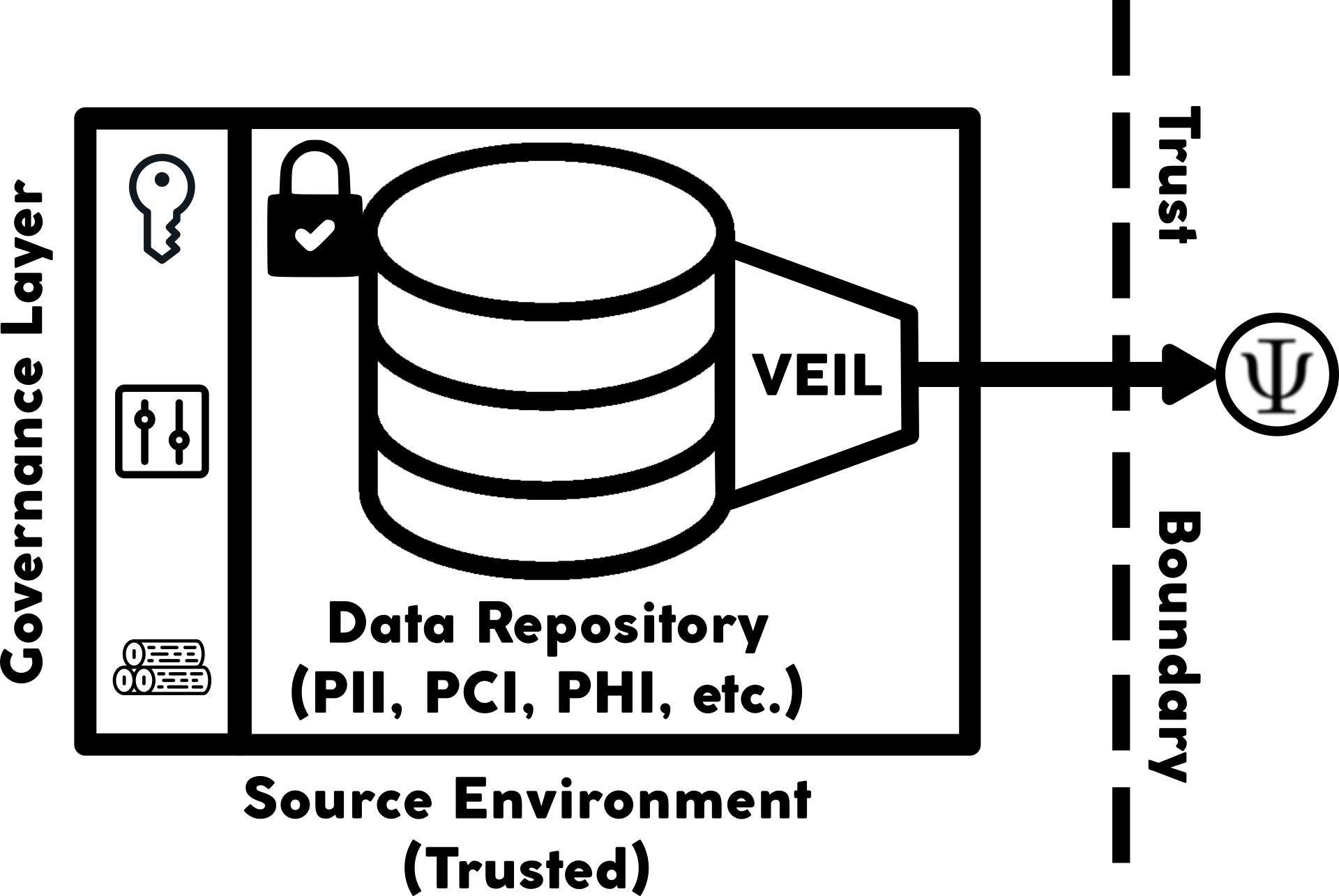}
\captionof{figure}{Illustration of the Source Environment}
\label{fig:VEIL_source_env}
\end{center}
Encodings are transmitted from the trusted environment to the cloud using mutually authenticated channels (TLS 1.3 or better), customer-controlled API keys or IAM roles, optional VPN or private link peering, or transport-level integrity checks. While no sensitive or identifiable information is ever transmitted, these additional measures are required to prevent other forms of attack, such as injection \cite{ayub2024embeddingbasedclassifiersdetectprompt} or poison pill attacks \cite{guo2024poisoningpillcircumventingdetection}.

\subsection{The Cloud Training Environment (Untrusted or Semi-Trusted)}

This environment provides elastic compute and storage capabilities for training. Under this architecture, all training is on non-sensitive latent vectors. All standard ML frameworks (PyTorch, TensorFlow, JAX, scikit-learn, etc.) expect vector or tensor inputs. Therefore, all these frameworks and training systems operate normally when ingesting the latent vectors for training. Since this environment is only ever exposed to non-invertible latent vectors, there is no exposure to sensitive gradients. This, in turn, means there is no risk of gradient-based inversion attacks \cite{zhu2019deepleakagegradients}, eliminating the need to implement DP-SGD, fine-tune $\epsilon$ and $\delta$ values, and track privacy budgets.\\
In this environment, latent vectors can also be stored for periodic model retraining. The recommended mode for storing latent vectors is a vectorDB solution (e.g., Milvus, Pinecone, Qdrant, or the pgvector plugin for PostgreSQL databases). However, because vectors can be batched into matrices, latent vectors can be stored in any database solution that supports array-like data. Some ML training or deployment solutions, such as Amazon SageMaker, require model inputs to be staged in blob storage, such as Amazon S3, prior to model ingestion. Since any storage in this environment contains only irreversibly anonymized latent vectors, these artifacts contain no PII, identifiers, or reconstruction pathways. Therefore, any storage of latent vectors for operational necessity or convenience creates no risk of sensitive data exposure.
\begin{center}
\includegraphics[scale=0.13]{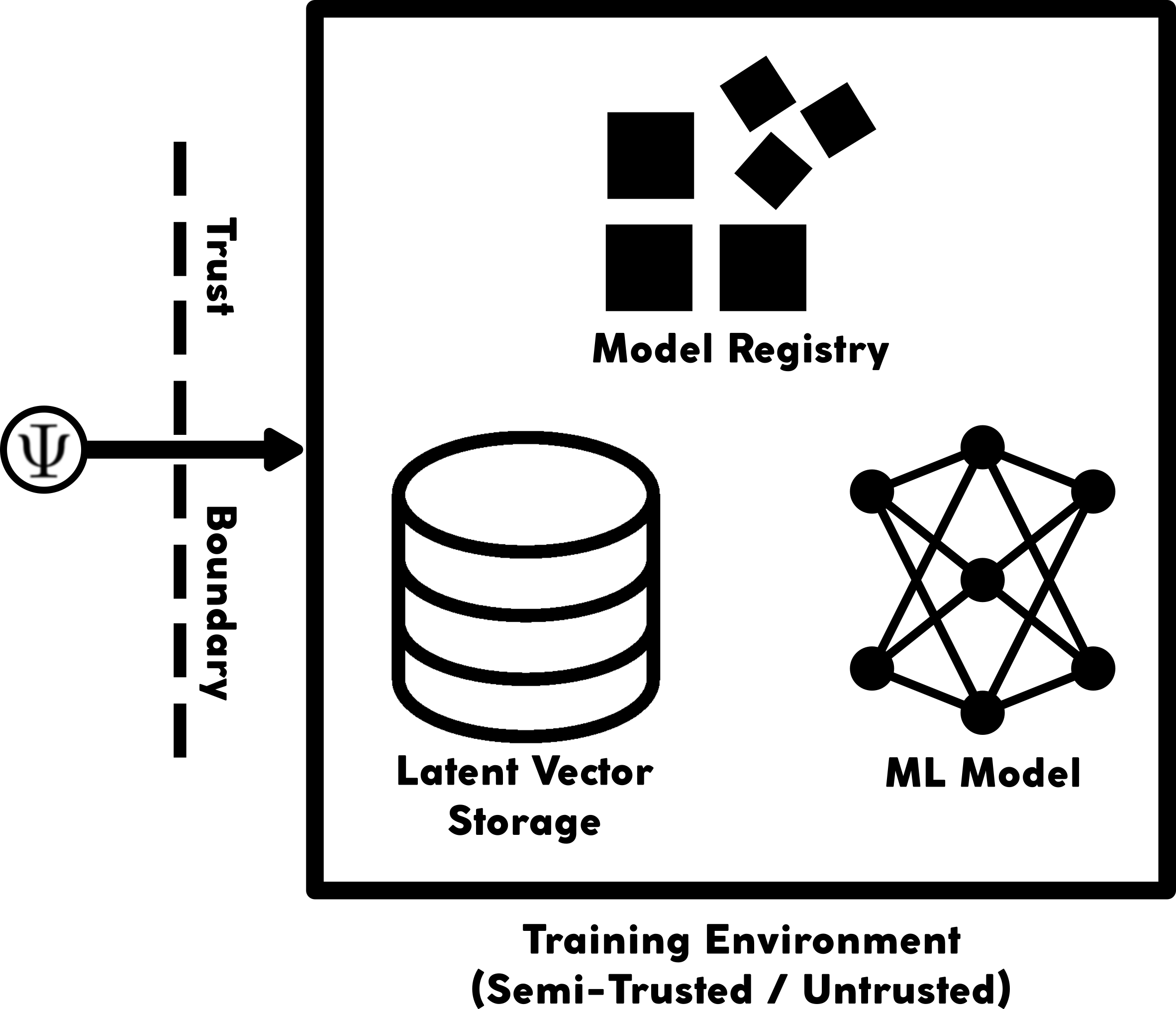}
\captionof{figure}{Illustration of the Training Environment}
\label{fig:VEIL_train_env}
\end{center}
The training environment may also host a model registry that stores trained models and versioned artifacts. Again, these do not create any risk of data exposure, as these models have never been exposed to any sensitive data or gradients over sensitive data. These models operate purely on ICA latent input vectors and therefore cannot contain any sensitive internal representations tied to sensitive data or their distributions.

\subsection{The Application / Inference Environment (Untrusted)}

During inference, new raw data are processed at the source to engineer input features, which are then passed through the VEIL to transform the input vectors into latent vectors via ICA\@. Only latent vectors are sent to the Inference Environment, which may exist on-premises, in the cloud, inside an edge device, or as part of a zero-trust integration with business applications. Predictions are then consumed by the appropriate system or user, depending on the specific use case. This preserves identical trust boundaries between training and inference, meaning the deployed model never accepts sensitive data as input and all downstream systems are shielded from it.
\begin{center}
\includegraphics[scale=0.1]{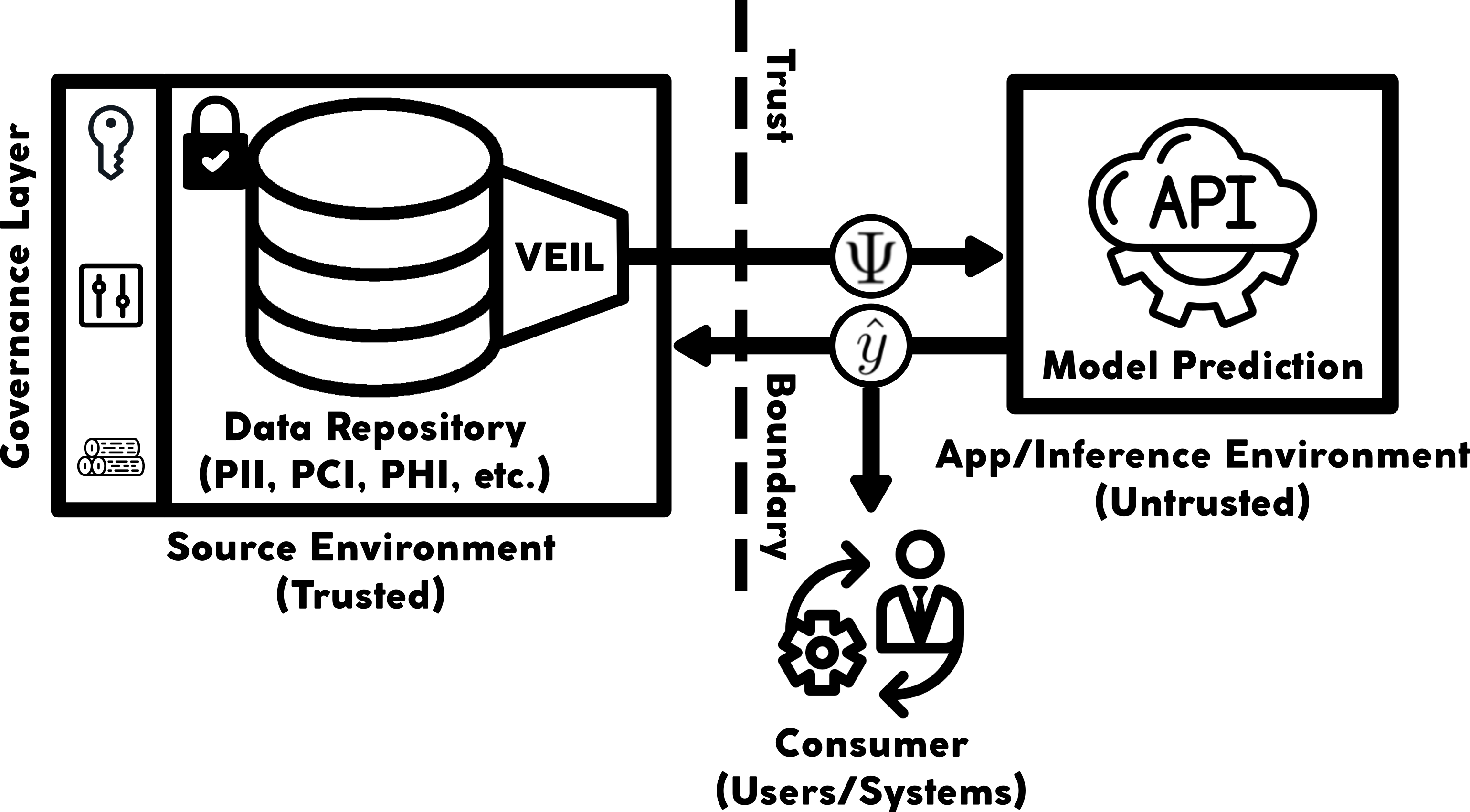}
\captionof{figure}{Illustration of the Application/Inference Environment}
\label{fig:VEIL_inf_env}
\end{center}
Model predictions are also returned to the Source Environment, enabling counterfactual analyses within this trusted environment. This supports regulatory compliance for use cases that require transparency and explainability in decision-making, such as an adverse action report in credit decision-making. This will be discussed in detail in Section 11. For a complete diagram of the VEIL architecture, see Figure \ref{fig:VEIL_architecture_full} in Appendix A.

\subsection{Deployment Patterns}

The VEIL architecture supports multiple deployment patterns. This section will describe some of the primary broad variants that are both possible and likely to be needed.\\
The first deployment pattern is well-suited to organizations in highly regulated industries that are extremely averse to risk and prefer to keep their trusted Source Environment on-premises, while needing access to the elastic storage and compute solutions offered by modern cloud platforms for training and deploying ML/AI models. These organizations are likely to be in the finance and healthcare sectors, thus creating regulatory constraints on data residency or sovereignty, and on how data are used. In this pattern, the Source Environment is on-premises, but the Training and Inference Environments are located on an enterprise cloud platform, such as AWS, Azure, or GCP.\\
The second deployment pattern is the most secure and well-suited to organizations in the defense and intelligence sectors, which handle highly classified data and strict data sovereignty requirements. The VEIL architecture offers these organizations a solution that requires no external connectivity, as all environments will be fully on-premises and air-gapped. However, the environments are still fully partitioned from each other to limit internal exposure, as a precaution against any breach of the Training or Inference Environments.\\
The third deployment pattern discussed in this section involves a hybrid or multicloud deployment. This deployment pattern is applicable to large, multi-regional enterprises that need ML training or inference across multiple cloud regions, cloud providers, or multiple vendor environments. These organizations will need a solution that complies with all regional regulations governing sensitive data, is cost-optimized, is compatible with distributed training workloads, and reduces vendor lock-in while maintaining privacy guarantees.\\
One real-world example of how this deployment pattern is useful can be found in a multinational financial services organization, such as a large credit reporting agency. Such an organization would provide ML/AI solutions for credit decisioning and identity fraud detection in the context of new account openings, among other services, to financial institutions across multiple regions. Providing these types of services relies heavily on a variety of highly sensitive data sources on consumers and their identities, finances, repayment history, income, employment history, real property ownership, residential history, and more. Clearly, these sensitive data will be subject to strict regulations and data residency requirements.\\
To accommodate these regulations, these organizations currently maintain entirely separate data operations in each region. This has many effects on the organization's operations and expenses. In particular, this often leads to useful solutions being deployed in single regions, limiting their benefit to the organization. In cases where a solution is deemed useful enough to be deployed in multiple regions, entire environments and pipelines for training and inference are duplicated across multiple regions in what are commonly called regionally-pinned deployments, as illustrated in Figure \ref{fig:regionally-pinned_deployments}.
\begin{center}
\includegraphics[scale=0.05]{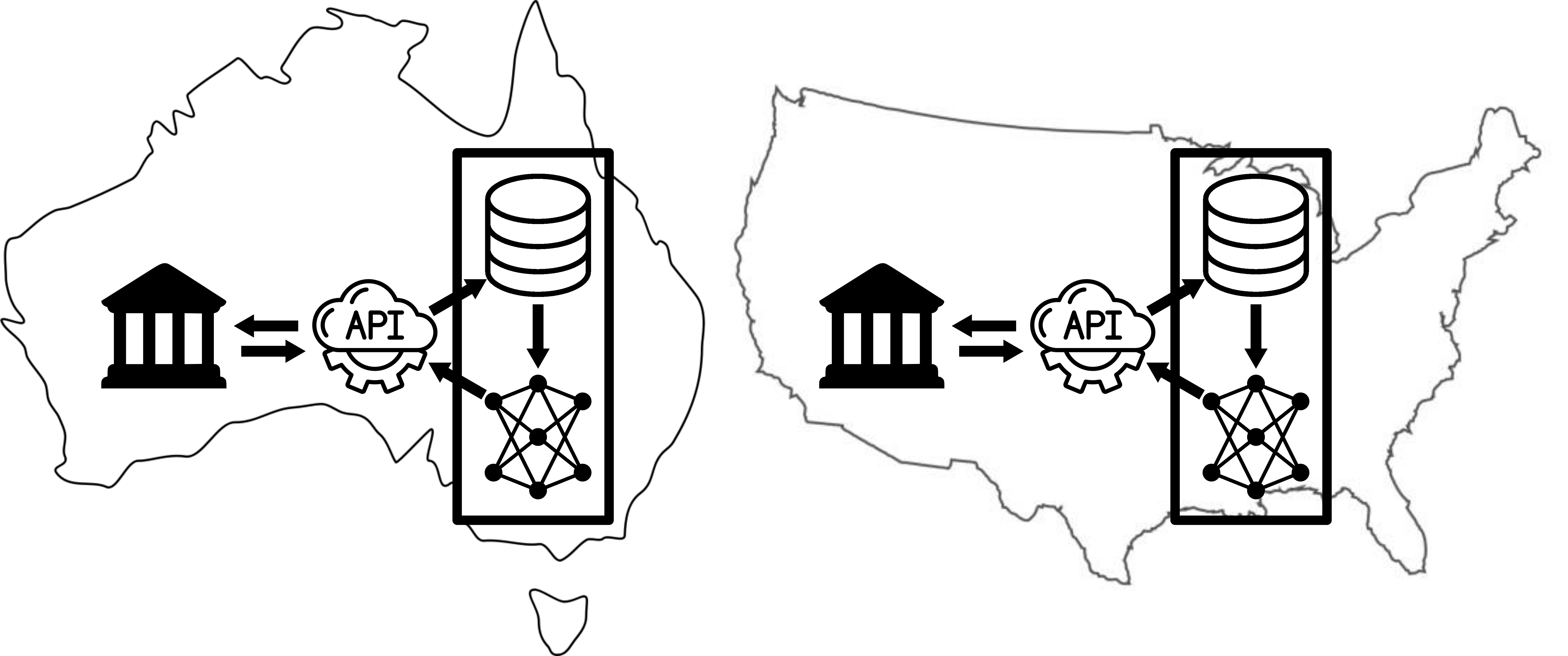}
\captionof{figure}{Illustration of duplicated ML pipelines in regionally-pinned deployments (map is not to scale)}
\label{fig:regionally-pinned_deployments}
\end{center}
It must be noted that in this regionally-pinned deployment pattern, organizations rarely, if ever, apply end-to-end protection to data traveling within these environments, leaving sensitive inputs exposed at key points in the pipeline. The data are not leaving the organization's environment, nor are they leaving any regional jurisdiction, so technically, this does not violate any regulations. However, if these environments are compromised, the input vectors passed from the database to the ML model's prediction API can be easily exfiltrated and later used to commit identity theft and other financial crimes \cite{tramèr2016stealingmachinelearningmodels}.\\
With the VEIL architecture and multi-region deployment pattern, multiple regional Source Environments can be defined to satisfy data residency regulations. Then, Training and Inference Environments can be established in a single region and sent latent vectors from the various Source Environments. Sensitive information never leaves its respective regional Source Environment, let alone its regional jurisdiction. Therefore, compliance with all regional regulations happens automatically and by design. This enables such an organization to move away from regionally siloed MLOps and regionally pinned ML deployments, thereby expanding each ML solution across the entire organization and/or all of its customers without regionally duplicated Training and Inference Environments and pipelines. Figure \ref{fig:global_VEIL_deployments} illustrates the globally transformed inference framework under the VEIL architecture.\\
The fourth deployment pattern is well-suited to any ML/AI scenario in which several organizations would benefit from the cooperative, reciprocal sharing of information, but regulations ban or greatly restrict the sharing of sensitive data. Examples of this include ML/AI modeling within hospitals or medical research, and fraud detection involving large-scale scams across multiple regions. Modeling in the context of medical research benefits greatly from training data that encompass a large and diverse patient population treated under a variety of conditions \cite{https://doi.org/10.1002/sim.8981, schinkel2023embracing}. Some large hospitals in urban centers have access to a large and diverse patient population and, as such, have become centers of excellence in AI for healthcare use cases. Simultaneously, most rural clinics have struggled to adopt ML/AI systems, as they lack access to large, diverse patient populations for model training \cite{10.1093/jamia/ocaf206}. This proposed deployment pattern would alleviate that imbalance.
\begin{center}
\includegraphics[scale=0.05]{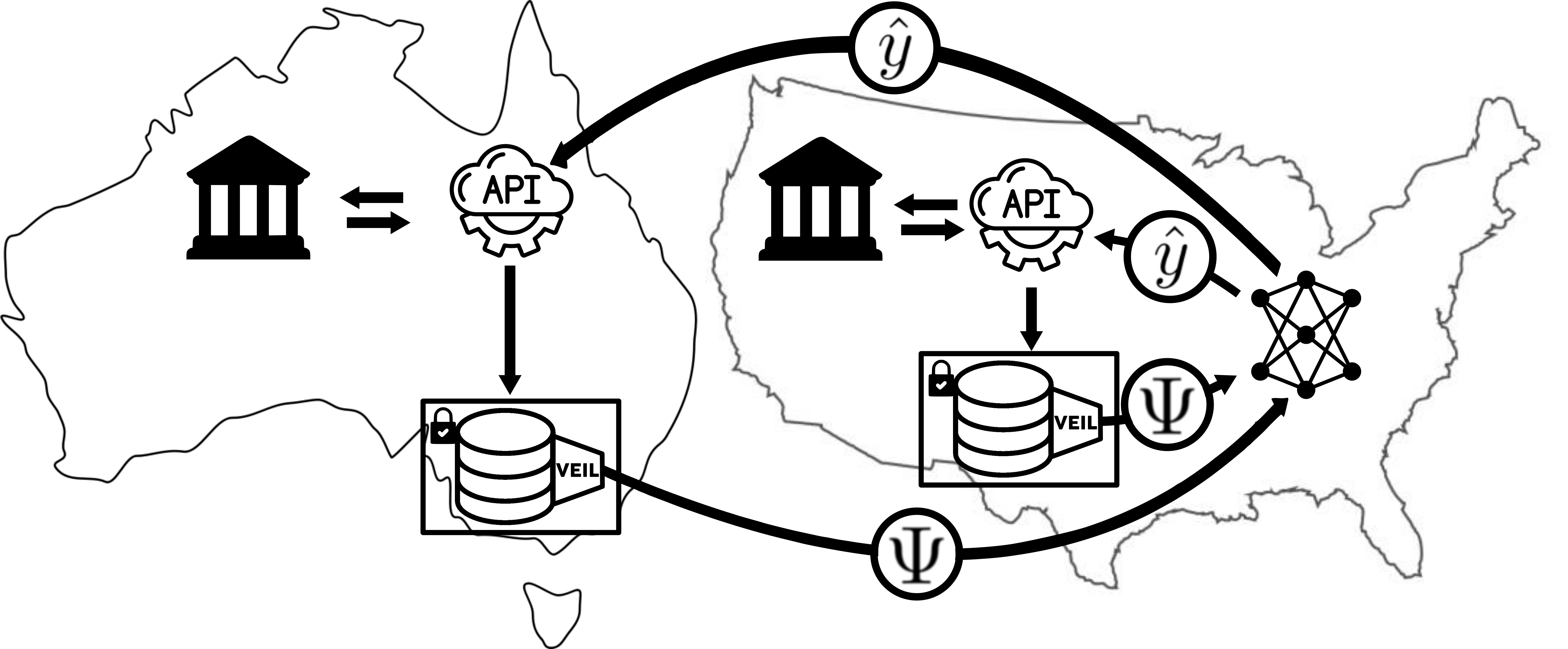}
\captionof{figure}{Illustration of a simplified, multi-regional ML deployment with the VEIL architecture (map is not to scale)}
\label{fig:global_VEIL_deployments}
\end{center}
The latent encodings produced by ICA and deployed under this variant of the VEIL architecture would enable multiple facilities or sites to access a shared ML/AI infrastructure in which models benefit from patterns observed across facilities, sites, or organizations, without any participant directly sharing sensitive data. Under this framework, each participant would maintain and control its own Source Environment, which houses and secures that participant's sensitive data. An external, unified ML/AI infrastructure then hosts models useful to all participants. Some examples include models for fraud detection, patient triage recommendations, readmission risk detection, payor rejection probability estimates, and more.\\
To prevent any information leakage, individual model calls can be mapped to specific consumers or patients within the Source Environment. This ensures that the correct model response can be mapped back to the original requester and the consumer or patient in question without using any identifying information to index the latent vector in the model call, thus preventing any latent vector from being linked to information found in external systems. This would also require standardization of the model call across all participants.\\
Under this pattern, the models can be trained on latents from all participants, greatly enhancing their performance by providing large, diverse training datasets that have been out of reach for many researchers and organizations \cite{chen2025hospital}. Simultaneously, all sensitive data remain in each participant's respective Source Environment, and no sensitive information is ever exposed to the shared ML/AI infrastructure, allowing each participant to remain compliant with their regulatory obligations for privacy and security. For a complete diagram of this shared ML/AI infrastructure in the context of a healthcare use case, see Figure \ref{fig:shared_hc_infrastructure} in Appendix A.

\subsection{Summary and Key Architectural Guarantees}

The VEIL architecture is built around a clear and durable trust boundary: all sensitive operations occur within the trusted Source Environment, and only compressed, non-invertible latent representations leave that environment. This architecture and ICA provide the following.
\begin{itemize}
    \item \textbf{Data minimization by design:} Raw data never cross the trust boundary.
    \item \textbf{Structural non-invertibility:} Latent vectors contain insufficient information to reconstruct original records.
    \item \textbf{Attack-surface elimination:} No gradients, activations, or raw features are exposed to Training or Inference Environments.
    \item \textbf{Compliance alignment:} VEIL fits naturally into frameworks requiring data minimization, privacy-by-design, and controlled data movement (GDPR, CPRA, HIPAA).
    \item \textbf{High utility:} ML performance is maintained because gradients are not clipped, and no artificial noise is injected into the training process.
    \item \textbf{Low latency:} Input vectors are significantly compressed, reducing transfer latency, complexity, ML model size, inference times, and power consumption.
    \item \textbf{Operational simplicity:} Unlike DP and HE, VEIL requires no privacy budgets, noise parameters, or specialized cryptography stacks to implement.
\end{itemize}
This structural design and its privacy guarantees enable organizations to move forward in their adoption of ML while meeting stringent regulatory requirements, keeping customer and consumer data safe, and avoiding the performance penalties and computational overhead of DP and HE\@. Organizations no longer need to avoid the convenience of modern cloud platforms to preserve the security and privacy of their data, since they can now train and deploy ML models without exposing sensitive data or gradients to these environments, and maintain a consistent privacy posture across the entire ML lifecycle.\\
Finally, large, multi-region enterprises now have a globally scalable, secure, low-latency, highly reliable, and cost-efficient ML architecture that eliminates the need for siloed operations and regionally-pinned ML deployments. In effect, ICA and the VEIL architecture provide a next-generation privacy foundation for enterprise ML that is secure, simple, scalable, and operationally aligned with modern MLOps patterns.

\section{Proof of Non-Invertibility}

Previous sections have asserted that the latent encoding produced by the technique described in this paper is non-invertible. In this section, this claim will be thoroughly proven.

\subsection{Notation and Assumptions}

Let $\mathcal{X} \subseteq \mathbb{R}^D$ be the input space of raw data (tabular records, images, text embeddings, etc.), with $D \in \mathbb{N}$ and $D \gg 1$. Let $\mathcal{Z}\subseteq \mathbb{R}^E$, be the latent encoding space with $E \ll D$. Let $\mathcal{Y}$ be the target space. Let $f_{\theta}:\mathcal{X} \to \mathcal{Z}$ be the above-described encoder, which is a continuous function with parameters $\theta$. Finally, let $g_{\phi}: \mathcal{Z} \to \mathcal{Y}$ be a downstream model trained in a cloud environment on latent encodings.\\
Assume that the encoder is trained to optimize the predictive performance on the supervised task, $g_{\phi} \circ f_{\theta}$, \textit{not} to reconstruct $\mathbf{x}$. Assume that any auxiliary losses are chosen to increase task utility and reduce identity leakage, not to preserve input information. Assume that there is no decoder network trained to approximate an inverse of $f_{\theta}$; thus, ICA is implemented via a one-way encoder. Assume that in deployment, $\theta$ and access to $f_{\theta}$ applied to arbitrary inputs remain within the customer's trusted Source Environment, while the untrusted Training and Inference Environments observe only $\mathcal{Z}$ and $g_{\phi}$. Finally, assume that $\mathcal{X}$ contains a nonempty open subset of $\mathbb{R}^D$. In other words, assume the input data live in a region with genuine $D$-dimensional variability, not on a low-dimensional manifold known to the attacker.\\
With these assumptions and definitions, multiple types of non-invertibility are formalized.

\subsection{Topological Non-Invertibility via Dimensionality Reduction}

In this subsection, the ICA encoder, or VEIL, is regarded as a function, and its non-invertibility is formalized and proven from an analytical, algebraic, and topological perspective. Before this is formally stated as a theorem and a proof is constructed, the following standard definitions must be recalled.
\begin{definition}[Open set]
\label{def:open_set}
A set $U \subseteq \mathbb{R}^D$ is said to be \textbf{open} if $\forall\mathbf{x} \in U$, $\exists\epsilon > 0$ such that
$$B_{\epsilon}\left(\mathbf{x}\right) \equiv \left\{\mathbf{y}\in\mathbb{R}^D: \norm{\mathbf{y}-\mathbf{x}} < \epsilon\right\}$$
is contained in $U$.
\end{definition}
\begin{definition}[Injective function]
\label{def:injective_func}
A function $f:X \to Y$ is said to be \textbf{injective} (one-to-one) if $\forall \mathbf{x}_1, \mathbf{x}_2 \in X$, the following is true.
$$f\left(\mathbf{x}_1\right) = f\left(\mathbf{x}_2\right) \implies \mathbf{x}_1 = \mathbf{x}_2$$
\end{definition}
\begin{definition}[Inverse function]
\label{def:inv_func}
Let $f: X \to Y$ be a function. An \textbf{inverse} of $f$ is a function $f^{-1}: Y \to X$ such that $\forall x \in X$, $f^{-1}\left(f\left(x\right)\right) = x$, and $\forall y \in Y$, $f\left(f^{-1}\left(y\right)\right) = y$. A necessary condition for $f$ to admit an inverse is that $f$ must be \textbf{bijective} (both injective and surjective).
\end{definition}
\begin{definition}[Continuous function]
\label{def:cont_func}
A function $f: U \to \mathbb{R}^m$, where $U \subseteq \mathbb{R}^n$ is open, is \textbf{continuous} if $\forall \mathbf{x} \in U$ and $\forall \epsilon > 0$, $\exists \delta > 0$ such that the following is true.
$$\norm{\mathbf{x} - \mathbf{\mathbf{x}'}} < \delta \implies \norm{f\left(\mathbf{x}\right) - f\left(\mathbf{x}'\right)} < \epsilon$$
\end{definition}
\noindent It should be noted that neural networks are compositions of continuous functions (affine mappings, convolutions, continuous activations, etc.), and are, therefore, continuous functions.\\
The proof of the non-injectivity of the encoder also relies upon a well-known theorem \cite{BredonTopologyGeometry1993}, which will not be re-proven here, as it is standard to take it as a known result in rigorous analysis and topology.
\begin{theorem}[Invariance of Domain]
\label{thm:invariance_of_domain}
Let $D \geq 1$. If $U \subseteq \mathbb{R}^D$ is open and the function $f:U \to \mathbb{R}^D$ is continuous and injective, then:
\begin{itemize}
    \item $f\left(U\right)$ is open in $\mathbb{R}^D$; and
    \item $f$ is a homeomorphism between $U$ and $f\left(U\right)$ ($f$ is bijective and invertible, and $f^{-1}$ is continuous). 
\end{itemize}
\end{theorem}
\noindent Theorem \ref{thm:invariance_of_domain} can now be used to prove the following lemma.
\begin{lemma}
\label{lem:empty_interior}
Let $D > E \geq 1$, and consider the following subset of $\mathbb{R}^D$.
$$S \equiv \mathbb{R}^E \times \left\{0\right\}^{D-E} = \left\{\left(x_1, \ldots, x_E, 0, \ldots, 0\right) \in \mathbb{R}^D: x_i \in \mathbb{R}\right\}$$
Then, $S$ has an empty interior in $\mathbb{R}^D$. That is, there is no nonempty subset $U \subseteq \mathbb{R}^D$ such that $U \subseteq S$.
\end{lemma}
\begin{proof}
Assume that $\exists s \in S$ and a radius $r > 0$ such that the open ball $B_r\left(s\right)$ is contained in $S$. By the definition of $S$, $s = \left(s_1, \ldots, s_D\right)$ with $s_{E+1}=\cdots=s_D=0$. Now consider $v \in \mathbb{R}^D$ such that $v=\left(0, \ldots, 0, \delta\right)$, where the first $D-1$ coordinates are 0 and the last coordinate is some $\delta \in \mathbb{R}$ such that $0 < \lvert\delta\rvert < r$.\\
Then, $s + v = \left(s_1, \ldots, s_E, 0, \ldots, 0, \delta\right)$ has a nonzero entry in the last coordinate. By the definition of $S$, $s+v \notin S$.\\
However, $\norm{v} = \lvert\delta\rvert < r$. Therefore, $s+v \in B_r\left(s\right)$, which contradicts the assumption that $B_r\left(s\right) \subseteq S$.\\
Thus, by contradiction, no such open ball exists, and the interior of $S$ is empty.
\end{proof}
\noindent These definitions and results can now be combined to give the following result.
\begin{theorem}[Encoder Non-Injectivity]
\label{thm:encoder_non-injectivity}
Let $D > E \geq 1$. Let $U \subseteq\mathbb{R}^D$ be a nonempty open set, and let the function $f: U \to \mathbb{R}^E$ be continuous. Then, $f$ cannot be injective.
\end{theorem}
\begin{proof}
Assume that $\exists f: U \to \mathbb{R}^E$ such that $f$ is a continuous injective function.\\
Let $i: \mathbb{R}^E \to \mathbb{R}^D$ be the following canonical linear embedding.
$$i\left(\psi_1, \ldots, \psi_E\right)=\left(\psi_1, \ldots, \psi_E, 0, \ldots, 0\right)$$
Clearly, $i$ is continuous and injective. Now, consider the following composition.
$$F \equiv i \circ f: U \to \mathbb{R}^D$$
Since $F$ is a composition of continuous injective functions, it is continuous and injective. The image of $F$ lies in the following subspace.
$$S \equiv \mathbb{R}^E \times \left\{0\right\}^{D-E} \subseteq \mathbb{R}^D$$
By assumption, $U$ is a nonempty open subset of $\mathbb{R}^D$, and, by construction, $F: U \to \mathbb{R}^D$ is continuous and injective. By Theorem \ref{thm:invariance_of_domain}, $F\left(U\right)$ is open in $\mathbb{R}^D$, and $F: U \to F\left(U\right)$ is a homeomorphism. In particular, $F\left(U\right)$ must be a nonempty open subset of $\mathbb{R}^D$.\\
However, by construction, $F\left(U\right) \subseteq S$. From Lemma \ref{lem:empty_interior}, the interior of $S$ is empty in $\mathbb{R}^D$, and therefore contains no nonempty open subsets of $\mathbb{R}^D$.\\
Thus, $F\left(U\right)$ cannot be open in $\mathbb{R}^D$ unless $F\left(U\right)$ is empty, which contradicts the fact that $U$ is nonempty and $F$ is injective.\\
By this contradiction, $f:U \to \mathbb{R}^D$ with $D>E$ cannot be injective.
\end{proof}
\begin{corollary}[Encoder Non-Invertibility]
\label{cor:encoder_non-invertibility}
Let $f:\mathbb{R}^D \to \mathbb{R}^E$ be a function with $E<D$. Then, the inverse function $f^{-1}:\mathbb{R}^E \to \mathbb{R}^D$ \textbf{does not exist} as a function defined on any open region of its domain.
\end{corollary}
\begin{proof}
Let $D>E$ and $f:\mathbb{R}^D \to \mathbb{R}^E$ be a continuous function. Suppose $\exists U \subseteq \mathbb{R}^D$ such that $U$ is open and nonempty and $f|_U$ is bijective onto its image $f\left(U\right)$. This yields a continuous injective function from $U$ to $\mathbb{R}^E$, which contradicts Theorem \ref{thm:encoder_non-injectivity}. Thus, no such open set $U$ exists.
\end{proof}
\noindent It has now been shown that an ICA encoder $f: \mathbb{R}^D \to \mathbb{R}^E$ with $E < D$ cannot be invertible on any open region of its domain. That is, the function $f^{-1}$ cannot exist, which is the sense in which the encoding is not simply computationally difficult to invert, but is structurally non-invertible.\\
To make this result more concrete, the various components of this result are described in terms of an ML modeling scenario and pipeline environment. The input space $\mathcal{X}=\mathbb{R}^D$ consists of input vectors $\mathbf{x}$ with $D$ raw features. The encoder is a neural network, which is equivalent to some continuous function $f_{\theta}:\mathbb{R}^D \to \mathbb{R}^E$. The latent encoded space has dimensionality $E$, which is chosen to be strictly smaller than $D$ (often much smaller) to enforce an informational bottleneck.\\
Assume the distribution of the data varies over more than $E$ degrees of freedom. That is, the set of possible inputs occupies, or approximates, some nonempty open region of the input space $\mathcal{X}\subseteq\mathbb{R}^D$. Then, by Theorem \ref{thm:encoder_non-injectivity}, the encoder, $f_{\theta}$ cannot be injective on that region. By the definition of an injective function, this means multiple distinct inputs, $\mathbf{x}_1 \neq \mathbf{x}_2$, must satisfy $f_{\theta}\left(\mathbf{x}_1\right)=f_{\theta}\left(\mathbf{x}_2\right)$. Thus, the latent encoding, $\Psi = f_{\theta}\left(\mathbf{x}\right)$ cannot uniquely identify the original data, $\mathbf{x}$; the preimage set $\mathcal{P}\left(\Psi\right) \equiv \left\{\mathbf{x}:f_{\theta}\left(\mathbf{x}\right)=\Psi\right\}$ contains multiple elements, typically infinitely many.\\
Therefore, no function $f^{-1}: \mathbb{R}^E \to \mathbb{R}^D$ can be a true inverse to the encoder $f_{\theta}$ on any open region in the input space; there is simply not enough information in $\Psi$ to recover $\mathbf{x}$ uniquely.\\
Even a computationally unbounded adversary with perfect knowledge of the encoder $f_{\theta}$ cannot define a unique inverse on any nontrivial region. However, this is not even a reasonable assumption of adversarial capabilities, as the VEIL architecture places the encoder inside the source database. This means the only way an adversary could gain access to the encoder is to compromise the source database environment itself, thereby negating the need to access the encoder.\\
Under a realistic threat model, the adversary would have no knowledge of the encoder $f_{\theta}$. In this situation, the adversary cannot even uniquely identify the original input space, which is the domain of the encoder. Under this realistic assumption, from the view of any adversary, the preimage set $\mathcal{P}\left(\Psi\right)$ now consists of countably infinite potential input spaces, which are each uncountably infinite.\\
Once again, this leads to an encoding that is not simply difficult to invert, or for which an inversion algorithm is currently unknown. It is logically impossible for an inverse to exist.

\subsection{Information-Theoretic Loss and Under-Determination (Source Environment Attacker)}

While the topological non-invertibility of the encoder was proven in the previous subsection, one might ask whether an adversary could still reconstruct the original input $\mathbf{x}$ ``with high probability''. This is reasonable and necessary to consider, as some researchers have investigated the probabilistic invertibility of various latent representations in ML and DL \cite{park2025investigatinginvertibilitymultimodallatent}. This question is addressed now in this subsection.\\
For completeness, the scenario of an ``optimal attacker'' will be considered. An optimal attacker is one who has somehow gained access to the encoder function $f_{\theta}$ and its parameters $\theta$. Under the VEIL architecture, for an attacker to gain access to or knowledge of the encoder, the attacker would be required to breach the trusted Source Environment, discover the encoder, and extract its structure and parameters. This is unlikely, but not impossible, and would represent the catastrophic failure of the encryption, governance, and/or other security systems protecting the Source Environment.\\
To be clear, it is assumed that although the attacker has gained access to many aspects of the Source Environment, they still do not have access to the raw data, as that would require a direct breach of the database itself. Obviously, an attacker with direct access to the database would be unlikely to need to invert the encoder. This would be the equivalent of a crew of robbers breaking into a bank vault, then, instead of taking the cash, deciding to spend their time and effort bypassing an additional aspect of the vault's security. Therefore, this section will consider the scenario in which an attacker has somehow gained knowledge of the encoder without gaining direct access to the source data to show that any compromise of the original data cannot be achieved through the encoder or the latent vectors it produces.\\
To begin, a clear definition of what is known and unknown to the attacker must be given. An optimal attacker would, by assumption, have knowledge of the encoder $f_{\theta}$ and its parameters $\theta$, the latent representation space $\mathcal{Z}$, and the downstream ML model $g_{\phi}$ and its output space $Y$. Knowledge of the encoder would also reveal that the input space is $\mathbb{R}^D$, and so the attacker would also know the dimensionality of the input space $D$, but would not have access to any specific inputs $\mathbf{x}$. It can also be assumed that the attacker would have full knowledge of the Training and Inference Environments. This would include any cloud storage solutions and logs, which would contain only latent vectors.\\
Finally, the attacker would not have knowledge of any scaling, normalization, distributions, or structure of the original data, including the data type (tabular, images, text, etc.) or any categorical encodings applied in the engineering of categorical features.\\
As in the previous subsection, key standard definitions \cite{applebaum2008, reza2016, khinchin1957} must be recalled before constructing a formal proof.
\begin{definition}[Entropy]
\label{def:entropy}
Let $X$ be a discrete random variable with image $\mathcal{X}$ and with probability mass function (PMF) $p_X\left(\mathbf{x}\right) = \mathbb{P}\left(X=\mathbf{x}\right)$. The \textbf{(Shannon) entropy} of $X$ is defined as follows.
$$H\left(X\right) \equiv -\sum_{\mathbf{x} \in \mathcal{X}} p_X\left(\mathbf{x}\right)\text{log}_b p_X\left(\mathbf{x}\right)$$
Conventionally, $0\text{ log}_b0 \equiv 0$. Entropy is measured in bits when $b=2$ and nats when $b=e$.
\end{definition}
\begin{definition}[Joint and Conditional Entropy]
\label{def:joint-cond_entropy}
Let $X$ and $Z$ be discrete random variables with images $\mathcal{X}$ and $\mathcal{Z}$, respectively, and joint PMF $p_{X,Z}\left(\mathbf{x},\Psi\right)$, where $\Psi$ is a particular latent vector. The \textbf{joint entropy} is given by the following function.
$$H\left(X,Z\right) \equiv -\sum_{\mathbf{x}\in\mathcal{X}}\sum_{\Psi\in\mathcal{Z}} p_{X,Z}\left(\mathbf{x},\Psi\right)\text{log}_b p_{X,Z}\left(\mathbf{x},\Psi\right)$$
The \textbf{conditional entropy} of $X$ given $Z$ is given by
$$H\left(X|Z\right) \equiv \sum_{\Psi\in\mathcal{Z}} p_Z\left(\Psi\right)H\left(X|Z=\Psi\right)$$
where
$$H\left(X|Z=\Psi\right) \equiv -\sum_{\mathbf{x}\in\mathcal{X}}p_{X|Z}\left(\mathbf{x}|\Psi\right) \text{log}_b p_{X|Z}\left(\mathbf{x}|\Psi\right)$$
and $p_{X|Z}\left(\mathbf{x}|\Psi\right)$ is the conditional PMF\@. By definition, $H\left(X|Z\right) \geq 0$.
\end{definition}
\begin{definition}[Mutual Information]
\label{def:mutual_info}
Let $X$ and $Z$ be random variables. The \textbf{mutual information} between $X$ and $Z$ is given by the following.
$$I\left(X;Z\right) \equiv H\left(X\right)-H\left(X|Z\right) = H\left(Z\right)-H\left(Z|X\right)$$
Mutual information measures how much observing $Z$ reduces uncertainty about $X$.
\end{definition}
\begin{definition}[Functional Dependency]
\label{def:functional_dependency}
Let $\mathcal{X}$ and $\mathcal{Z}$ be sets and suppose there is some relation that maps values in $\mathcal{X}$ to values in $\mathcal{Z}$. The set $\mathcal{X}$ is said to \textbf{functionally determine} $\mathcal{Z}$ if, and only if, each value in $\mathcal{X}$ is associated with precisely one value in $\mathcal{Z}$. The relation is then said to be a \textbf{functional dependency}, written $\mathcal{X} \to \mathcal{Z}$.
\end{definition}
\begin{definition}[Support]
\label{def:support}
Let $X$ be a discrete random variable with image $\mathcal{X}$ and PMF $p_X\left(\mathbf{x}\right)$. The \textbf{support} of the distribution of $X$ is the following set.
$$\text{supp}\left(X\right) \equiv \left\{\mathbf{x}\in\mathcal{X}: p_X\left(\mathbf{x}\right)>0\right\}$$
\end{definition}
\noindent The proof is now presented in the discrete setting, as this aligns with how finite-precision encoders and data are represented in practice. In a theoretical sense, the raw data live mathematically in $\mathbb{R}^D$. However, no real ML system ever sees a truly continuous input, as every stage of data representation and processing uses finite-precision floating-point arithmetic. This induces a finite, discrete representational universe.\\
The results derived in this discrete setting will also hold for the continuous setting. Suppose the mathematical input $X^*$ is continuous in $\mathbb{R}^D$. The actual ML input is $X=Q\left(X^*\right)$, where $Q$ is the floating-point quantization operator. This has two important implications.
\begin{enumerate}
\item $X^* \mapsto X$ is a many-to-one mapping. Quantization irretrievably destroys information, even for an omniscient attacker.
\item The attacker never sees $X^*$. Only $X$ exists in any computational pipeline that reaches the encoder.
\end{enumerate}
Therefore, the following holds true.
$$H\left(X^*|Z\right) \geq H\left(X|Z\right)$$
Thus, any lower bound on the discrete conditional entropy holds \textbf{a fortiori} for the truly continuous source. The discrete proofs actually provide a lower bound on the continuous uncertainties, meaning the real uncertainty is at least as great as the derived uncertainty.\\
In the VEIL architecture, the encoder is a deterministic function on the discrete space of representable inputs. Let $\mathcal{X}$ be the set of all possible inputs. Let $\mathcal{Z}$ be the set of all possible latent encodings. Let $f_{\theta}:\mathcal{X} \to \mathcal{Z}$ be the encoder, and let $X$ be a random variable on $\mathcal{X}$. Finally, define $Z=f_{\theta}\left(X\right)$.\\
Thus, conditionally on $X$, the latent representation $Z$ is deterministic, which yields the following.
$$\mathbb{P}\left(Z=\Psi|X=\mathbf{x}\right) = 
    \begin{cases}
        1 & \Psi =f_{\theta}\left(\mathbf{x}\right)\\
        0 & \text{otherwise}
    \end{cases}$$
Therefore, $H\left(Z|X\right)=0$, and mutual information simplifies as follows.
$$I\left(X;Z\right) = H\left(Z\right) - H\left(Z|X\right) = H\left(Z\right)$$
Intuitively, the encoder is a noiseless channel from $\mathcal{X}$ to $\mathcal{Z}$, but it may compress, or lose, information by mapping multiple inputs $\mathbf{x}$ to the same latent vector encoding $\Psi$. The following is a fundamental fact that relates conditional entropy to functional dependence between $\mathcal{X}$ and $\mathcal{Z}$.
\begin{lemma}[Zero Conditional Entropy if and only if Functional Dependence]
\label{lem:entropy-dependence}
Let $X$ and $Z$ be discrete random variables with images $\mathcal{X}$ and $\mathcal{Z}$, respectively. Then, $H\left(X|Z\right)=0$ if, and only if, there exists a (possibly deterministic) function $f:\mathcal{Z} \to \mathcal{X}$ such that $\mathbb{P}\left(X=f\left(Z\right)\right) = 1$.
\end{lemma}
\begin{proof}
Suppose $H\left(X|Z\right)=0$. Then, $\forall\Psi\in\text{supp}\left(Z\right)$, $H\left(X|Z=\Psi\right)=0$. However, the entropy of a discrete distribution is 0 if, and only if, that distribution is a point mass. That is, $\exists\mathbf{x}_{\Psi}$ such that $\mathbb{P}\left(X=\mathbf{x}_{\Psi}|Z=\Psi\right)=1$.\\
Define $f\left(\Psi\right)=\mathbf{x}_{\Psi}$ $\forall\Psi \in \text{supp}\left(Z\right)$. Then, the following result is given.
$$\mathbb{P}\left(X=f\left(Z\right)\right)=\sum_{\Psi\in\text{supp}\left(Z\right)} p_Z\left(\Psi\right) \mathbb{P}\left(X=f\left(\Psi\right)|Z=\Psi\right)$$
$$\mathbb{P}\left(X=f\left(Z\right)\right)=\sum_{\Psi\in\text{supp}\left(Z\right)} p_Z\left(\Psi\right)\cdot1$$
$$\mathbb{P}\left(X=f\left(Z\right)\right)=1$$
Conversely, if $\exists f:\mathcal{Z} \to \mathcal{X}$ such that $\mathbb{P}\left(X=f\left(Z\right)\right)=1$, then $\forall\Psi\in\text{supp}\left(Z\right)$, $\mathbb{P}\left(X=f\left(\Psi\right)|Z=\Psi\right)=1$.\\
Thus, the conditional distribution of $X$ given $Z=\Psi$ is a point mass, and $H\left(X|Z=\Psi\right)=0$. Applying the definition of conditional entropy from Definition \ref{def:joint-cond_entropy} yields the following.
$$H\left(X|Z\right)=\sum_{\Psi\in\text{supp}\left(Z\right)}\mathbb{P}\left(Z=\Psi\right) H\left(X|Z=\Psi\right)$$
$$H\left(X|Z\right)=\sum_{\Psi\in\text{supp}\left(Z\right)}\mathbb{P}\left(Z=\Psi\right)\cdot0$$
$$H\left(X|Z\right)=0$$
\end{proof}
\noindent Lemma \ref{lem:entropy-dependence} establishes that if an adversary can reconstruct $X$ from $Z$ with certainty, that is, if there is a true inverse $f^{-1}$ such that $X=f^{-1}\left(Z\right)$, then $H\left(X|Z\right)=0$. Conversely, if $H\left(X|Z\right)>0$, then no such deterministic inverse exists.\\
Now consider the following definition and proposition.
\begin{definition}[Injectivity on Support]
\label{def:injectivity_on_support}
Let $X$ and $Z$ be random variables with images $\mathcal{X}$ and $\mathcal{Z}$, respectively. Let $f:\mathcal{X} \to \mathcal{Z}$ be a function. The function $f$ is said to be \textbf{injective on the support of} $X$ if, and only if, $\forall\mathbf{x}_1, \mathbf{x}_2 \in \text{supp}\left(X\right)$, the following is true.
$$\mathbf{x}_1 \neq \mathbf{x}_2 \implies f\left(\mathbf{x}_1\right) \neq f\left(\mathbf{x}_2\right)$$
\end{definition}
\begin{proposition}[Non-Injectivity Implies Information Loss]
\label{prop:non-injectivity_info_loss}
Let $\mathcal{X}$ and $\mathcal{Z}$ be sets, $X$ a discrete random variable on $\mathcal{X}$, and $f:\mathcal{X} \to \mathcal{Z}$ a deterministic encoder. Let $Z \equiv f\left(X\right)$. Suppose $\exists \mathbf{x}_1, \mathbf{x}_2 \in \text{supp}\left(X\right)$ such that $\mathbf{x}_1 \neq \mathbf{x}_2$ and $f\left(\mathbf{x}_1\right)=f\left(\mathbf{x}_2\right)=\Psi^{\star}$. Then, $H\left(X|Z\right)>0$.
\end{proposition}
\begin{proof}
Let $\mathcal{X}$ and $\mathcal{Z}$ be sets, $X$ a discrete random variable on $\mathcal{X}$, and $f:\mathcal{X} \to \mathcal{Z}$ a deterministic encoder. Let $Z \equiv f\left(X\right)$. Suppose $\exists \mathbf{x}_1, \mathbf{x}_2 \in \text{supp}\left(X\right)$ such that $\mathbf{x}_1 \neq \mathbf{x}_2$ and $f\left(\mathbf{x}_1\right)=f\left(\mathbf{x}_2\right)=\Psi^{\star}$.\\
Since $\mathbf{x}_1, \mathbf{x}_2 \in \text{supp}\left(X\right)$, then by Definition \ref{def:support}, $p_X\left(\mathbf{x}_1\right), p_X\left(\mathbf{x}_2\right)>0$.\\
Define the event $V \equiv \left\{Z=\Psi^{\star}\right\}$. Since $f\left(\mathbf{x}_1\right)=f\left(\mathbf{x}_2\right)=\Psi^{\star}$, the following must be true.
$$p_Z\left(\Psi^{\star}\right) = \mathbb{P}\left(f\left(X\right)=\Psi^{\star}\right) \geq p_X\left(\mathbf{x}_1\right) + p_X\left(\mathbf{x}_2\right) > 0$$
Now, consider the conditional distribution $p_{X|Z}\left(\mathbf{x}|\Psi^{\star}\right)$. For $\mathbf{x} \in \left\{\mathbf{x}_1, \mathbf{x}_2\right\}$, the definition of conditional probability yields the following.
$$p_{X|Z}\left(\mathbf{x}|\Psi^{\star}\right) \equiv \frac{\mathbb{P}\left(X=\mathbf{x}, Z=\Psi^{\star}\right)}{\mathbb{P}\left(Z=\Psi^{\star}\right)} = \frac{\mathbb{P}\left(X=\mathbf{x}\right)} {\mathbb{P}\left(Z=\Psi^{\star}\right)} > 0$$
Thus, the conditional PMF $p_{X|Z}\left(X=\cdot\text{ }|Z=\Psi^{\star}\right)$ assigns positive probability to at least two distinct values in $\text{supp}\left(X\right)$. Therefore, it is not a point mass, and its entropy is strictly positive $H\left(X|Z=\Psi^{\star}\right)>0$.\\
Definition \ref{def:joint-cond_entropy} then yields the following.
$$H\left(X|Z\right)=\sum_{\Psi\in\text{supp}\left(Z\right)} p_Z\left(\Psi\right)H\left(X|Z=\Psi\right)$$
$$H\left(X|Z\right) \geq p_Z\left(\Psi^{\star}\right)H\left(X|Z=\Psi^{\star}\right)$$
It has already been shown that $p_Z\left(\Psi^{\star}\right), H\left(X|Z=\Psi^{\star}\right)>0$. Thus, their product is strictly positive. Therefore, $H\left(X|Z\right)>0$.
\end{proof}
\noindent Combining Lemma \ref{lem:entropy-dependence} and Proposition \ref{prop:non-injectivity_info_loss}, it must be concluded that if the function $f_{\theta}:\mathcal{X}\to\mathcal{Z}$ is non-injective on $\text{supp}\left(X\right)$, then $H\left(X|Z\right)>0$. Therefore, there does \textit{not} exist a function $f^{-1}:\mathcal{Z}\to\mathcal{X}$ with $\mathbb{P}\left(X=f^{-1}\left(Z\right)\right)=1$. In other words, even for an attacker with full knowledge of the joint distribution, no perfect inverse exists. This is what is meant by information-theoretic under-determination; the data processing step $X \mapsto Z$ has irreversibly lost information about $X$.\\
This information loss can also be characterized through mutual information. Proposition \ref{prop:non-injectivity_info_loss} establishes that if a function $f:\mathcal{X}\to\mathcal{Z}$ is non-injective on $\text{supp}\left(X\right)$, then $H\left(X|Z\right)>0$. Applying Definition \ref{def:mutual_info} yields the following.
$$I\left(X;Z\right) \equiv H\left(X\right) - H\left(X|Z\right) < H\left(X\right)$$
This means that the encoder \textit{strictly reduces} the information that $Z$ carries about $X$. In an information-theoretic sense, some aspect of $X$ has been irreversibly removed.\\
Consider the consequences of this for an optimal attacker. An attacker with access to the latent representation $Z$ and complete knowledge of the joint distribution $p_{X,Z}$ wishes to construct an estimator $\hat{X}\left(Z\right)$ to guess or approximate the true $X$. An optimal Bayes estimator, in terms of minimizing the probability of error, is the following maximum a posteriori (MAP) rule.
$$\hat{\mathbf{x}}\left(\Psi\right) = \argmax_{\mathbf{x}\in\mathcal{X}} p_{X|Z}\left(\mathbf{x}|Z=\Psi\right)$$
Define the probability of correct reconstruction as follows.
$$P_{\text{recon}} \equiv \mathbb{P}\left(\hat{X}\left(Z\right)=X\right)$$
The probability $P_{\text{recon}}$ is related to $H\left(X|Z\right)$ in the following well-known inequality \cite{scarlett2019introductoryguidefanosinequality, gerchinovitz2019fanosinequalityrandomvariables, principe2010}, which will not be re-proven here.
\begin{theorem}[Fano's Inequality]
\label{thm:fano}
Let $X$ be a discrete random variable taking values in the set $\mathcal{X}$ with $|\mathcal{X}|=M$, and let $Z$ be any other random variable. Let $\hat{X}\left(Z\right)$ be any statistical estimator of $X$, and define the probability of error as $P_{\text{err}} \equiv \mathbb{P}\left(\hat{X}\left(Z\right) \neq X\right)$. Then, the following inequality is true.
$$H\left(X|Z\right) \leq H\left(P_{\text{err}}\right) + P_{\text{err}}\text{log}_b \left(M-1\right)$$
Here, $H\left(P_{\text{err}}\right)$ is the binary entropy of $P_{\text{err}}$.
\end{theorem}
\begin{corollary}[Fundamental Under-Determination of the Inversion Problem]
\label{cor:under-determination}
Let $X$ and $Z$ be discrete random variables with images $\mathcal{X}\subset\mathbb{R}^D$ and $\mathcal{Z}\subset\mathbb{R}^E$, respectively, with $E<D$. Let $f_{\theta}:\mathcal{X} \to \mathcal{Z}$ be a deterministic encoder. Define an estimator $\hat{X}\left(Z\right)$, and define $P_{\text{err}} \equiv \mathbb{P}\left(\hat{X}\left(Z\right) \neq X\right)$. Then, $P_{\text{err}}>0$.
\end{corollary}
\begin{proof}
Let $X$ and $Z$ be discrete random variables with images $\mathcal{X}\subset\mathbb{R}^D$ and $\mathcal{Z}\subset\mathbb{R}^E$, respectively, with $E<D$. Let $f_{\theta}:\mathcal{X} \to \mathcal{Z}$ be a deterministic encoder. Define an estimator $\hat{X}\left(Z\right)$, and define $P_{\text{err}} \equiv \mathbb{P}\left(\hat{X}\left(Z\right) \neq X\right)$.\\
Since $E<D$, by Theorem \ref{thm:encoder_non-injectivity}, $f_{\theta}$ cannot be injective on $\text{supp}\left(X\right)$. Therefore, $\exists\mathbf{x}_1, \mathbf{x}_2 \in \text{supp}\left(X\right)$ such that $\mathbf{x}_1 \neq \mathbf{x}_2$ and $f_{\theta}\left(\mathbf{x}_1\right)=f_{\theta}\left(\mathbf{x}_2\right)$.\\
Since $f_{\theta}$ is non-injective, by Proposition \ref{prop:non-injectivity_info_loss}, $H\left(X|Z\right)>0$. That is, there is irreducible uncertainty about $X$, given $Z$. In other words, by Definition \ref{def:mutual_info}, $I\left(X;Z\right) \equiv H\left(X\right)-H\left(X|Z\right)<H\left(X\right)$, meaning the encoder has permanently removed some information about $X$.\\
By Lemma \ref{lem:entropy-dependence}, $H\left(X|Z\right)>0$ implies that no true inverse $f_{\theta}^{-1}$ exists.\\
And, by Theorem \ref{thm:fano}, $H\left(X|Z\right)>0 \implies P_{\text{err}}>0$.
\end{proof}
\noindent Corollary \ref{cor:under-determination} clearly establishes that the probability of error $P_{\text{err}}$ for any estimator $\hat{X}\left(Z\right)$ is bounded away from 0. Thus, the uncertainty about $X$ given $Z$ is fundamental, and is not due to any limitation in an attacker's algorithmic or computational power. This, again, is what is meant when it is said that the inversion problem is under-determined. The encoder has removed enough information that there is no unique solution to the inverse problem.\\
In higher-dimensional cases, which are extremely common in ML applications, the cardinality of the pre-image set $\mathcal{P}\left(\Psi\right)$ is quite often infinite within the mathematical input space $\mathbb{R}^D$. However, since any real ML pipeline only sees quantized data of finite precision, $\mathcal{P}\left(\Psi\right)$ is not infinite, but is still extremely large. Assuming 32-bit precision, which is the most common for training and inference in ML pipelines, the input space $\mathbb{R}^D$ becomes the following.
$$\mathcal{X} = \left\{0,1\right\}^{32D}$$
Therefore, while $\mathcal{P}\left(\Psi\right)$ may be finite, the probability of an attacker correctly identifying the correct input approaches the following.
$$P_{\text{recon}} \approx \frac{1}{|\mathcal{P}\left(\Psi\right)|}=\frac{1}{2^{32D}} \approx 0$$
Therefore, in most ML pipelines, even for an optimal attacker, the probability of accurately reconstructing any given input via encoder inversion collapses to 0.

\subsection{Information-Theoretic Loss and Under-Determination (Training or Inference Environment Attacker)}

In the previous subsection, the case of an optimal attacker, who would have knowledge of the encoder and its parameters, was considered for completeness and rigor. However, given the architectural isolation of the encoder under the VEIL architecture, this would require such an attacker to have breached the security of the Source Environment and even the source database itself. Not only does this pose several significant challenges to an attacker, but this is also extremely atypical of strategies employed by attackers seeking to steal sensitive data from an ML pipeline \cite{10.1145/3624010}. Given the extreme difficulties and costs associated with directly attacking a database, many attackers instead opt to intercept sensitive data in flight as it travels through the ML pipeline, typically en route to the prediction API \cite{tramèr2016stealingmachinelearningmodels}.\\
In this subsection, a cloud-side attacker, who has compromised the Training and Inference Environments, will be considered, as this is aligned with security research \cite{liu2025datareconstructionattacksdefenses, 10.1145/2810103.2813677}, real-world observations \cite{ibm2024mlops, siposova2023dataexfiltration}, and the threat model addressed by the VEIL architecture.\\
A cloud-side attacker is assumed to have knowledge of the latent representation space $\mathcal{Z}$, the downstream ML model $g_{\phi}$, the target space $Y$, and any infrastructure in the Training and Inference Environments, including storage solutions, model registries, and logs, which will only have knowledge of the latent representations.\\
A cloud-side attacker has no knowledge of the encoder $f_{\theta}$, its domain $\mathbb{R}^D$, or even the dimensionality $D$ of the input space. To give the attacker a little help, assume the attacker has enough expertise in ppML to guess that the latent encodings are the result of an autoencoder, and is familiar enough with autoencoders to correctly guess that $D>E$.\\
From the attacker's epistemic perspective, given a particular latent encoding $\Psi$, the preimage set is not a single set, but a family of sets. The preimage set is given by the following.
$$\mathcal{P}\left(\Psi\right)=\bigcup_{d=E+1}^{\infty}\bigcup_{f\in\mathcal{F}_d}\left\{\mathbf{x}\in\mathbb{R}^d:f\left(\mathbf{x}\right)=\Psi\right\}$$
Here, $d$ is the possible input dimension, which is unknown but must satisfy $d>E$, and $\mathcal{F}_d$ is the set of all continuous functions $f:\mathbb{R}^d\to\mathbb{R}^E$ that possibly could have produced the encoder's output $\Psi$. This union is vast. For each candidate dimension, $d$, the set $\mathcal{F}_d$ is uncountable. For each $f\in\mathcal{F}_d$, the preimage $\mathcal{P}\left(\Psi\right)$ is typically a $\left(d-E\right)$-dimensional manifold. The cloud-side attacker must consider all possible input domains and encoders consistent with the observation of $\Psi$. Thus, the attacker's preimage is not a single set, but a countably infinite family of uncountable manifolds across all possible input dimensions.\\
When the attacker does not know $f_{\theta}$ or $\mathcal{X}$, the conditional entropy becomes $H\left(X|Z,\mathcal{I}_{\text{cloud}}\right)$, where $\mathcal{I}_{\text{cloud}}$ is the information available to the attacker. The attacker must model $\mathcal{X}$ as the union of all possible input spaces and all encoders. Thus, the attacker's set of all possible inputs is given by the following.
$$\mathcal{X}_{\text{attacker}}=\bigcup_{d=E+1}^{\infty}\mathbb{R}^d$$
This set is unbounded in dimension, not measurable as a finite-entropy space, and, therefore, cannot be assigned any finite prior by the attacker without additional information. Therefore, the entropy of this union is $H\left(X|Z, \mathcal{I}_{\text{cloud}}\right)=+\infty$. This follows because the entropy of a mixture model is bounded below by the entropy of its components. This mixture consists of infinitely many components with increasing dimensionality. Each additional dimension introduces additional uncertainty. Thus, the entropy grows without bound.\\
Formally, $\forall d \in \mathbb{N}$ such that $d>E$, let $X_d$ be a hypothetical random variable supported on $\mathbb{R}^d$. The attacker must therefore consider the mixture $X=X_d$ with probability $p_d=\mathbb{P}\left(D=d\right)$ where $\sum_{d=E+1}^{\infty}p_d = 1$. Thus, the attacker's full prior over $X$ must be the following.
$$\mathbb{P}\left(X\in A|Z=\Psi\right)=\sum_{d=E+1}^{\infty} p_d\mathbb{P}\left(X_d \in A_d|Z=\Psi\right)$$
Here, $A_d=A \cap\mathbb{R}^d$. This mixture expresses the fact that the attacker must evaluate reconstruction under uncertainty about the space containing $X$, how large that space is, and how many degrees of freedom the original sample has.\\
From this, the following facts are considered to analyze $H\left(X|Z\right)$ from the attacker's epistemic perspective. First, even if the attacker knew $D=d$ and knew the correct encoder $f_{\theta}$ for that $d$, the mapping $f_{\theta}:\mathbb{R}^d \to \mathbb{R}^E$ has underdetermined preimages, as $\text{dim}\mathcal{P}\left(\Psi\right)=d-E$; therefore $H\left(X_d|Z=\Psi\right)>0$.\\
However, the cloud-side attacker does not know $D$. Therefore, the entropy includes $H\left(D|Z\right)$, which is unbounded if the attacker cannot rule out arbitrarily large input spaces. Furthermore, for each candidate dimension $d$, the entropy $H\left(X_d|Z=\Psi\right)$ grows at least linearly in $d$, because each new input coordinate introduces an independent degree of freedom in the preimage manifold. Therefore, by the Law of Total Entropy, the attacker's conditional entropy on $X$ given $Z$ is as follows.
$$H\left(X|Z\right)=H\left(D|Z\right) + \sum_{d=E+1}^{\infty} p_dH\left(X_d|Z, D=d\right)$$
Even if $p_d$ decays exponentially, this weighted sum diverges. Therefore, $H\left(X|Z\right)=+\infty$.\\
By Theorem \ref{thm:fano}, the following inequality must hold.
$$H\left(X|Z\right) \leq H\left(P_{\text{err}}\right)+P_{\text{err}}\text{log}_b\left(|\mathcal{X}|-1\right)$$
If $H\left(X|Z\right)=+\infty$, then this inequality can only be true under two conditions.
\begin{enumerate}
\item $P_{\text{err}}=1$
\item $|\mathcal{X}|=+\infty$
\end{enumerate}
In either case, the only possible interpretation is $P_{\text{recon}}=\mathbb{P}\left(\hat{X}\left(Z\right)=X\right)=0$. That is, the probability that the attacker succeeds in reconstructing the original data is not small, or negligible. It is identically 0 under the stated assumptions.

\subsection{Summary of Non-Invertibility Guarantees}

In this section, it has been shown that ICA encodings are non-invertible in multiple reinforcing senses. Taken together, these arguments support the claim that ICA provides structural, architectural non-invertibility, rather than relying on probabilistic obfuscation (as in DP) or cryptographic protection (as in HE). This non-invertibility is a key pillar of ICA's privacy guarantees, even under post-quantum threats. This makes ICA the clear standard for secure, privacy-preserving ML deployments in the looming post-quantum age. The power of quantum computers will render both probabilistic obfuscation for DP and traditional cryptographic protection for HE obsolete. However, regardless of how powerful future computing systems become, they will never be able to extract information not present in a given latent representation.

\section{Privacy Attack Simulations}

While the previous section rigorously establishes the non-invertibility of the latent encodings produced by ICA and the VEIL architecture, this section acknowledges that there are various ways a bad actor may attack a machine learning system. Here, several attacks are recreated and simulated against both VEIL and DP systems to measure and compare their performance under different threat scenarios.

\subsection{Models and Datasets}

To provide a thorough study and to demonstrate the ability of the VEIL system to withstand privacy attacks while maintaining the predictive utility of the protected downstream model, the attack simulation studies in this section utilized all of the models and datasets from previous sections of this paper, since VEIL has already been shown to preserve predictive utility in these specific pipelines. However, these datasets are not representative of the kinds of sensitive data that would require protection. To add realism to the attack simulation studies, the following datasets and models were incorporated.\\
The first additional study utilizes the Home Credit Default Risk dataset \cite{homecredit2018}. The Home Credit Default Risk dataset was released by Home Credit Group in conjunction with a Kaggle machine learning competition in 2018. The dataset is structured across multiple relational tables encompassing application-level data, bureau credit history, prior loan records, point-of-sale records, and installment payment histories. The primary training table comprises 356,255 loan application records characterized by features spanning applicant demographics, employment status, income, asset ownership, residential history, and social network attributes, among others.\\
When joined across its constituent tables, the effective feature space extends to several hundred variables, rendering the dataset genuinely high-dimensional. In this study, the final number of engineered features was $D=737$ with 537 numerical features and 200 binary features after applying one-hot encoding to categorical features. The supervised learning objective is binary classification: predicting whether a given applicant will experience payment difficulties, defined as a late payment exceeding a specified threshold on one or more of the first installments of a loan.\\
The dataset is characterized by substantial class imbalance, with defaulting applicants representing a minority of the total population. This posed a significant challenge for the DP pipeline, as DP is known to suffer from a loss of predictive performance in imbalanced learning scenarios \cite{farrand2020neither, rosenblatt2024differential, suriyakumar2021chasing}. While the other modeling pipelines did not require this, the data needed to be balanced to obtain usable results from the DP pipeline. To control for any confounding factors, all four modeling pipelines ingested records from the same balanced training, validation, and test datasets. This dataset's combination of multi-table relational structure, diverse feature types, and real-world lending context makes it a widely adopted benchmark for credit risk modeling research.\\
The predictive benchmarking study for this dataset utilized a downstream MLP classifier. Both autoencoder pipelines emitted a latent encoding of dimensionality $E=2$. The prediction benchmark results are shown in Table \ref{tab:home_credit_results}.
\begin{center}
\begin{tabular}{lcc}
\toprule
\textbf{Model Pipeline} & \textbf{Test ROC-AUC} & \textbf{vs. Raw} \\
\midrule
Raw Data & $0.7582$ & $+0.0000$ \\
DP & $0.7426$ & $-0.0156$ \\
Dense AE & $0.6676$ & $-0.0906$ \\
SCRAE & $0.7734$ & $+0.0152$ \\
\bottomrule
\end{tabular}
\captionof{table}{Home Credit Test ROC-AUC Comparison}
\label{tab:home_credit_results}
\end{center}
The second additional study utilizes the Default of Credit Card Clients dataset, which was compiled by Yeh and Lien \cite{yeh2009default} and published in conjunction with their 2009 study in \textit{Expert Systems with Applications} \cite{yeh2009comparisons}, in which various data mining techniques were compared for their accuracy in predicting the probability of credit card default. The dataset was drawn from a Taiwanese bank and comprises 30,000 credit card client records collected between April and September 2005. Each record is characterized by 24 features spanning client demographic attributes---including sex, education level, marital status, and age---credit limit, repayment status across six consecutive months, billed statement amounts, and prior payment amounts. The supervised learning objective is binary classification: predicting whether a given client will default on their payment in the following month. Its combination of behavioral payment history features, demographic attributes, and a clearly defined binary outcome has made it a widely adopted benchmark for credit risk modeling and consumer default prediction research.\\
In this study, the downstream predictor was an MLP classifier, and both autoencoder pipelines emitted a latent encoding of dimensionality $E=2$. The prediction benchmark results are shown in Table \ref{tab:cc_default_results}.
\begin{center}
\begin{tabular}{lcc}
\toprule
\textbf{Model Pipeline} & \textbf{Test ROC-AUC} & \textbf{vs. Raw} \\
\midrule
Raw Data & $0.7761$ & $+0.0000$ \\
DP & $0.7674$ & $-0.0087$ \\
Dense AE & $0.5834$ & $-0.1927$ \\
SCRAE & $0.7794$ & $+0.0033$ \\
\bottomrule
\end{tabular}
\captionof{table}{CC Default Test ROC-AUC Comparison}
\label{tab:cc_default_results}
\end{center}
The third and final additional study utilizes the Curated Breast Imaging Subset of the Digital Database for Screening Mammography (CBIS-DDSM) dataset \cite{sawyer-lee2016cbis, lee2017curated}, an updated and standardized version of the original DDSM, developed by researchers at Stanford University and released through The Cancer Imaging Archive in 2016. The underlying DDSM comprises 2,620 scanned film mammography studies from four institutional sources, and encompasses normal, benign, and malignant cases with verified pathological diagnoses.\\
In this study, the downstream predictor was an MLP classifier, and both autoencoder pipelines emitted a latent encoding of dimensionality $E=2$. The prediction benchmark results are shown in Table \ref{tab:cbis-ddsm_results}.
\begin{center}
\begin{tabular}{lcc}
\toprule
\textbf{Model Pipeline} & \textbf{Test ROC-AUC} & \textbf{vs. Raw} \\
\midrule
Raw Data & $0.6076$ & $+0.0000$ \\
DP & $0.5700$ & $-0.0376$ \\
Dense AE & $0.5377$ & $-0.0699$ \\
SCRAE & $0.6249$ & $+0.0173$ \\
\bottomrule
\end{tabular}
\captionof{table}{CBIS-DDSM Test ROC-AUC Comparison}
\label{tab:cbis-ddsm_results}
\end{center}
Therefore, all of the modeling pipelines subjected to the following attack simulation studied involved VEIL encoded models that outperformed the predictive benchmark set by their respective raw data pipelines.

\subsection{Input Reconstruction}

For completeness, the formal non-invertibility results developed in Section 9 were supplemented by direct empirical reconstruction-attack simulations. The purpose of these evaluations was to test whether an adversary, given protected model inputs, could recover the corresponding raw inputs $\mathbf{x} \in \mathcal{X}$ either exactly or approximately. In the VEIL pipeline, the protected inputs are latent vectors $\Psi \in \mathcal{Z}$; in the DP pipeline, they are the corresponding DP-protected inputs used by the downstream model. The threat model follows the reconstruction setting studied by Liu et al.\ \cite{liu2025datareconstructionattacksdefenses}, in which an attacker observes protected representations and attempts to infer the original data from which those representations were produced. Whereas the setting evaluated by Liu et al.\ centered on inputs protected by DP, the simulations reported here applied the same reconstruction logic to both DP-protected pipelines and VEIL-protected pipelines across the eight model--dataset pairs considered in this paper.\\
The attacker was assumed to possess a collection of exported protected inputs, denoted $Z=\left\{\Psi_i\right\}_{i=1}^{n}$ in the VEIL case, and to know that each vector had been derived from an unobserved source record $\mathbf{x}_i$. To mount the strongest practical form of the attack, the attacker was further granted supervised attack data consisting of matched pairs $\left(\Psi_i,\mathbf{x}_i\right)$ for decoder training. The attack setup therefore required observed protected inputs, matched raw records for a supervised attack-training subset, a decoder architecture and reconstruction objective, and held-out matched records for evaluation. The attack objective was defined as learning a reconstruction rule $h_{\omega}:\mathcal{Z}\to\mathcal{X}$ such that $h_{\omega}\left(\Psi_i\right)\approx\mathbf{x}_i$ on held-out examples. If an exact inverse existed, this rule would correspond to $f_{\theta}^{-1}:\mathcal{Z}\to\mathcal{X}$. If no exact inverse existed, the remaining empirical question was whether a learned approximation could nevertheless recover useful information about the original records.\\
The evaluation was implemented with two complementary components. First, a structural validation stage was used to determine whether the assumptions required by Theorem \ref{thm:encoder_non-injectivity} and Corollary \ref{cor:encoder_non-invertibility} were satisfied. In particular, the encoder was checked for dimensional compression $E<D$, continuity of the mapping, and genuine input variability exceeding the latent dimension. When these conditions hold, the original input cannot be uniquely identified from the latent vector because multiple distinct inputs must share the same latent representation. This stage therefore maps the empirical simulation back to the mathematical assertions in Section 9: failure of exact inversion is not attributed to insufficient attacker optimization, but to the absence of a well-defined inverse function.\\
Second, a decoder stage was used to test whether approximate reconstruction was possible in practice. A supervised multilayer decoder was trained to map protected inputs back to raw feature vectors and was evaluated on held-out records not used for decoder fitting. Reconstruction quality was measured separately according to feature type: continuous and numerical features were evaluated by reconstruction error, while binary or categorical indicators were evaluated by classification-style agreement with the true feature values. The decoder was not judged in isolation. Its performance was compared against a naive baseline that ignored the protected input entirely and predicted only the empirical center of the training data: the feature-wise mean for numerical variables and the feature-wise mode for binary or categorical variables. This baseline was used for both the DP and VEIL decoder attacks, and represents the information available from the marginal distribution of the data alone. A decoder was therefore considered successful only if it improved materially on this naive estimator, because performance at or near the baseline would indicate that the protected representation supplied little or no usable record-specific information. Statistical significance was assessed by permutation testing, in which the correspondence between protected inputs and source records was shuffled to form a null distribution for the observed decoder advantage.\\
This empirical test is conservative relative to the VEIL deployment threat model. Under the VEIL architecture, the encoder remains confined to the trusted Source Environment, no decoder is deployed, and the untrusted Training and Inference Environments observe only latent vectors and downstream model artifacts. By contrast, the reconstruction simulation grants the attacker supervised access to paired examples of latent vectors and raw feature rows. The decoder test therefore asks whether inversion would remain ineffective even under an attack setting that is stronger than the Training or Inference Environment Attacker analyzed in Section 9.4, and stronger in practical learning capability than the Source Environment Attacker analyzed in Section 9.3.\\
Even under this favorable supervised formulation, a fundamental obstacle remains. Supervised learning ordinarily assumes that there exists a target function $f$ such that $y=f\left(x\right)+\epsilon$, and that the learning problem consists of estimating that function from examples. In the reconstruction setting, however, the desired function is the inverse encoder map $f_{\theta}^{-1}$. Section 9 proves that, under the stated dimensionality and continuity assumptions, such a function does not exist on any open region of the input space. The decoder is therefore not estimating a difficult but well-defined inverse; it is attempting to select one representative raw input from an under-determined preimage set containing multiple valid candidates. For this reason, the supervised decoder attack is expected to struggle even when paired examples are supplied, and its performance relative to the naive baseline is the relevant empirical measure of whether any useful approximate reconstruction has been learned.
\begin{center}
\begin{tabular}{lcc}
\toprule
\textbf{Model (DP)} & \textbf{Decoder Advantage} & \textbf{$p$-value} \\
\midrule
MNIST & $+0.0029$ & $0.0010$ \\
Ames Housing & $+0.0041$ & $0.0010$ \\
YearPredictionMSD & $-0.0124$ & $0.6064$ \\
Fashion-MNIST & $-45.1287$ & $0.7792$ \\
E2006 & $-0.0021$ & $0.9341$ \\
Home Credit & $+0.0000$ & $0.7723$ \\
CC Default & $+0.0299$ & $0.0010$ \\
CBIS-DDSM & $+0.0267$ & $0.0010$ \\
\bottomrule
\end{tabular}
\captionof{table}{Differential Privacy Reconstruction Attack Results}
\label{tab:dp_recon_results}
\end{center}
\begin{center}
\begin{tabular}{lcc}
\toprule
\textbf{Model (VEIL)} & \textbf{Decoder Advantage} & \textbf{$p$-value} \\
\midrule
MNIST & $-0.0002$ & $0.7822$ \\
Ames Housing & $+0.0000$ & $0.4554$ \\
YearPredictionMSD & $-0.0002$ & $0.6436$ \\
Fashion-MNIST & $-45.3064$ & $0.7862$ \\
E2006 & $-0.0068$ & $0.9540$ \\
Home Credit & $+0.0000$ & $0.7723$ \\
CC Default & $-0.0001$ & $0.6436$ \\
CBIS-DDSM & $-97.3507$ & $0.5050$ \\
\bottomrule
\end{tabular}
\captionof{table}{VEIL Reconstruction Attack Results}
\label{tab:veil_recon_results}
\end{center}
The DP results in Table \ref{tab:dp_recon_results} show mixed reconstruction behavior. Several DP-protected pipelines exhibit statistically significant positive decoder advantage, indicating that the learned decoder recovered some record-specific structure beyond the naive baseline in those cases. Other DP pipelines produced near-zero or negative advantage, indicating that the decoder either failed to improve on the baseline or performed worse than it. Thus, the DP simulations do not support a uniform reconstruction-resistance conclusion across all evaluated datasets and models.\\
The VEIL results in Table \ref{tab:veil_recon_results} are materially different. Across the eight evaluated model--dataset pairs, the decoder advantage is either near zero or negative, and the reported permutation $p$-values do not indicate statistically significant improvement over the null reconstruction baseline. These results are consistent with the structural claims established in Section 9: the VEIL encoder compresses the input into a lower-dimensional, continuous, task-aligned representation for which no exact inverse exists, and the supervised decoder simulations did not identify a useful approximate inverse. Accordingly, the empirical reconstruction attacks provide practical support for the conclusion that VEIL latents do not expose the original input records in a reconstructable form under the evaluated threat model.\\

\subsection{Attribute Inference}

Whereas the reconstruction evaluation in the previous subsection tests whether the full input, $\mathbf{x}$, can be recovered from a protected representation, attribute inference asks a narrower but operationally important question: whether an adversary can predict a particular sensitive attribute, $S$, for a record without reconstructing the full record. In this section, let $R$ denote the protected record representation available outside the trusted Source Environment. For the VEIL pipeline, $R=\Psi=f_{\theta}\left(\mathbf{x}\right)$ is the latent vector emitted by the encoder. For the DP pipeline, $R$ denotes the corresponding DP-protected model input used by the downstream model. The attacker's objective is to learn a predictor $a_{\omega}:R\mapsto S$ and then apply it to protected records whose sensitive attribute is unknown to the attacker. This threat is related to the attribute-inference and model-inversion literature, in which correlated model inputs, representations, or outputs are used to infer sensitive attributes that were not intended to be disclosed \cite{FredriksonJR15, GongTMHSSSS14, JayaramanE22}.\\
The information-theoretic reason that such an attack can be possible is captured by the Privacy Funnel framework \cite{makhdoumi2014informationbottleneckprivacyfunnel}. A privacy-preserving mapping is designed to retain information about a useful disclosure variable while suppressing information about a private variable. If the representation $R$ is useful for a supervised prediction task with label $Y$, and if the sensitive attribute $S$ is statistically correlated with the task label, with task-relevant input features, or with population structure present in the input distribution, then $R$ may still have positive mutual information with $S$. In that case, a classifier trained on labeled examples of $R$ may perform better than a naive baseline even though $R$ remains non-invertible as a representation of the full input record. Attribute inference therefore tests attribute-level leakage, not decoder-level reconstruction.\\
The threat model used in the present simulations assumes that the attacker obtains protected representations from a proxy response, downstream storage, logs, model-serving traces, or another location beyond the trusted Source Environment. However, access to protected representations alone is not sufficient to train the attack. The attacker must also obtain sensitive-attribute labels for overlapping records and must be able to join those labels to the corresponding protected representations. This requires both halves of the linkage: protected records, $R_i$, and labeled auxiliary information, $S_i$, connected by identifiers, quasi-identifiers, stable request handles, timestamps, row order, or some other operational join key. Without a join key, the attacker has two disconnected datasets. Without labels, the attacker has protected records but no supervised training signal.\\
Once the attacker has constructed labeled pairs $\left(R_i,S_i\right)$, the attack is straightforward supervised learning. An attack classifier is trained to predict $S$ from $R$, and its performance is compared with a baseline that does not use the protected representation. The attack is meaningful only when it provides a positive practical advantage over that baseline and the advantage is statistically significant under label permutation. This comparison is important because some sensitive attributes can be guessed from population frequencies or ordinary imputation even without access to the target model or protected representation; a privacy-relevant attribute-inference result should therefore measure the incremental signal exposed by $R$ rather than merely the predictability of $S$ in the underlying population \cite{JayaramanE22}.\\
In many attribute-inference studies, $S$ is operationalized as a binary indicator, such as a protected-category flag or a one-hot encoding of a categorical attribute. Several benchmark datasets used in the present study do not naturally contain categorical sensitive attributes of this form. MNIST and Fashion-MNIST are image datasets whose raw inputs are pixel brightness values, YearPredictionMSD contains numerical timbre features, and E2006 consists of TF-IDF representations of SEC filings. For these cases, a binary attack attribute was therefore constructed either from the raw input variables or from the task labels so that the same supervised attribute inference protocol could be applied across datasets. The derived attributes were:
\begin{itemize}
\item MNIST: \texttt{is\_odd\_number} (positive rate: $0.5082$)
\item YearPredictionMSD: \texttt{timbre\_001} $\geq$ median
\item Fashion-MNIST: \texttt{is\_upper\_body\_garment} (positive rate: $0.4000$)
\item E2006: \texttt{payment\_history\_volatility} $\geq$ median
\item CBIS-DDSM: \texttt{dataset\_group} = ``mass'' (positive rate: $0.5131$)
\end{itemize}
This construction remains within the intended scope of attribute inference: the protected attribute need not be an explicit input column if it is statistically recoverable from the representation. Ouaari et al., drawing on adversarial reconstruction work, note that partially reconstructed face images may expose enough visual information to infer ethnicity or skin color even when the person cannot be specifically identified \cite{ouaari2023robust, xiao2020adversarial}. Thus, the experiments below test whether protected representations retain exploitable signal about a specified binary private attribute, whether that attribute is native to the dataset or defined for the stress test.\\
Inference is most feasible when protected records are exported or stored in bulk together with stable handles, when logs preserve request-response associations, or when an attacker has insider, storage, or application-layer access that bypasses transient network protections. It is also more feasible when the target attribute has low or moderate entropy and is correlated with the supervised task or with input features deliberately preserved for utility. Inference is constrained when the auxiliary data do not cover the same population, when all identifiers and quasi-identifiers are removed before protected records leave the Source Environment, when downstream identifiers are internal and opaque, when no bulk export path exists, and when protected representations are handled ephemerally. Point-lookup inference patterns also raise the bar for a network-only attacker because the attacker must observe and correlate both the request and response streams; storage or log compromise, by contrast, can collapse this constraint if the stored artifacts remain joinable.\\
The experiments reported in Tables \ref{tab:dp_attr_results} and \ref{tab:veil_attr_results} apply this attribute inference protocol to all eight model-dataset pairs considered in the privacy attack study. For each dataset, the same attack logic is evaluated against a modeling pipeline protected by DP and against the corresponding pipeline protected by VEIL\@. The ``vs. Baseline'' column reports the attack advantage over the baseline predictor accuracy, while the $p$-value reports the permutation significance of the observed advantage. A statistically significant positive advantage indicates that the protected representation exposed exploitable information about the evaluated sensitive attribute. A zero, negative, or non-significant advantage indicates that the simulation did not identify reliable attribute-level signal beyond the baseline.
\begin{center}
\begin{tabular}{lcc}
\toprule
\textbf{Model (DP)} & \textbf{vs. Baseline} & \textbf{$p$-value} \\
\midrule
MNIST & $+0.4380$ & $0.0010$ \\
Ames Housing & $-0.0093$ & $0.9375$ \\
YearPredictionMSD & $+0.1038$ & $0.0010$ \\
Fashion-MNIST & $+0.3902$ & $0.0010$ \\
E2006 & $+0.1196$ & $0.0010$ \\
Home Credit & $+0.0000$ & $1.0000$ \\
CC Default & $+0.0000$ & $0.9850$ \\
CBIS-DDSM & $+0.0799$ & $0.0010$ \\
\bottomrule
\end{tabular}
\captionof{table}{DP Attribute Inference Attack Results}
\label{tab:dp_attr_results}
\end{center}
\begin{center}
\begin{tabular}{lcc}
\toprule
\textbf{Model (VEIL)} & \textbf{vs. Baseline} & \textbf{$p$-value} \\
\midrule
MNIST & $+0.3535$ & $0.0010$ \\
Ames Housing & $-0.0106$ & $0.9379$ \\
YearPredictionMSD & $+0.0939$ & $0.0628$ \\
Fashion-MNIST & $+0.3973$ & $0.0010$ \\
E2006 & $+0.0000$ & $1.0000$ \\
Home Credit & $+0.0000$ & $1.0000$ \\
CC Default & $-0.0015$ & $0.9990$ \\
CBIS-DDSM & $-0.0055$ & $0.6634$ \\
\bottomrule
\end{tabular}
\captionof{table}{VEIL Attribute Inference Attack Results}
\label{tab:veil_attr_results}
\end{center}
\noindent DP results in Table \ref{tab:dp_attr_results} show statistically significant positive attribute-inference advantage for MNIST, YearPredictionMSD, Fashion-MNIST, E2006, and CBIS-DDSM\@. Ames Housing, Home Credit, and CC Default do not show significant positive advantage under the reported test. These results demonstrate that DP protection did not uniformly prevent attribute inference in the evaluated pipelines. This does not contradict the purpose of DP, since DP limits the effect of any individual training record on a released computation; it does not necessarily remove population-level or task-level correlations between a sensitive attribute and the features that a useful model must preserve.\\
The VEIL results in Table \ref{tab:veil_attr_results} show statistically significant positive attribute-inference advantage for MNIST and Fashion-MNIST\@. YearPredictionMSD shows a positive advantage, but its $p$-value of $0.0628$ is not statistically significant under the usual $0.05$ threshold. Ames Housing, E2006, Home Credit, CC Default, and CBIS-DDSM do not show significant positive advantage. Relative to the DP simulations, VEIL reduces the number of significant attribute-inference outcomes from five of eight datasets to two of eight datasets, and eliminates statistically significant positive advantage for YearPredictionMSD, E2006, and CBIS-DDSM under the reported protocol. The Fashion-MNIST result remains significant for both approaches, and the MNIST result remains significant for both approaches although the measured VEIL advantage is lower than the DP advantage.\\
The significance of the VEIL positive cases is precise. They do not show that VEIL latent vectors are invertible, nor do they overturn the topological and information-theoretic non-invertibility results established in Section 9. They show instead that, if protected artifacts can be linked to external sensitive-attribute labels, a classifier may exploit residual task-aligned signal in the protected representation to infer a particular attribute. The privacy risk identified here is therefore not reconstruction of the full source record, but enrichment by linkage. It is a concrete failure mode at the pipeline layer rather than a refutation of the encoder-layer non-invertibility result.\\
Accordingly, VEIL deployments should enforce an operational rule that no latent artifact, request, response, or downstream model input may be stored, indexed, or exported in a form that enables auxiliary information to be joined back to it. Preserving the VEIL trust boundary requires the encoder to remain non-invertible and protected inside the Source Environment, but it also requires MLOps controls that eliminate stable join handles, prevent bulk extraction of protected representations, audit access to protected artifacts, and treat highly task-correlated sensitive attributes as attributes requiring explicit leakage testing before deployment.

\subsection{Property Inference}

Property inference is a population-level privacy attack. Unlike reconstruction, which attempts to recover an input record $\mathbf{x}$, and unlike attribute inference, which attempts to predict a sensitive attribute $S_i$ for an individual protected record, property inference attempts to estimate an aggregate property of a collection of records. The target may be a demographic composition, the fraction of records satisfying a sensitive predicate, the prevalence of a subgroup, or another population statistic not intended to be revealed. Prior work has shown that learned models and collaborative learning updates can reveal unintended global properties of the data used to produce them, even when those properties are not the explicit prediction target \cite{10.1145/3243734.3243834, melis2018exploitingunintendedfeatureleakage}. In the VEIL setting, the analogous question is whether a batch of protected representations reveals a statistic of the source population from which that batch was drawn.\\
Let $S\in\{0,1\}$ denote a sensitive predicate, and let
\begin{equation}
\pi_S=\frac{1}{B}\sum_{i=1}^{B}\mathbf{1}\left\{S_i = 1 \right\}
\end{equation}
denote the true prevalence of that predicate in a target batch of size $B$. The attacker's goal is not to determine every $S_i$ exactly, but to construct an estimator $\widehat{\pi}_S$ whose error $\left|\widehat{\pi}_S-\pi_S\right|$ is lower than the error of a baseline estimator that does not use the protected representation. For the VEIL pipeline, the protected representation is the latent vector $R=\Psi=f_{\theta}\left(\mathbf{x}\right)$. For the DP pipeline, $R$ denotes the corresponding DP-protected model input. The attacker therefore observes a batch $\{R_i\}_{i=1}^{B}$ and attempts to infer a property of the underlying population that generated the batch.\\
The threat model is intentionally narrower than unrestricted oracle access to the encoder. In a VEIL deployment, the encoder is not exposed as a general-purpose public transformation service. Deployed models and approved data sources must be registered with the encoder, which is commonly implemented behind a database proxy that controls query traffic to the source data and ensures that model inputs are encoded before leaving the Source Environment. General querying of the encoder with arbitrary auxiliary records is therefore disallowed. The reference-population attack studied in the broader property-inference literature---in which an adversary repeatedly encodes auxiliary populations with known property proportions and trains a meta-estimator---is not treated as a normal VEIL capability. Such an attack would require either an improperly exposed encoder, leaked pre-encoded reference populations, or an insider path that bypasses the intended registration and access-control regime.\\
For this reason, the variant considered here relies on joinability, as in the attribute-inference experiment in Section 10.3. The attacker is assumed to obtain protected representations outside the Source Environment from a proxy response, model-serving trace, downstream storage system, log, vector database, or compromised application component. To turn those representations into a property-inference training set, the attacker must also obtain auxiliary sensitive-attribute information and join it to protected artifacts by means of identifiers, quasi-identifiers, stable request handles, timestamps, row order, or some other operational linkage. If the attacker has only an unordered collection of latent vectors and no labels, no aggregate-labeled reference populations, and no join key, then the attacker has no supervised signal from which to calibrate a property estimator.\\
Given joinable auxiliary data, there are two natural ways to operationalize the attack. In a predict-then-count attack, the adversary first trains an attribute predictor $a_{\omega}:R\mapsto S$ on joined pairs $(R_i,S_i)$, applies it to the target batch, and then estimates $\pi_S$ by averaging the predicted sensitive labels. In a mixture-proportion attack, the adversary instead fits class-conditional summaries or density models for $\mathbb{P}(R\mid S=0)$ and $\mathbb{P}(R\mid S=1)$ using the joined auxiliary data, and then estimates the mixture weight that best explains the observed target batch. The second approach illustrates why property inference can be feasible even when individual attribute inference is weak: noisy individual evidence may average out over a sufficiently large batch, allowing small distributional shifts in $R$ to become visible at the population level.\\
The baseline for this evaluation is a property estimator that does not use the protected representation. Operationally, this corresponds to predicting the target prevalence from the attacker's prior information, such as the empirical prevalence observed in the attack-training population, rather than from the target batch's latent vectors. Attacker success is measured by mean absolute error (MAE) in the estimated aggregate property. The ``vs. Baseline'' column in Tables \ref{tab:dp_prop_results} and \ref{tab:veil_prop_results} reports the reduction in MAE obtained by the attack relative to that representation-free baseline,
\begin{dmath}
\Delta_{\mathrm{MAE}}=\mathrm{MAE}_{\mathrm{baseline}}-\mathrm{MAE}_{\mathrm{attack}}.
\end{dmath}
Thus, a positive value indicates that the attack estimated the population property more accurately than the baseline; a value near zero indicates no practical improvement; and a negative value, if observed, would indicate that the protected representation made the estimate worse than the baseline. The reported $p$-value is obtained by a permutation test and measures whether the observed MAE reduction is statistically distinguishable from the null condition in which the representation is not reliably associated with the target property.\\
Property inference is feasible when three conditions hold simultaneously. First, the protected representation must preserve information statistically correlated with the property. In VEIL, this can occur when the property is correlated with the supervised task label, with task-relevant input features, or with population structure that the encoder must retain to preserve predictive utility. Second, the attacker must have a training signal, either through improperly exposed reference encodings or, in the deployment-relevant case considered here, through joinable labeled latent records. Third, the attacker must have enough target records for aggregate statistics to stabilize; a single latent vector rarely identifies a population property, but a large batch may reveal a shift in means, variances, class-conditional densities, or predicted sensitive-label frequencies.\\
Property inference is constrained when any of these requirements fail. If the target property is independent of the encoded task-relevant features, then $\mathbb{P}(R\mid S=1)$ and $\mathbb{P}(R\mid S=0)$ will not be meaningfully distinguishable for the purposes of estimating the property. If the attacker lacks join keys or aggregate-labeled reference populations, there is no calibration data for the estimator. If the target batch is small, sampling variation can dominate the weak aggregate signal. If operational controls prevent persistent identifiers, stable row ordering, bulk latent exports, arbitrary encoder queries, and long-lived storage of request--response associations, then the attack is reduced to observing protected vectors without the auxiliary structure needed to interpret them.\\
The leakage tested here is therefore precise. A successful property-inference attack would indicate that a batch of protected representations exposes some population-composition information beyond what the attacker could infer from priors alone. It would not imply that the attacker can reconstruct raw records, recover arbitrary input features, identify the true sensitive attribute of each individual record, or invert the VEIL encoder. Conversely, a failed or non-significant property-inference attack does not prove that every possible population statistic is hidden; it shows that, for the tested dataset, model, property, sample size, and auxiliary-information assumptions, the simulation did not identify reliable aggregate leakage beyond the baseline. This distinction is important because ICA is designed to eliminate reconstructable raw inputs, while property inference measures residual distributional correlation in task-aligned representations.
\begin{center}
\begin{tabular}{lcc}
\toprule
\textbf{Model (DP)} & \textbf{vs. Baseline} & \textbf{$p$-value} \\
\midrule
MNIST & $+0.0385$ & $0.0629$ \\
Ames Housing & $+0.0175$ & $0.3750$ \\
YearPredictionMSD & $+0.0571$ & $0.0130$ \\
Fashion-MNIST & $+0.1482$ & $0.0099$ \\
E2006 & $+0.0215$ & $0.2038$ \\
Home Credit & $+0.0000$ & $0.6250$ \\
CC Default & $+0.0000$ & $0.4905$ \\
CBIS-DDSM & $+0.0008$ & $0.3750$ \\
\bottomrule
\end{tabular}
\captionof{table}{DP Property Inference Attack Results}
\label{tab:dp_prop_results}
\end{center}
\begin{center}
\begin{tabular}{lcc}
\toprule
\textbf{Model (VEIL)} & \textbf{vs. Baseline} & \textbf{$p$-value} \\
\midrule
MNIST & $+0.0439$ & $0.0420$ \\
Ames Housing & $+0.0136$ & $0.5000$ \\
YearPredictionMSD & $+0.0339$ & $0.0889$ \\
Fashion-MNIST & $+0.0000$ & $0.5644$ \\
E2006 & $+0.0234$ & $0.1838$ \\
Home Credit & $+0.0000$ & $0.6250$ \\
CC Default & $+0.0001$ & $0.4905$ \\
CBIS-DDSM & $+0.0001$ & $0.3750$ \\
\bottomrule
\end{tabular}
\captionof{table}{VEIL Property Inference Attack Results}
\label{tab:veil_prop_results}
\end{center}
The DP results in Table \ref{tab:dp_prop_results} show two statistically significant positive property-inference results under the usual $0.05$ threshold: YearPredictionMSD, with an MAE improvement of $+0.0571$ and $p=0.0130$, and Fashion-MNIST, with an MAE improvement of $+0.1482$ and $p=0.0099$. MNIST shows a positive improvement of $+0.0385$, but its $p$-value of $0.0629$ does not meet the $0.05$ threshold. Ames Housing, E2006, Home Credit, CC Default, and CBIS-DDSM do not show statistically significant improvement over the baseline. These results indicate that DP protection did not uniformly eliminate aggregate-property leakage in the evaluated pipelines. This is expected in principle: DP constrains the influence of individual records, but it does not necessarily remove all distributional information about groups or population composition when that information is correlated with features retained for predictive utility.\\
The VEIL results in Table \ref{tab:veil_prop_results} show a different pattern. MNIST exhibits a statistically significant positive MAE improvement of $+0.0439$ with $p=0.0420$, indicating that the tested latent batches contained exploitable aggregate information about the evaluated property. YearPredictionMSD and E2006 show positive MAE improvements of $+0.0339$ and $+0.0234$, respectively, but neither is statistically significant under the $0.05$ threshold. Ames Housing, Fashion-MNIST, Home Credit, CC Default, and CBIS-DDSM do not show significant property-inference advantage. Thus, under the reported protocol, VEIL yields one significant property-inference outcome across the eight evaluated model--dataset pairs, compared with two significant outcomes for DP.\\
The strongest logical conclusion supported by these results is not that VEIL eliminates all possible aggregate-property leakage. Rather, the results show that aggregate inference is possible only when the protected representations remain statistically aligned with the property and when the attacker has the auxiliary linkage needed to interpret those representations. The MNIST result is consistent with a property that is highly coupled to task-relevant visual structure, so the encoder's preservation of predictive information also preserves some group-level signal. The non-significant results across the remaining VEIL experiments indicate that the evaluated latents did not provide a reliable MAE improvement over a representation-free baseline for those properties and sample conditions.\\
These findings are consistent with the architectural claims of ICA\@. VEIL latents are not raw records, are not shown to be invertible, and are not shown to support reconstruction of source inputs. The remaining risk demonstrated by property inference is a pipeline-layer linkage risk: if an attacker can collect batches of protected representations and connect them to auxiliary labels or known aggregate compositions, then task-aligned statistical structure may reveal limited population-level facts. Therefore, the appropriate deployment response is operational as well as mathematical. VEIL deployments should keep the encoder registered only to approved data sources and deployed models, deny arbitrary encoding queries, prevent bulk latent extraction, avoid exporting stable join handles with protected vectors, rotate or scope request identifiers, minimize latent retention in logs and vector stores, and evaluate highly task-correlated sensitive properties during model validation. Under those controls, the attacker is left with non-invertible latent vectors but without the auxiliary reference or join structure required to convert them into a reliable property estimator.

\subsection{Membership Inference (Informed)}

Membership inference asks whether an adversary can determine that a particular record was used to train a target model. In the standard supervised-learning formulation, the attacker is given information about a candidate record and the behavior of a trained model, and must decide whether that record was a member of the model's training set. Shokri et al.\ introduced the modern shadow-model formulation of this attack, showing that prediction vectors can carry membership signal \cite{ShokriSSS17}. Yeom et al.\ connected membership advantage to the generalization gap and model overfitting \cite{YeomGFJ18}, while Salem et al.\ showed that useful attacks can persist under weaker assumptions than the original shadow-model setting \cite{SalemZHBFB19}. More recent surveys emphasize that membership inference is a broad family of attacks whose feasibility depends on the target model, the attacker's auxiliary information, and the exposed model interface \cite{10.1145/3523273}. Encoder-targeted attacks, such as EncoderMI, further show that representations themselves can leak membership when an encoder is exposed as a queryable service and the attacker can probe the encoder with augmented inputs \cite{10.1145/3460120.3484749}.\\
In the present experiment, the membership target is the downstream predictor $g_{\phi}$, not the VEIL encoder. This distinction is central to the VEIL threat model. The encoder $f_{\theta}$ remains inside the trusted Source Environment and is not treated as a public oracle. Therefore, asking whether a record was used to train the encoder would require either Source Environment compromise or unauthorized encoder-query access, both of which are outside the downstream-compromise scenario evaluated here. The in-scope question is whether an attacker who observes protected downstream artifacts can determine whether a candidate protected record was used to train $g_{\phi}$. For VEIL, the protected record is the latent vector $R_i=\Psi_i=f_{\theta}\left(\mathbf{x}_i\right)$. For the DP pipeline, $R_i$ denotes the corresponding DP-protected input supplied to the downstream model.\\
Let $M_i\in\{0,1\}$ denote the unknown membership indicator for a candidate record, where $M_i=1$ means that $R_i$ was included in the downstream training set and $M_i=0$ means that it was held out. The attacker's exact goal is to construct a decision rule $A$ such that
\begin{dmath}
\widehat{M}_i=A\left(R_i,g_{\phi}\left(R_i\right),y_i\right)
\end{dmath}
is more accurate than a baseline rule that does not use the protected record or the model's behavior on that record. The attack is label-informed because the attacker is granted the true downstream label or target value $y_i$ for the candidate record. This is a stronger setting than a purely black-box posterior attack, because knowledge of $y_i$ allows the attacker to compute target-conditioned quantities such as loss, true-label confidence where applicable, prediction error, entropy, margin, or analogous output-derived uncertainty features. Those features are then used to train or calibrate a binary member-versus-non-member attack classifier.\\
The simulation applies this label-informed attack to two protected versions of each of the eight model--dataset pipelines considered in this section: one version protected with Differential Privacy and one version protected with VEIL\@. For each pipeline, records used to train the downstream model are treated as members and held-out records are treated as non-members. The attacker is not given raw source records for VEIL inversion, and is not given the encoder as a callable service. Instead, the attacker receives the protected downstream representation, the downstream model output, the true task label or target, and membership-labeled comparison data sufficient to fit the attack classifier. This is a deliberately informed attack setting: it grants more auxiliary information than an attacker who merely intercepts unlabeled inference traffic, but less than an attacker who has stolen an explicit training manifest. If an attacker already possessed authoritative training logs or storage records identifying which protected records were used for training, no inference attack would be necessary; membership would have been directly disclosed by the compromised operational artifact.\\
The baseline for the experiment is the accuracy of a representation-free membership rule that uses only the membership prior. The evaluation challenge is balanced between members and non-members, so the natural baseline is chance-level accuracy, approximately $0.5000$. The ``vs. Baseline'' column in Tables \ref{tab:dp_mbr_label_results} and \ref{tab:veil_mbr_label_results} reports the improvement in membership accuracy over that baseline,
\begin{dmath}
\Delta_{\mathrm{acc}}=\mathrm{Acc}_{\mathrm{attack}}-\mathrm{Acc}_{\mathrm{baseline}}.
\end{dmath}
A positive value indicates that the attack classified membership more accurately than the prior-only baseline; a value near zero indicates no practical membership signal; and a negative value indicates that the learned attack performed worse than the baseline. The reported $p$-value is a permutation-test significance measure for the observed accuracy advantage. Accordingly, evidence of membership leakage requires both a positive accuracy improvement and statistical support that the improvement is not an artifact of the particular member/non-member split.
\begin{center}
\begin{tabular}{lcc}
\toprule
\textbf{Model (DP)} & \textbf{vs. Baseline} & \textbf{$p$-value} \\
\midrule
MNIST & $+0.0132$ & $0.0130$ \\
Ames Housing & $+0.0119$ & $0.0639$ \\
YearPredictionMSD & $+0.0011$ & $0.3750$ \\
Fashion-MNIST & $+0.0071$ & $0.3097$ \\
E2006 & $+0.0608$ & $0.1139$ \\
Home Credit & $+0.0000$ & $0.9890$ \\
CC Default & $+0.0056$ & $0.1229$ \\
CBIS-DDSM & $+0.0307$ & $0.0160$ \\
\bottomrule
\end{tabular}
\captionof{table}{DP Label-Informed Membership Inference Results}
\label{tab:dp_mbr_label_results}
\end{center}
\begin{center}
\begin{tabular}{lcc}
\toprule
\textbf{Model (VEIL)} & \textbf{vs. Baseline} & \textbf{$p$-value} \\
\midrule
MNIST & $+0.0104$ & $0.2478$ \\
Ames Housing & $+0.0034$ & $0.3986$ \\
YearPredictionMSD & $+0.0000$ & $0.9441$ \\
Fashion-MNIST & $+0.0000$ & $0.9901$ \\
E2006 & $+0.0644$ & $0.1089$ \\
Home Credit & $+0.0000$ & $0.9890$ \\
CC Default & $-0.0089$ & $0.8561$ \\
CBIS-DDSM & $+0.0065$ & $0.4266$ \\
\bottomrule
\end{tabular}
\captionof{table}{VEIL Label-Informed Membership Inference Results}
\label{tab:veil_mbr_label_results}
\end{center}
Label-informed membership inference is feasible when the downstream model behaves measurably differently on training records than on held-out records, and when the attacker has enough information to detect that difference. This is most likely when the downstream model overfits, when the training set is small, when candidate records are distinctive, when output confidence or loss values are exposed at high precision, or when the attacker has membership-labeled comparison data, shadow data, or compromised artifacts that approximate the training distribution. In contrast, the attack is constrained when the downstream model generalizes well, when regularization or DP training reduces the influence of individual records, when outputs are coarsened or limited, when the attacker lacks the true label, when the attacker lacks a calibrated member/non-member comparison set, or, in the VEIL setting, when the encoder cannot be queried with arbitrary auxiliary inputs. Thus, failure of this attack does not prove that membership inference is impossible in every deployment; it shows that the exposed protected representations and downstream outputs did not provide a reliable accuracy advantage under the stated label-informed protocol.\\
The DP results in Table \ref{tab:dp_mbr_label_results} show statistically significant positive membership-inference advantage for MNIST, with an accuracy improvement of $+0.0132$ and $p=0.0130$, and for CBIS-DDSM, with an accuracy improvement of $+0.0307$ and $p=0.0160$. Ames Housing shows a positive advantage of $+0.0119$, but its $p$-value of $0.0639$ does not meet the usual $0.05$ threshold. E2006 shows a larger numerical advantage of $+0.0608$, but the corresponding $p$-value of $0.1139$ does not provide statistical support for reliable leakage under the permutation test. YearPredictionMSD, Fashion-MNIST, Home Credit, and CC Default do not show statistically significant positive advantage. These results indicate that the DP-protected versions did not uniformly suppress label-informed membership signal across the evaluated pipelines. This does not contradict the purpose of DP; rather, it reflects the practical privacy--utility tradeoff in which the downstream model may still expose loss or confidence differences that are useful to an informed attacker under a particular privacy parameterization, model class, dataset, and output interface.\\
The VEIL results in Table \ref{tab:veil_mbr_label_results} show no statistically significant positive label-informed membership-inference result across the eight evaluated pipelines. MNIST and Ames Housing show small positive advantages, but their $p$-values are $0.2478$ and $0.3986$, respectively. E2006 again shows the largest numerical advantage, $+0.0644$, but its $p$-value of $0.1089$ does not establish statistically reliable leakage. YearPredictionMSD, Fashion-MNIST, Home Credit, CC Default, and CBIS-DDSM provide no significant positive evidence of attack success. Under this protocol, therefore, the attacker did not obtain a reliable improvement over the balanced prior baseline against any VEIL-protected downstream pipeline.\\
The logical conclusion is comparative and empirical. DP is designed to bound the influence of individual training records, and, when configured appropriately, can provide formal protection against membership disclosure. However, the reported DP simulations still produced significant label-informed attack advantage in two of the eight model-dataset pairs. VEIL uses a different privacy mechanism. It does not claim DP's privacy-budget guarantee; instead, it removes raw inputs from the untrusted downstream environment, confines the encoder to the Source Environment, and exposes only task-aligned non-invertible representations to downstream training and inference. In these experiments, that architectural separation was accompanied by zero statistically significant label-informed membership-inference successes across the eight VEIL pipelines. The appropriate conclusion is not that VEIL makes membership inference mathematically impossible in all circumstances, because an overfit downstream model, excessive output exposure, or compromised training logs can still create membership risk. The supported conclusion is that, for the tested pipelines and the label-informed attack protocol, VEIL reduced the exploitable downstream membership signal below the level needed to beat the balanced baseline in a statistically reliable way, while the DP comparison retained measurable leakage in selected cases.

\subsection{Membership Inference (Black-Box)}

The black-box membership-inference experiment removes the principal additional privilege granted in Section 10.5. The adversary is no longer assumed to know the true downstream label or target value $y_i$ for the candidate record, and is not given model internals, gradients, parameters, training loss traces, or a callable VEIL encoder. The target remains the downstream predictor $g_{\phi}$ rather than the encoder $f_{\theta}$. Thus, for VEIL, the in-scope candidate record remains the protected latent representation $R_i=\Psi_i=f_{\theta}\left(\mathbf{x}_i\right)$; for the DP pipeline, $R_i$ denotes the corresponding DP-protected downstream input. The black-box attacker observes only the protected candidate representation and the downstream model's externally visible response, and attempts to infer whether that candidate was used to train $g_{\phi}$ \cite{ShokriSSS17, SalemZHBFB19, 10.1145/3523273}.\\
Formally, the attacker attempts to learn a decision rule of the form
\begin{dmath}
\widehat{M}_i=A_{\mathrm{bb}}\left(R_i,g_{\phi}\left(R_i\right)\right),
\end{dmath}
where $M_i=1$ indicates downstream-training membership and $M_i=0$ indicates non-membership. In a classification model that exposes posterior probabilities or scores, the black-box attack may use label-agnostic quantities such as maximum confidence, entropy, top-two margin, rank concentration, or the shape of the prediction vector. In a regression model, the attacker may use the predicted value, output magnitude, local smoothness, or distributional features of the protected representation and response. However, the attacker cannot compute the true-label loss, correctness indicator, class-conditioned confidence, residual, or prediction error unless $y_i$ is separately known. This loss of target-conditioned information is the key distinction between the black-box setting here and the label-informed setting in Section 10.5.\\
The intuition remains the same as in membership inference generally: overfit models often behave differently on records they have seen during training than on held-out records, and an attack can succeed if the exposed interface preserves enough of that behavioral difference to separate members from non-members \cite{YeomGFJ18}. The black-box setting is weaker because it must infer this difference indirectly from observable outputs rather than from output correctness or loss. It is also distinct from encoder-membership attacks such as EncoderMI, which assume black-box access to the encoder itself and exploit the stability of embeddings under augmented queries \cite{10.1145/3460120.3484749}. VEIL's evaluated deployment does not expose the encoder as an arbitrary query endpoint, so the experiment below evaluates downstream-model membership leakage through the protected representation and downstream response only.\\
As in the label-informed experiment, the membership challenge is balanced between members and non-members, and the baseline is the accuracy of a representation-free rule that uses only the membership prior. The ``vs. Baseline'' column in Tables \ref{tab:dp_mbr_bb_results} and \ref{tab:veil_mbr_bb_results} therefore reports
\begin{dmath}
\Delta_{\mathrm{acc}}=\mathrm{Acc}_{\mathrm{blackbox}}-\mathrm{Acc}_{\mathrm{baseline}}.
\end{dmath}
A positive value indicates that black-box observations improved membership classification over chance; a value near zero indicates no practical membership signal under this protocol; and a negative value indicates performance below the prior-only baseline. The reported $p$-value is again a permutation-test measure of whether the observed advantage is statistically reliable rather than an artifact of the sampled member/non-member split.\\
\begin{center}
\begin{tabular}{lcc}
\toprule
\textbf{Model (DP)} & \textbf{vs. Baseline} & \textbf{$p$-value} \\
\midrule
MNIST & $+0.0114$ & $0.0749$ \\
Ames Housing & $+0.0102$ & $0.3117$ \\
YearPredictionMSD & $+0.0010$ & $0.1875$ \\
Fashion-MNIST & $+0.0051$ & $0.2168$ \\
E2006 & $+0.0620$ & $0.1259$ \\
Home Credit & $+0.0000$ & $1.0000$ \\
CC Default & $-0.0082$ & $0.9990$ \\
CBIS-DDSM & $-0.0032$ & $0.7129$ \\
\bottomrule
\end{tabular}
\captionof{table}{DP Black-Box Membership Inference Results}
\label{tab:dp_mbr_bb_results}
\end{center}
\begin{center}
\begin{tabular}{lcc}
\toprule
\textbf{Model (VEIL)} & \textbf{vs. Baseline} & \textbf{$p$-value} \\
\midrule
MNIST & $-0.0074$ & $0.5295$ \\
Ames Housing & $-0.0051$ & $0.7463$ \\
YearPredictionMSD & $+0.0007$ & $0.3125$ \\
Fashion-MNIST & $+0.0021$ & $0.3826$ \\
E2006 & $+0.0550$ & $0.1149$ \\
Home Credit & $+0.0000$ & $1.0000$ \\
CC Default & $-0.0053$ & $0.7992$ \\
CBIS-DDSM & $-0.0032$ & $0.7129$ \\
\bottomrule
\end{tabular}
\captionof{table}{VEIL Black-Box Membership Inference Results}
\label{tab:veil_mbr_bb_results}
\end{center}
\noindent The DP results in Table \ref{tab:dp_mbr_bb_results} show no statistically significant positive black-box membership-inference advantage at the usual $0.05$ threshold. MNIST has a small positive advantage of $+0.0114$, but its $p$-value is $0.0749$, and Ames Housing, YearPredictionMSD, Fashion-MNIST, E2006, Home Credit, CC Default, and CBIS-DDSM likewise fail to provide statistically reliable evidence of black-box leakage. E2006 shows the largest numerical advantage, $+0.0620$, but its $p$-value of $0.1259$ does not support a reliable attack under the permutation test. The logical conclusion is that the DP-protected pipelines did not expose enough label-agnostic output signal for this black-box attack to beat the balanced baseline in a statistically meaningful way. This is a materially weaker result for the attacker than the label-informed DP experiment, where MNIST and CBIS-DDSM showed significant positive advantages. The comparison indicates that, for the DP pipelines, the true label or target value supplied in the informed setting contributed important membership signal that was not recoverable from black-box outputs alone.\\
The VEIL results in Table \ref{tab:veil_mbr_bb_results} also show no statistically significant positive black-box membership-inference advantage across the eight evaluated pipelines. MNIST, Ames Housing, CC Default, and CBIS-DDSM are negative relative to the baseline; YearPredictionMSD and Fashion-MNIST are only slightly positive; Home Credit is exactly at baseline; and E2006 again has the largest numerical positive advantage, $+0.0550$, with a non-significant $p$-value of $0.1149$. Under this protocol, therefore, the black-box attacker did not obtain reliable evidence that the protected VEIL representation and downstream response distinguished training members from held-out records.\\
Taken together, the black-box results support a narrower and more conservative conclusion than a universal impossibility claim. DP has a formal privacy objective: when DP training is configured with an appropriate privacy budget and implementation, it bounds the influence of individual records and thereby mitigates membership disclosure risk \cite{abadi2016deep}. In these black-box experiments, the DP pipelines achieved the desired empirical result: no statistically significant black-box membership leakage was observed. VEIL achieves protection through a different mechanism. It confines raw inputs and the encoder to the Source Environment, exports only task-aligned non-invertible representations, and withholds the encoder as a public oracle. In these black-box experiments, that architectural separation also produced zero statistically significant membership-inference successes. Thus, the DP and VEIL results are empirically similar in the black-box setting, while the label-informed setting separates them more clearly: DP showed significant leakage in two evaluated pipelines, whereas VEIL did not. The practical interpretation is that black-box membership inference is a lower-power attack than label-informed membership inference for these pipelines, and that VEIL's main comparative advantage appears when the attacker has more auxiliary information but still lacks raw records, encoder-query access, and authoritative training logs. If those operational artifacts are compromised, membership may be disclosed directly without inference; if the encoder is exposed as an arbitrary query service, encoder-specific membership attacks become a separate threat model. Under the evaluated black-box downstream-compromise model, however, both DP and VEIL reduce exploitable membership signal to the level of statistically insignificant noise, with VEIL preserving the same no-success pattern observed in the stronger label-informed test.

\subsection{Prediction Leakage}

Prediction leakage is an output-layer privacy attack. It does not attempt to invert an encoded representation, reconstruct the original input record, or determine whether a record was used in training. Instead, the attacker observes the downstream model's prediction and uses the semantic or statistical relationship between that prediction and a sensitive attribute to infer information about the individual record. This risk is closely related to the model-inversion literature, where model outputs and auxiliary information can reveal sensitive facts that were not intended to be disclosed, and to output-based membership-inference work showing that prediction behavior itself can carry privacy signal \cite{Fredrikson2014Warfarin, FredriksonJR15, YeomGFJ18}.\\
To avoid overloading notation, let $O_i$ denote the observable output returned by the downstream predictor for record $i$. In the VEIL pipeline, $O_i=g_{\phi}\left(\Psi_i\right)$ where $\Psi_i=f_{\theta}\left(\mathbf{x}_i\right)$ remains the non-invertible latent vector. In the DP pipeline, $O_i$ is the output of the corresponding downstream model operating on the DP-protected input used in these experiments. For classification, $O_i$ may be a class label, a probability vector, a confidence score, or logits. For regression, $O_i$ may be a scalar estimate, a vector-valued prediction, a risk score, or a discretized decision band. Let $S_i$ denote a sensitive attribute that the system operator did not intend to expose through the prediction interface. The prediction-leakage attacker seeks a rule $b_{\eta}:O_i\mapsto \widehat{S}_i$ that predicts $S_i$ from the observed output alone, possibly after using auxiliary labeled examples to estimate the relationship between $O$ and $S$.\\
The threat model is therefore deliberately weaker than the reconstruction and attribute-inference threat models considered above. The attacker need not know the VEIL architecture, the encoder, the encoder parameters, the latent dimensionality, the downstream model weights, the training set, or the mathematical non-invertibility arguments developed in Section 9. The attacker can treat the system as an ordinary prediction service or as an ordinary downstream log source. What the attacker must have is access to individual-level predictions and a way to link those predictions to records or persons. The linkage may be an explicit identifier, a query identifier, a customer identifier, a timestamped request trace, a row order in an exported file, or any other operational join key. Without such linkage, the same observations may support only aggregate or population-level inference, not individual attribute assignment.\\
Setting up the attack requires four ingredients. First, the attacker needs observable predictions $O_i$ from API responses, prediction logs, analytics exports, downstream databases, message queues, monitoring traces, or compromised application components. Second, the attacker needs identifiers or quasi-identifiers that attach each prediction to a record of interest. Third, the attacker needs a candidate sensitive attribute $S$ whose relationship to the model output is known, suspected, or learnable. This relationship may come from domain knowledge, from public statistics, from the meaning of the task itself, or from a labeled reference set containing pairs $\left(O_i,S_i\right)$. Fourth, the attacker needs a decision procedure for converting outputs into sensitive-attribute guesses. In the simplest case, the procedure is just a threshold or lookup rule; for example, a sufficiently high medical-cost prediction may indicate chronic disease burden, and a predicted income value may directly place a person into an income bracket. In more complex cases, the attacker trains a lightweight classifier or regressor from auxiliary pairs $\left(O_i,S_i\right)$ and evaluates whether it improves over a prior-only or imputation baseline.\\
The attack applies only when the output is both visible and informative. Visibility is a deployment condition: individual predictions must leave the trust boundary, be returned to an unauthorized party, be retained in logs with identifiers, or be reachable through downstream compromise. Informativeness is a statistical condition: the prediction must have nonzero exploitable dependence on the sensitive attribute. In information-theoretic terms, prediction leakage requires positive mutual information between the exposed output and the sensitive attribute,
\begin{dmath}
I\left(O;S\right)>0.
\end{dmath}
This condition is automatically satisfied when the model directly predicts the sensitive attribute or a deterministic transformation of it. It may also hold when the task output is a strong proxy for the sensitive attribute, such as a credit score that reflects income stability, an insurance premium that reflects health risk, a readmission score that reflects diagnosis burden, or a property valuation that reflects neighborhood wealth. It does not follow merely from high predictive utility. A useful model for the intended task need not leak a particular sensitive attribute if that attribute is independent of the output, if the model suppresses the relevant signal, or if the output is too coarse to support reliable inference.\\
This distinguishes prediction leakage from attribute inference. Attribute inference, as evaluated earlier in this section, gives the attacker protected representations $R_i$ and requires labeled pairs $\left(R_i,S_i\right)$ so that a supervised attack model can learn $R\mapsto S$. It asks whether the exported representation contains incremental sensitive-attribute signal beyond an appropriate baseline, a distinction that is important because some attributes can be imputed from population structure even without access to the target system \cite{JayaramanE22}. Prediction leakage instead gives the attacker $O_i$, the downstream prediction. It is usually simpler because the output is already semantically aligned with the task and may be directly interpretable. It is also narrower because it succeeds only for attributes that are revealed by, or strongly correlated with, that output. Thus, attribute inference is a representation-layer attack, while prediction leakage is an output-layer attack. A successful prediction-leakage result does not imply that $R_i$ or $\Psi_i$ is invertible, nor does it show that the attacker can reconstruct $\mathbf{x}_i$. It shows that the chosen output communicates sensitive information.\\
Inference is most feasible when predictions are individualized, stable, high-resolution, and retained with joinable identifiers. Full probability vectors, confidence scores, logits, calibrated risks, and continuous regression outputs generally provide a richer attack surface than coarse labels or broad bands because they expose more information about the model's assessment. This observation is consistent with prior model-inversion work showing that confidence information can increase leakage relative to less informative output interfaces \cite{FredriksonJR15}. The attack is also stronger when the attacker has labeled calibration data, when $S$ is highly correlated with the target task, when the sensitive attribute is a causal or proxy feature for the predicted outcome, or when the prediction itself is the sensitive value. In such settings, very little machine learning may be required; the attack may reduce to reading the output.\\
Inference is constrained when any of the required ingredients is missing. If predictions never leave the trusted environment, there is no output available to the attacker under this threat model. If predictions are aggregated before export, the attacker may learn a population statistic but cannot reliably assign attributes to individuals. If identifiers are removed, rotated, or separated from outputs, the attacker may observe sensitive-looking predictions without being able to link them to records. If $O$ and $S$ are weakly related, conditionally independent given the task, or deliberately decorrelated by model design, the best attack may collapse to the prior baseline. If only coarse decision bands are exposed, the attack may still be possible but is usually information-limited compared with exposing raw scores. Finally, if the output is visible only to an authorized subject for whom the prediction is intended, the residual issue is access governance rather than unintended third-party inference.\\
The scope of what leaks is therefore precise. Prediction leakage can reveal the downstream prediction itself, a sensitive target if the model was built to predict one, or a correlated sensitive attribute when the output acts as a proxy. It can also produce labeled examples that enable a chain attack: if the same adversary also obtains protected representations for overlapping records, then predictions may be used to assign noisy labels $\widehat{S}$ to those records, after which the adversary can train an attribute-inference model on pairs $\left(R_i,\widehat{S}_i\right)$ and apply it to other records for which only $R$ is available. What prediction leakage does not reveal is the full raw input, the encoder, the complete latent representation, the training set, or unrelated attributes for which the output carries no signal. It is not a reconstruction attack and does not contradict the topological or information-theoretic non-invertibility of VEIL latents.\\
The mitigations are primarily operational and interface-level. Individual predictions should be treated as potentially sensitive artifacts whenever the output is itself sensitive or is correlated with sensitive attributes. Systems should avoid storing raw predictions with durable identifiers unless there is a clear need, should separate prediction logs from identity data, should minimize retention, should enforce access controls on prediction stores and monitoring traces, and should avoid exporting individual-level predictions when aggregate reporting is sufficient. When individual responses must be returned, the interface should disclose the least informative output compatible with the business and scientific purpose: for example, returning a decision band rather than a full confidence vector, rounding regression scores, suppressing logits, or limiting secondary diagnostic explanations. These reductions do not eliminate leakage when the disclosed decision is itself sensitive, but they reduce unnecessary side-channel information.
\begin{center}
\begin{tabular}{lcc}
\toprule
\textbf{Model (DP)} & \textbf{vs. Baseline} & \textbf{$p$-value} \\
\midrule
MNIST & $+0.4496$ & $0.0010$ \\
Ames Housing & $+0.2300$ & $0.0010$ \\
YearPredictionMSD & $+0.1747$ & $0.0010$ \\
Fashion-MNIST & $+0.3659$ & $0.0010$ \\
E2006 & $+0.4494$ & $0.0010$ \\
Home Credit & $-0.0001$ & $1.0000$ \\
CC Default & $+0.0001$ & $1.0000$ \\
CBIS-DDSM & $+0.0088$ & $0.1215$ \\
\bottomrule
\end{tabular}
\captionof{table}{DP Prediction Leakage Results}
\label{tab:dp_pred_leak_results}
\end{center}
\begin{center}
\begin{tabular}{lcc}
\toprule
\textbf{Model (VEIL)} & \textbf{vs. Baseline} & \textbf{$p$-value} \\
\midrule
MNIST & $+0.4314$ & $0.0010$ \\
Ames Housing & $+0.0493$ & $0.5545$ \\
YearPredictionMSD & $+0.1387$ & $0.0099$ \\
Fashion-MNIST & $+0.2344$ & $0.0199$ \\
E2006 & $+0.4553$ & $0.0010$ \\
Home Credit & $-0.0051$ & $1.0000$ \\
CC Default & $+0.0000$ & $1.0000$ \\
CBIS-DDSM & $+0.0005$ & $0.4356$ \\
\bottomrule
\end{tabular}
\captionof{table}{VEIL Prediction Leakage Results}
\label{tab:veil_pred_leak_results}
\end{center}
Additional mitigations can be applied during validation and deployment. Before release, the operator should audit correlations and attack advantages between candidate outputs and sensitive attributes, using baselines that account for class imbalance and ordinary imputation. For binary or continuous sensitive attributes, this may include correlation and threshold attacks; for multiclass attributes, it should be evaluated as classification advantage over the relevant baseline rather than as a single Pearson correlation. If output leakage is unacceptable, the operator can redesign the task, coarsen the output, aggregate results, apply access controls, or add calibrated output noise. Differentially private output mechanisms may reduce leakage, but they reintroduce the privacy--utility trade-off that DP-based approaches must manage through privacy parameters and noise calibration \cite{abadi2016deep}. In VEIL deployments specifically, the encoder protects the path from raw input $\mathbf{x}$ to latent representation $\Psi$; it does not by itself make the downstream prediction $O$ nonsensitive. Prediction outputs, logs, and identifiers must therefore be governed as separate protected assets.\\
The DP results in Table \ref{tab:dp_pred_leak_results} show statistically significant prediction leakage for MNIST, Ames Housing, YearPredictionMSD, Fashion-MNIST, and E2006, each with a positive advantage and $p=0.0010$. Home Credit and CC Default are essentially at baseline, while CBIS-DDSM shows only a small, non-significant positive advantage. These results indicate that, for several evaluated pipelines, the downstream output itself exposed enough information about the selected sensitive attribute to support inference over the baseline. The result should not be interpreted as a reconstruction result: it does not show that the protected inputs were recoverable, only that the released predictions were informative about the tested attribute.\\
The VEIL results in Table \ref{tab:veil_pred_leak_results} show the same output-layer risk pattern, but with a different distribution across datasets. MNIST, YearPredictionMSD, Fashion-MNIST, and E2006 show statistically significant positive prediction-leakage advantage. Ames Housing has a positive numerical advantage but is not statistically significant, and Home Credit, CC Default, and CBIS-DDSM remain at or near baseline. Thus, VEIL substantially constrains reconstruction and representation-layer leakage in the earlier experiments, but it does not make downstream predictions nonsensitive when those predictions themselves encode, proxy, or reveal the selected attribute. This is the expected boundary of the architecture: VEIL protects raw inputs before they leave the Source Environment, while prediction leakage concerns what the deployed predictor is permitted to disclose after inference.\\
The practical conclusion is that prediction leakage must be managed independently of encoder non-invertibility. In both the DP and VEIL pipelines, significant results occur when the observable prediction contains task-aligned information that is also informative about the sensitive attribute. Conversely, the non-significant results show that prediction visibility alone is not sufficient; the exposed output must also carry usable sensitive-attribute signal. For VEIL deployments, the correct mitigation is therefore not to weaken the encoder or abandon the latent representation, but to govern prediction outputs as sensitive operational artifacts: restrict access, reduce unnecessary output granularity, separate logs from identifiers, aggregate where possible, and audit sensitive-output correlations before deployment.

\subsection{Model Extraction}

Model extraction, also called model stealing, is a black-box attack in which an adversary treats a deployed prediction interface as an oracle and uses its input-output behavior to train a local surrogate model \cite{tramèr2016stealingmachinelearningmodels, orekondy2019knockoffnets}. The attack does not require direct access to the target model's weights, architecture, training set, gradients, or loss function. Instead, the attacker needs the ability to submit valid inputs to the deployed model and observe the corresponding outputs. In the notation used throughout this paper, let $g_{\phi}:\mathcal{Z}\rightarrow\mathcal{Y}$ denote the deployed downstream predictor, where $\mathcal{Z}$ is the space of inputs accepted by the prediction API and $\mathcal{Y}$ is the output space. In a conventional or DP-protected pipeline, $\mathcal{Z}$ may be the raw feature space or a perturbed version of it. In the VEIL architecture, $\mathcal{Z}$ is the latent space produced by the protected encoder, so the exposed predictor maps latent vectors $\Psi$ to predictions $O$.\\
The model extraction threat model is therefore distinct from reconstruction, inversion, membership inference, and direct artifact theft. The adversary is assumed to have query access to the deployed model, but not privileged access to the training environment or model file. To set up the attack, the adversary must have five operational ingredients: a queryable endpoint for the target predictor; a way to construct syntactically valid inputs in the endpoint's input space; a query budget large enough to cover the complexity of the target mapping; observable outputs such as labels, scores, probabilities, logits, or regression values; and local compute for fitting and validating a surrogate model. The attacker does not need the surrogate architecture to match the target architecture. A linear model, tree ensemble, neural network, or other learner may be used, provided that it approximates the target input-output function with sufficient fidelity.\\
The attack applies whenever the deployed model can be queried repeatedly and the responses are specific enough to support supervised learning. It is especially relevant for prediction APIs, embedded models that can be instrumented, multi-tenant inference services, and compromised inference logs containing input-output pairs. It does not describe the case where an attacker simply steals the serialized model artifact; that is direct model theft rather than extraction. It is also substantially weakened when the attacker cannot construct valid inputs, receives only highly aggregated outputs, is constrained by strict rate limits, or is subject to monitoring that detects abnormal query patterns. Richer outputs reduce the cost of extraction: confidence scores, class probabilities, logits, and real-valued regression predictions carry more information per query than hard class labels, although label-only extraction can still be possible at higher query budgets \cite{tramèr2016stealingmachinelearningmodels, jagielski2020highaccuracyhighfidelity}.\\
The basic extraction procedure is straightforward. First, the attacker collects or synthesizes query inputs. These may come from public examples, the attacker's own legitimate queries, synthetic samples drawn from an estimated input distribution, active-learning query construction, or compromised logs. In a VEIL deployment targeting the downstream predictor, these inputs must be latent vectors $\Psi$ rather than raw records $\mathbf{x}$. Second, the attacker submits a finite query set $\mathcal{Q}=\left\{\Psi_i\right\}_{i=1}^B$ to the target model and records the outputs $O_i=g_{\phi}\left(\Psi_i\right)$. Third, the attacker trains a surrogate $\hat{g}_{\omega}$ on the extracted dataset $\left\{\left(\Psi_i,O_i\right)\right\}_{i=1}^B$ by minimizing disagreement with the target outputs. Finally, the attacker evaluates the surrogate against held-out target queries to estimate fidelity. More sophisticated versions adaptively choose new queries based on the current surrogate's uncertainty or disagreement, connecting model extraction to active learning and improving query efficiency when the attacker's query budget is limited \cite{chandrasekaran2019exploringconnectionsactivelearning}.\\
Fidelity is the central evaluation concept. It measures agreement with the target model, not necessarily agreement with the ground-truth label or response. Thus, a surrogate can have high fidelity even when the target model itself is inaccurate, and a surrogate can have useful accuracy while failing to reproduce the target's exact decision boundary \cite{jagielski2020highaccuracyhighfidelity}. In the experiments reported here, the query budget was fixed at $B=500$, and fidelity was measured using $R^2$ between the target model outputs and surrogate outputs on evaluation queries,
\begin{equation}
R^2 = 1 - \frac{\sum_i \norm{O_i - \hat{O}_i}_2^2}{\sum_i \norm{O_i - \bar{O}}_2^2},
\end{equation}
where $\hat{O}_i=\hat{g}_{\omega}\left(\Psi_i\right)$ and $\bar{O}$ is the empirical mean of the target outputs on the evaluation set. A value near $1$ indicates that the surrogate closely reproduces the target's numerical outputs under the evaluated query distribution, while a low or negative value indicates that the query budget, surrogate class, sampled input distribution, or output visibility was insufficient to reproduce the target behavior. The query budget is therefore inseparable from the reported fidelity. Simple, low-dimensional, smooth, or nearly linear target functions may be extracted with hundreds of queries; complex, high-dimensional, discontinuous, or highly nonlinear target functions generally require many more queries, better query selection, or richer output information.\\
Model extraction matters for two reasons. First, it is a form of intellectual-property theft: the target model embodies training data acquisition, feature engineering, architecture selection, hyperparameter tuning, compute, validation, compliance work, and operational deployment. A high-fidelity copy can allow an attacker or competitor to reproduce valuable behavior without paying those costs. Second, it is an attack multiplier. Once a surrogate has been trained, the adversary can run unlimited offline experiments without rate limits, access controls, billing, logging, anomaly detection, or account termination. The copied model can then be used to generate predictions for any inputs the attacker possesses, probe decision boundaries, prepare adversarial examples, or support follow-on privacy attacks such as membership inference or inversion under a much cheaper local threat model \cite{papernot2017practicalblackbox, 10.1145/3624010}. Extraction does not create new access to protected inputs; it only removes the operational constraints that would otherwise limit repeated use of the prediction function.\\
The results in Table \ref{tab:mod_ex_results} show that, under a fixed query budget of $500$, many of the evaluated downstream predictors were extracted with high numerical fidelity. This is expected whenever the deployed model is simple relative to the information revealed by each query. Several DP and VEIL results are close to $R^2=1$, indicating that the surrogate nearly reproduced the prediction API's output surface on the evaluated distribution. Other cases, such as MNIST, Fashion-MNIST, and especially the VEIL version of CBIS-DDSM, show substantially lower fidelity, indicating that $500$ queries were not sufficient to learn the exposed target behavior under the reported attack protocol. These values should be interpreted as extraction fidelity for the deployed predictor only, not as evidence that the attacker recovered the training data, reconstructed raw inputs, or learned the internal representation used by VEIL.
\begin{center}
\begin{tabular}{lcc}
\toprule
\textbf{Model} & \textbf{DP $R^2$} & \textbf{VEIL $R^2$} \\
\midrule
MNIST & $0.6689$ & $0.3843$ \\
Ames Housing & $0.9871$ & $0.9466$ \\
YearPredictionMSD & $1.0000$ & $1.0000$ \\
Fashion-MNIST & $0.5596$ & $0.4514$ \\
E2006 & $0.9999$ & $1.0000$ \\
Home Credit & $1.0000$ & $0.6937$ \\
CC Default & $1.0000$ & $0.9969$ \\
CBIS-DDSM & $0.9833$ & $0.1660$ \\
\bottomrule
\end{tabular}
\captionof{table}{Model Extraction Results}
\label{tab:mod_ex_results}
\end{center}
The VEIL-specific interpretation is particularly important. In VEIL, the end-to-end predictive function is the composite mapping $h\left(\mathbf{x}\right)=g_{\phi}\left(f_{\theta}\left(\mathbf{x}\right)\right)$, where $f_{\theta}$ is the encoder protected inside the trusted Source Environment and $g_{\phi}$ is the downstream predictor exposed in the less-trusted inference environment. Because $f_{\theta}$ is trained to produce task-aligned latent representations, much of the useful prediction logic can be transferred into the encoder. The deployed prediction API can therefore be intentionally simple: in many cases, the downstream model looks like a shallow or nearly linear model operating on a very low-dimensional latent vector, often with only two latent input variables in the experiments described in this paper.\\
This architectural fact changes what extraction means. If an attacker extracts $g_{\phi}$ with high $R^2$, the attacker has learned a map from latent coordinates to outputs. The attacker has not learned what those latent coordinates mean in terms of the original feature space, has not learned the protected encoder $f_{\theta}$, and has not inverted $\Psi$ back to $\mathbf{x}$. In many of the high-fidelity VEIL extraction cases, the attack is therefore almost trivial precisely because the exposed downstream model was deliberately made simple. The extracted surrogate may accurately reproduce a two-variable decision surface in latent space, but those two variables remain non-invertible, task-aligned encodings generated inside the Source Environment. From the adversary's perspective, the surrogate says how to transform unknown latent coordinates into predictions; it does not reveal how income, payment history, image pixels, clinical measurements, or any other raw source features were transformed into those coordinates.\\
For this reason, the table can be misleading if read as a measure of overall VEIL compromise. High downstream extraction fidelity does not imply that the attacker has uncovered the full prediction logic of the protected system. Under VEIL, the economically and privacy-relevant logic is split across the protected encoder and the exposed predictor. The prediction API may contain only the final, simplified head of the composite model, while the sensitive representation-learning function remains inside the Source Environment. Thus, model extraction remains a legitimate operational risk for the exposed downstream predictor and can enable offline misuse of predictions for latent vectors the attacker already has, but it does not defeat ICA's core protection: raw inputs and the encoder mapping from $\mathbf{x}$ to $\Psi$ are not exported to the untrusted inference environment.\\
Operational mitigations are still necessary. VEIL deployments should authenticate prediction access, rate-limit and quota queries, monitor input distributions for extraction-like behavior, restrict unnecessary output precision, avoid exposing logits or full probability vectors unless required, watermark or fingerprint deployed predictors when appropriate, and prevent bulk logging or reuse of latent vectors. Most importantly, the encoder endpoint must not be exposed outside the Source Environment. If the attacker can query $f_{\theta}$ directly by submitting raw inputs and observing latent vectors, the scenario becomes encoder extraction rather than ordinary model extraction, and the trust-boundary assumption of VEIL has been violated.

\subsection{Summary of Privacy Attack Simulation Results}

The simulations in Sections 10.1--10.8 evaluate privacy attacks at several distinct layers of the protected ML pipeline: reconstruction from protected inputs, inference of individual sensitive attributes, inference of aggregate population properties, membership inference under both label-informed and black-box assumptions, leakage through downstream predictions, and extraction of the exposed downstream predictor. Across these experiments, the VEIL architecture exhibits a consistent pattern. The latent vectors exported by VEIL preserve predictive utility, but they do not behave like recoverable substitutes for the original records. The strongest evidence for this conclusion comes from the reconstruction experiments: the DP pipelines produced statistically significant positive decoder advantage in four of the eight evaluated model--dataset pairs, whereas the VEIL pipelines produced no statistically significant positive reconstruction advantage in any of the eight cases. This empirical result is consistent with the non-invertibility guarantees developed in Section 9 and with the architectural premise that raw inputs and the encoder remain confined to the Source Environment.\\
The inference results refine this conclusion. VEIL does not eliminate every possible statistical relationship between protected artifacts and sensitive attributes, because the encoder is intentionally trained to preserve task-relevant information. When a sensitive attribute is itself correlated with task-relevant structure, some residual inference risk can remain. However, the observed VEIL leakage is narrower than the corresponding DP leakage in the evaluated protocols. Attribute inference yielded statistically significant positive advantage in two VEIL cases, MNIST and Fashion-MNIST, compared with five DP cases. Property inference yielded one significant VEIL result, MNIST, compared with two DP results. Label-informed membership inference yielded zero significant VEIL results, while DP yielded two. Black-box membership inference yielded zero significant results for both protection methods. Thus, the principal residual VEIL risks are not raw-record reconstruction or encoder inversion, but linkage-based inference against task-correlated attributes or populations when protected artifacts can be joined to auxiliary information.\\

\begin{center}
\setlength{\tabcolsep}{3pt}
\begin{tabular}{lcc}
\toprule
\textbf{Attack} & \textbf{DP Sig.} & \textbf{VEIL Sig.} \\
\midrule
Reconstruction & $4/8$ & $0/8$ \\
Attribute Inference & $5/8$ & $2/8$ \\
Property Inference & $2/8$ & $1/8$ \\
Membership (Informed) & $2/8$ & $0/8$ \\
Membership (Blackbox) & $0/8$ & $0/8$ \\
Prediction Leakage & $5/8$ & $4/8$ \\
\bottomrule
\end{tabular}
\captionof{table}{Summary of Statistically Significant Privacy Attack Results}
\label{tab:privacy_attack_summary}
\end{center}

\noindent Table \ref{tab:privacy_attack_summary} summarizes the number of statistically significant positive attack outcomes reported in Tables \ref{tab:dp_recon_results}--\ref{tab:veil_pred_leak_results}. The table should be interpreted carefully. It counts statistically significant attack advantage relative to the relevant baseline for each attack class; it does not claim that the same amount or severity of information leaked in every positive case. Reconstruction, attribute inference, property inference, membership inference, and prediction leakage measure different adversarial objectives. In particular, prediction leakage is an output-layer phenomenon: if the deployed model is allowed to return an output that is itself sensitive, or is strongly correlated with a sensitive attribute, then either a DP or VEIL pipeline can leak that attribute through the prediction interface. The VEIL encoder prevents raw input export and reconstruction; it does not make downstream predictions nonsensitive once those predictions are intentionally disclosed.\\
The model extraction experiment requires a separate interpretation because it measures surrogate fidelity rather than statistical leakage over a baseline. Under a fixed budget of $500$ queries, several downstream predictors were extracted with high $R^2$ in both DP and VEIL settings. For VEIL, however, high extraction fidelity means that the attacker has approximated $g_{\phi}:\mathcal{Z}\rightarrow\mathcal{Y}$, the exposed mapping from latent coordinates to outputs. It does not mean that the attacker has recovered $f_{\theta}$, inverted $\Psi$ to obtain $\mathbf{x}$, learned the meaning of the latent coordinates in the original feature space, or obtained raw source records. Model extraction remains an operational and intellectual-property risk for the exposed downstream predictor, especially when outputs are high precision and query access is weakly controlled, but it does not defeat ICA's core privacy claim unless the encoder itself is exposed or the Source Environment trust boundary is violated.\\
Taken together, the simulation results support the following practical conclusion. VEIL provides strong empirical resistance to reconstruction and downstream membership inference while maintaining predictive utility, and it reduces several forms of representation-layer inference relative to the DP baselines evaluated here. The remaining risks are concentrated at the pipeline and interface layers: joinable latent artifacts can support attribute or property inference when auxiliary labels are available; individual predictions can leak sensitive information when the output is itself informative; and exposed downstream predictors can be copied if query access is uncontrolled. These risks are not contradictions of non-invertibility. They identify the operational controls required for a complete deployment: keep the encoder inside the Source Environment, prevent arbitrary encoder queries, avoid bulk export or durable storage of latent vectors with join handles, minimize and govern prediction logs, limit unnecessary output precision, rate-limit and monitor prediction APIs, and explicitly test task-correlated sensitive attributes and properties during model validation. Under these controls, the attacker is left with non-invertible, task-aligned representations and constrained downstream outputs rather than raw, reconstructable records.

\section{Model Explainability and Counterfactual Analysis with the VEIL Architecture}

Explainability in the VEIL architecture is naturally performed inside the trusted Source Environment, as the raw input, $\mathbf{x}$, and the returned prediction, $\hat{\mathbf{y}}$, are co-located there even though the deployed ML model consumes only latent vectors, $\Psi$. This aspect of the VEIL architecture was mentioned in Section 8.3, and is shown in Figures \ref{fig:VEIL_inf_env} and \ref{fig:VEIL_architecture_full}.\\
Let $f_{\theta}$ be an encoder with properties described in Section 9 to ensure non-invertibility, and let $g_{\phi}$ be the downstream predictor. Define the composite predictor as follows.
\begin{dmath}
h\left(\mathbf{x}\right)=g_{\phi}\left(f_{\theta}\left(\mathbf{x}\right)\right)
\end{dmath}
Any explainer that requires only query access to $h$ can then be applied to the composite predictor inside the trusted Source Environment, thus enabling model explanations and counterfactual analysis without exposing raw data outside the trust boundary defined in Section 8. In particular, SHapley Additive exPlanations (SHAP) \cite{LundbergLee2017Arxiv} may attribute the output of the composite predictor directly to the original input features or pixels, even though the downstream model itself is trained only on latent representations, and therefore makes predictions only on latent representations. The following example makes this workflow concrete by applying SHAP to the VEIL-protected Home Credit Default Risk pipeline, where explanations are expressed in the original feature space while prediction still proceeds through the protected latent representation.

\subsection{SHAP Analysis of the Home Credit Default Risk Model}

The Home Credit Default Risk experiment provides an especially useful test case for VEIL because it joins the two central requirements of privacy-preserving supervised ML: privacy protection and predictive utility. As shown in Section 10, the Home Credit pipeline protected under VEIL withstood every simulated privacy attack under the reported protocols, while Table \ref{tab:home_credit_results} shows that the VEIL-protected SCRAE pipeline achieved a test ROC-AUC of $0.7734$, outperforming both the Raw Data baseline at $0.7582$ and the DP baseline at $0.7426$. Thus, for this dataset, VEIL did not merely preserve the usefulness of the model while protecting sensitive inputs; it improved predictive performance relative to both unprotected and differentially private alternatives.\\
In a real ML deployment, a model of this kind would naturally be used for credit decisioning. In the United States, adverse credit decisions implicate notice and explanation requirements under a combination of the Fair Credit Reporting Act (FCRA), the Equal Credit Opportunity Act (ECOA), and Regulation B. The FCRA requires notice when adverse action is taken based in whole or in part on information in a consumer report, including score-related disclosures where applicable \cite{uscode15_1681m}. ECOA and Regulation B require the creditor to provide, or make available, a statement of specific reasons for the adverse action \cite{uscode15_1691, cfpb_regulation_b_1002_9, cfpb_regulation_b_interpretation_1002_9}. The Consumer Financial Protection Bureau (CFPB) has further emphasized that these obligations apply when creditors use complex algorithms or AI/ML systems, and that generic or inaccurate reason codes are insufficient \cite{cfpb_circular_2022_03, cfpb_circular_2023_03}. In this setting, counterfactual analysis is not merely an optional interpretability feature. It is an operational mechanism for generating and validating applicant-facing explanations, because the organization must be able to identify which applicant characteristics materially influenced the decision and determine whether plausible changes to those characteristics would have altered the outcome.\\
Figure \ref{fig:home_cred_shap_global_importance} reports the global SHAP feature importance for the VEIL-protected Home Credit model by ranking raw input features according to their average contribution to the composite predictor $h\left(\mathbf{x}\right)=g_{\phi}\left(f_{\theta}\left(\mathbf{x}\right)\right)$. The largest contributions are concentrated in a relatively small set of credit-risk variables, led by the external source score features and followed by loan-structure, payment-burden, and applicant-history variables. This indicates that the protected model's decisions are being driven by interpretable credit-risk signals rather than by an unexplainable latent coordinate system.
\begin{center}
\includegraphics[scale=0.36]{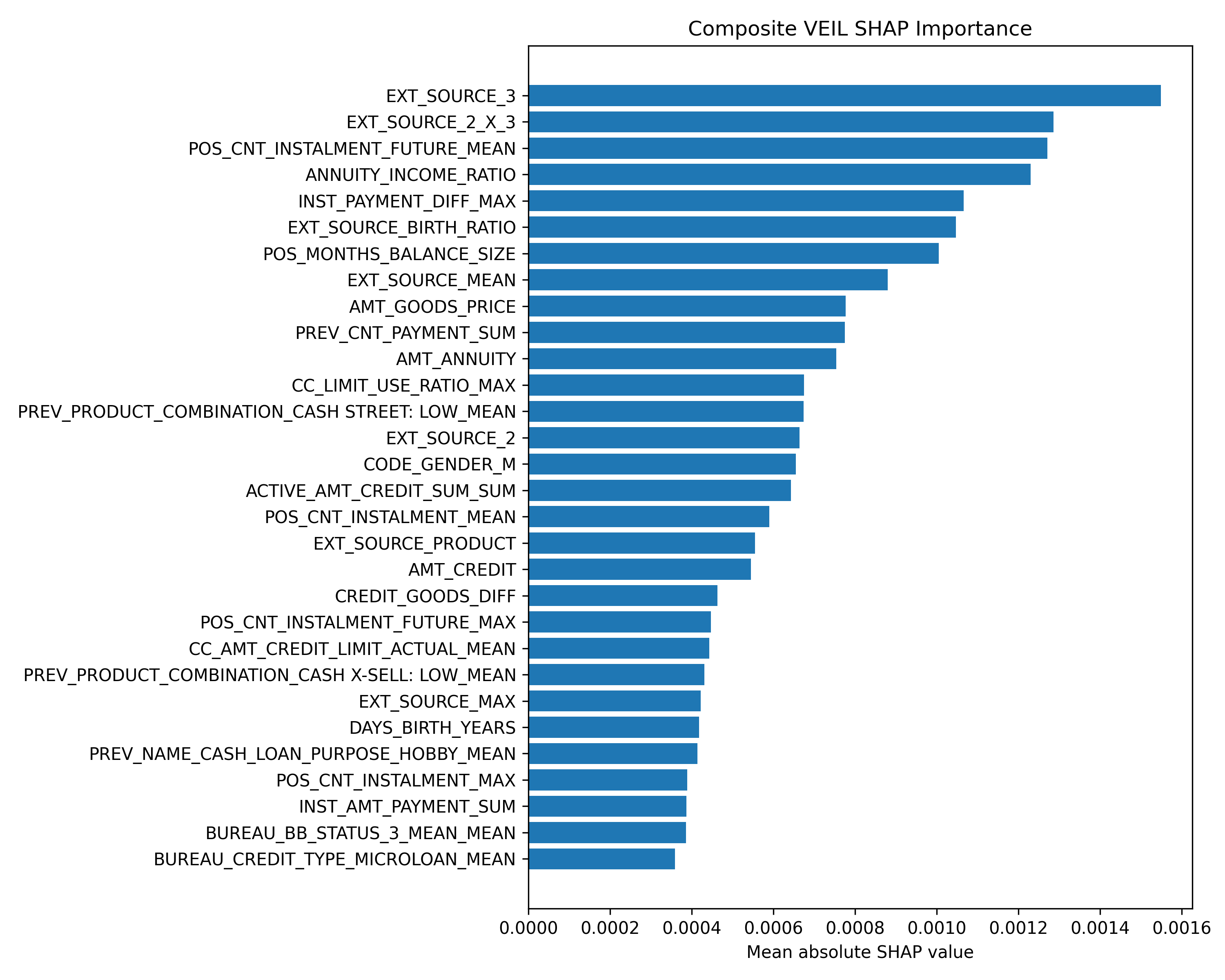}
\captionof{figure}{Global Feature Importance for Home Credit Default Risk}
\label{fig:home_cred_shap_global_importance}
\end{center}
This global view establishes which original applicant attributes carry the most explanatory weight. The next plot examines the leading feature, \texttt{EXT\_SOURCE\_3}, to determine whether its contribution follows a substantively plausible credit-risk pattern.
\begin{center}
\includegraphics[scale=0.4]{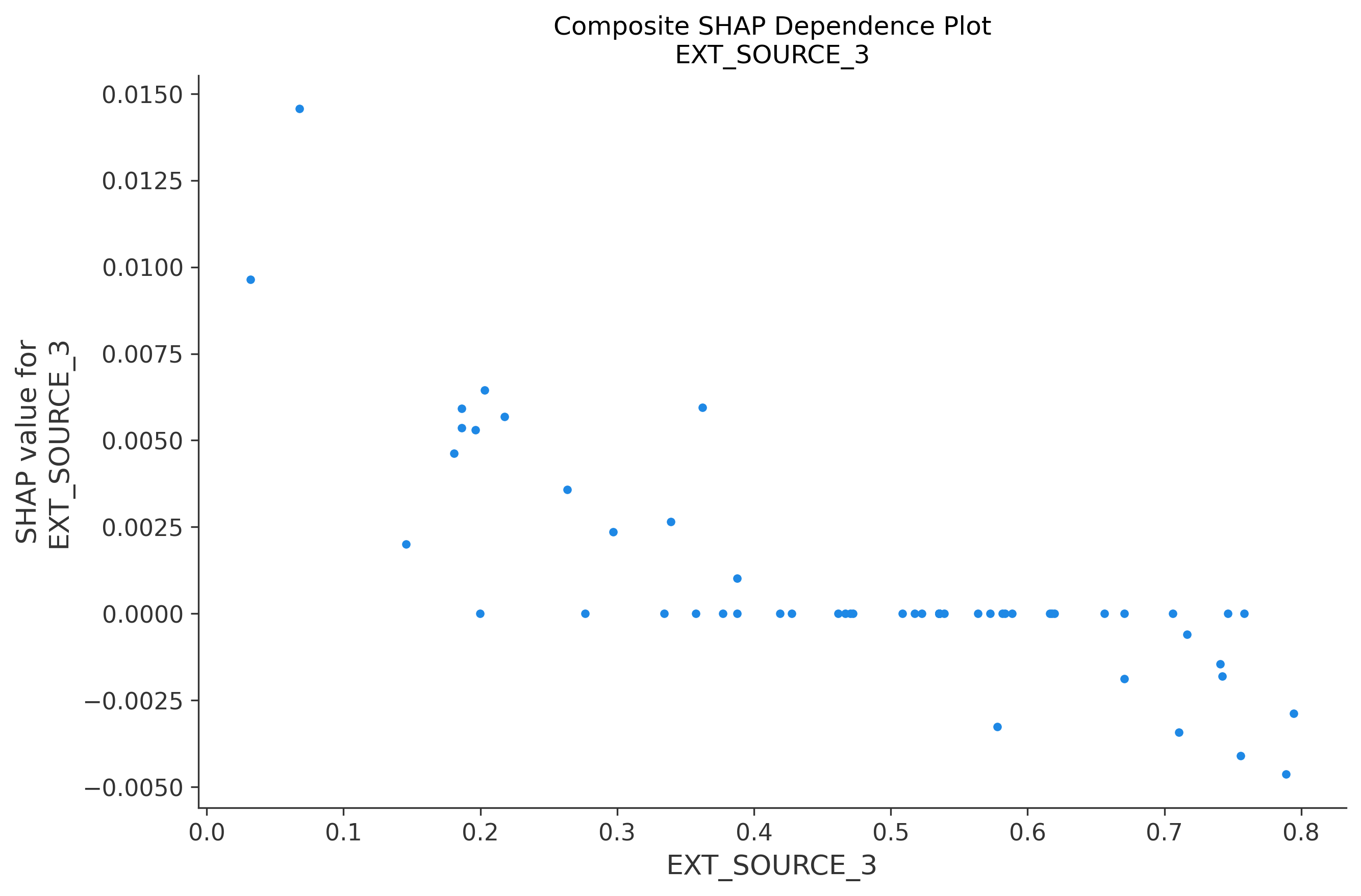}
\captionof{figure}{Dependence Plot for Top Feature of Home Credit Default Risk}
\label{fig:home_credit_shap_top_depend}
\end{center}
Figure \ref{fig:home_credit_shap_top_depend} gives a more detailed view of the leading feature identified in the global SHAP analysis. The dependence plot shows a clear negative relationship between \texttt{EXT\_SOURCE\_3} and its contribution to the composite predictor: low values of \texttt{EXT\_SOURCE\_3} produce positive SHAP values, thereby increasing the predicted risk of default, while larger values tend to produce SHAP values near zero or below zero, thereby reducing the predicted risk.
\begin{center}
\includegraphics[scale=0.275]{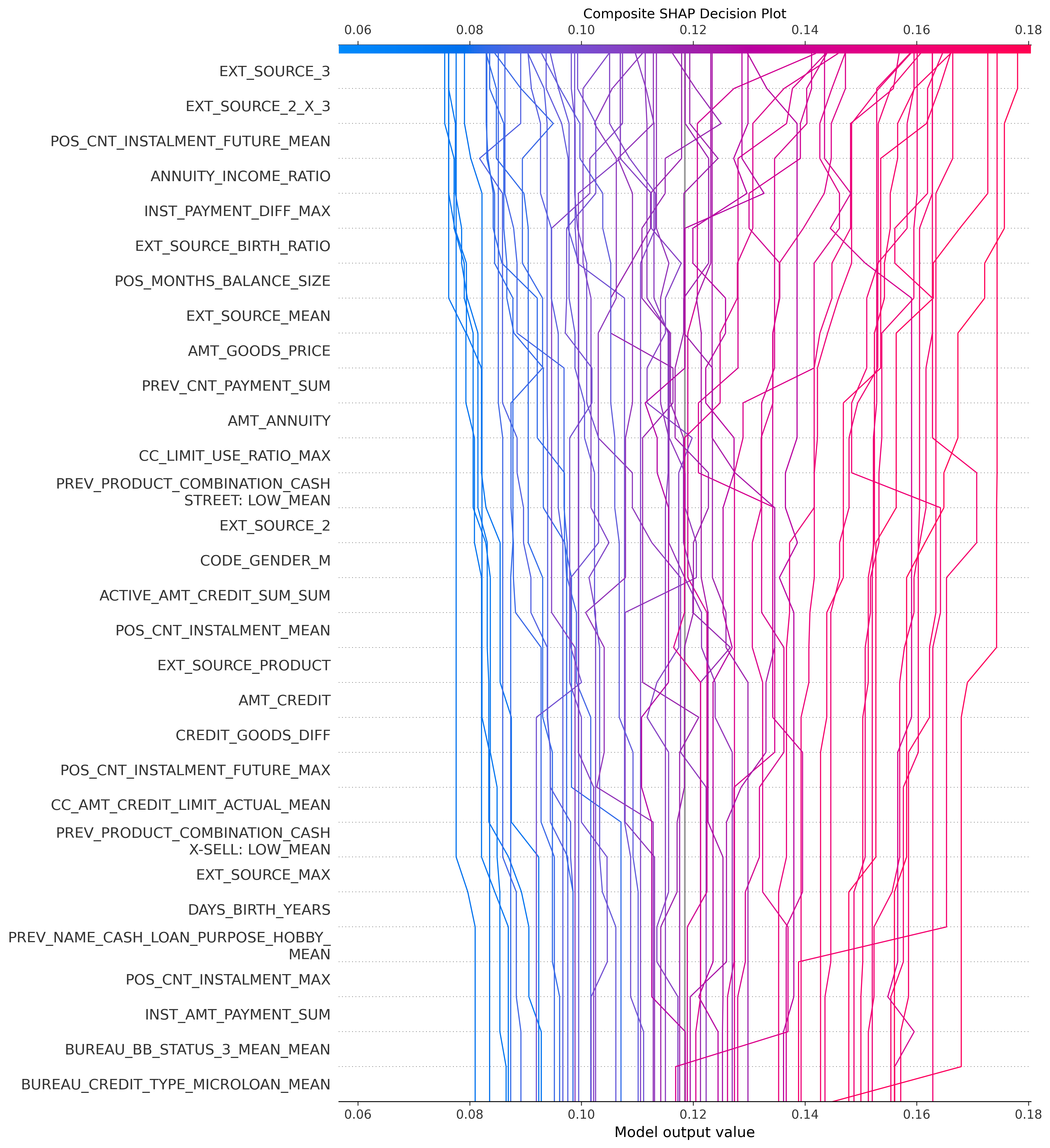}
\captionof{figure}{Decision Plot for Home Credit Default Risk}
\label{fig:home_cred_shap_decision}
\end{center}
This behavior is consistent with the interpretation of the external-source variables as credit-risk strength indicators, where stronger external scores should generally mitigate default risk. The effect also appears nonlinear and threshold-like. Most of the largest positive contributions occur at very low values of \texttt{EXT\_SOURCE\_3}, whereas mid-range and high values are concentrated close to zero with a small number of negative contributions at the upper end of the observed range. Thus, the protected model is most sensitive to \texttt{EXT\_SOURCE\_3} when the external score is weak, and the marginal influence of the feature diminishes once the score reaches a moderate level.\\
The dependence plot therefore complements the global importance ranking by showing not only that \texttt{EXT\_SOURCE\_3} is influential, but also that its influence is directionally sensible and concentrated where external-source scores are weakest. The following decision plot then moves from a single-feature view to the accumulation of multiple feature effects across applicants.\\
Figure \ref{fig:home_cred_shap_decision} extends the preceding global-importance and dependence analyses by showing how the most influential raw features accumulate into individual predictions for the VEIL-protected Home Credit model. Each trajectory represents one applicant, beginning near the model's expected output and then moving left or right as successive feature contributions are added. The separation of the trajectories across the output axis shows that the composite predictor produces meaningfully different risk estimates for different applicants, while the ordering of the features shows that much of this separation is generated by a small set of interpretable credit-risk variables. In particular, \texttt{EXT\_SOURCE\_3} appears at the top of the decision plot and induces substantial early divergence among applicants, which is consistent with its dominant position in Figure \ref{fig:home_cred_shap_global_importance} and the nonlinear dependence pattern shown in Figure \ref{fig:home_credit_shap_top_depend}. Additional variables, including interactions among external-source scores, installment-count features, annuity-to-income ratios, payment-difference measures, and credit-balance features, then refine the applicant-specific prediction paths.
\begin{center}
\includegraphics[scale=0.275]{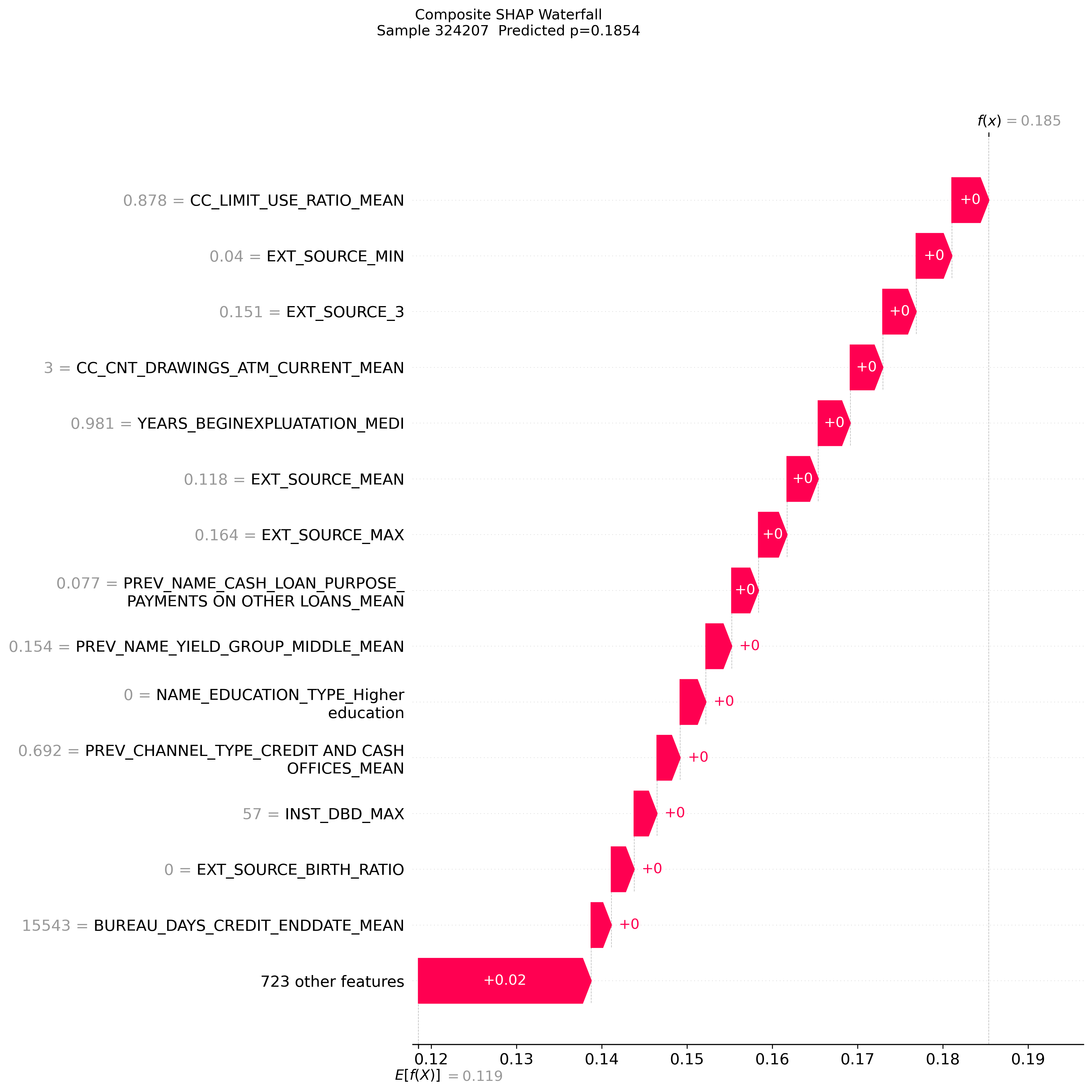}
\captionof{figure}{Waterfall Plot for Home Credit Default Risk}
\label{fig:home_cred_shap_waterfall}
\end{center}
This plot adds a cohort-level perspective to the preceding analyses. Rather than ranking features by average magnitude or isolating the marginal behavior of one feature, it shows how the same set of credit-risk variables produces distinct prediction paths for different applicants. The final plot narrows this view to one applicant-level explanation.\\
Figure \ref{fig:home_cred_shap_waterfall} provides a local explanation for one Home Credit applicant by decomposing the composite VEIL prediction into additive SHAP contributions in the original feature space. The prediction begins from the model's expected output, $\mathbb{E}\left[f\left(X\right)\right]=0.119$, and the displayed feature contributions move the applicant's predicted default probability upward to $f\left(x\right)=0.185$, corresponding to the reported predicted probability of approximately $0.1854$. The plot is therefore an applicant-level counterpart to the preceding global and decision-plot analyses: instead of showing which features matter on average, it shows why this particular applicant was assigned an elevated risk score relative to the model baseline.\\
The increase in predicted risk is not driven by a single opaque latent coordinate, but by a coherent accumulation of raw credit-risk signals. The largest combined contribution comes from the aggregate group of remaining features, while the named positive contributors include a high credit-limit utilization ratio, weak external-source scores, ATM drawing behavior, prior-loan purpose and channel variables, installment-history variables, and bureau-credit timing information. The low values of \texttt{EXT\_SOURCE\_MIN}, \texttt{EXT\_SOURCE\_3}, \texttt{EXT\_SOURCE\_MEAN}, and \texttt{EXT\_SOURCE\_MAX} are especially consistent with the earlier dependence plot, where weaker external-source scores increased predicted default risk. Because the waterfall contains only positive displayed contributions, the explanation indicates that few, if any, of the top-ranked local features materially offset the applicant's risk.\\
Taken together, the four SHAP visualizations provide a unified explanation workflow for the VEIL-protected Home Credit model. The global importance plot identifies the dominant raw credit-risk variables, the dependence plot verifies that the leading feature behaves in a substantively reasonable direction, the decision plot shows how those variables accumulate into applicant-specific prediction paths, and the waterfall plot explains an individual elevated-risk prediction. The operational conclusion is that VEIL protection has not converted the model into an opaque latent-space scoring system. Although the downstream predictor consumes only the protected VEIL representation, the trusted Source Environment can still query the composite function $h\left(\mathbf{x}\right)=g_{\phi}\left(f_{\theta}\left(\mathbf{x}\right)\right)$ and express the resulting explanations in terms of original applicant attributes. This allows an organization to inspect which raw features moved an applicant toward or away from default risk, compare similarly situated applicants, and validate adverse-action reasons or counterfactual scenarios without exporting raw applicant records outside the trusted Source Environment.\\

\subsection{Explanatory Visualizations of the CBIS-DDSM Model}

The CBIS-DDSM experiment provides the medical-imaging counterpart to the Home Credit explainability analysis. As shown in Table \ref{tab:cbis-ddsm_results}, the VEIL-protected SCRAE pipeline achieved a test ROC-AUC of $0.6249$, outperforming both the Raw Data baseline at $0.6076$ and the DP baseline at $0.5700$. Thus, as in the VEIL-protected Home Credit Default Risk experiment, the protected representation did not merely preserve predictive utility; it improved predictive utility relative to both the unprotected and differentially private alternatives. The same pattern also appears in the privacy simulations. For CBIS-DDSM, the VEIL pipeline produced no statistically significant positive reconstruction advantage, no statistically significant positive attribute-inference or property-inference advantage, no statistically significant positive membership-inference advantage in either the label-informed or black-box setting, and no statistically significant positive prediction-leakage advantage. In the model-extraction experiment, the CBIS-DDSM VEIL downstream predictor was also among the least extractable models, with surrogate fidelity of only $R^2=0.1660$ under the fixed-query protocol reported in Table \ref{tab:mod_ex_results}. Taken together, these results show that the VEIL-protected CBIS-DDSM model retained useful diagnostic signal while withstanding the simulated privacy attacks evaluated in Section 10.\\
Healthcare explainability is important for both compliance and clinical adoption. In the United States, the ONC HTI-1 Final Rule creates transparency requirements for predictive decision support interventions in certified health IT, including source-attribute information intended to help clinical users assess fairness, appropriateness, validity, effectiveness, and safety \cite{onc2024hti1, onc2024dsi}. In the European Union, the AI Act treats many AI systems connected to regulated medical devices as high-risk systems and requires high-risk AI systems to be designed with sufficient transparency, instructions for use, output interpretation support, and human oversight \cite{euAIAct2024}. For machine-learning-enabled medical devices, the FDA, Health Canada, and the United Kingdom's MHRA have also identified transparency, user-centered information, and logic or explainability as guiding principles for safe deployment \cite{fda2024mlmdtransparency}. Even where a particular deployment is not directly governed by one of these frameworks, the practical implication is the same: a diagnostic AI system must expose enough information for qualified healthcare professionals to understand, contest, and appropriately rely on the model output.\\
This need for explanation is not merely administrative. Clinical AI systems are used in high-stakes workflows where physicians remain accountable for patient care, and where poorly calibrated trust can be dangerous in either direction. Prior clinical explainability research emphasizes that explanations should support the clinician's actual decision context rather than simply expose model internals, because clinicians need to know when the model is likely to be useful, when it may be unreliable, and what evidence drove the recommendation \cite{tonekaboni2019clinicians, amann2020explainability}. In mammography review, an unexplained malignancy prediction would place the physician in an operational dilemma. To preserve throughput, the physician would have to blindly trust the model's score; to avoid blind reliance, the physician would have to review the image without guidance, which would eliminate much of the efficiency gain that motivated the ML system in the first place. A useful VEIL-protected imaging model must therefore return not only a prediction, but also a clinically meaningful indication of the image regions that most influenced that prediction.\\
Figures \ref{fig:original_cbis-ddsm01257}--\ref{fig:smoothgrad_cbis-ddsm01257} show how this explanation workflow operates for a single CBIS-DDSM mammography image. Figure \ref{fig:original_cbis-ddsm01257} presents the original image as it remains inside the trusted Source Environment. Figure \ref{fig:occlusion_cbis-ddsm01257} reports an occlusion analysis, in which local image patches are systematically masked and the resulting change in the model output identifies regions whose removal most affects the prediction \cite{zeiler2014visualizing}. Figure \ref{fig:int_grad_cbis-ddsm01257} reports Integrated Gradients, which attributes the prediction to pixels by accumulating gradients along a path from a reference image to the observed image \cite{sundararajan2017axiomatic}. Figure \ref{fig:smoothgrad_cbis-ddsm01257} reports SmoothGrad, which reduces visual noise in gradient explanations by averaging saliency over noisy perturbations of the same input \cite{smilkov2017smoothgrad}. Each method therefore asks a different explanatory question: what region changes the prediction when removed, what pixels receive path-integrated attribution, and what saliency pattern remains stable under perturbation.
\begin{center}
\includegraphics[scale=0.55]{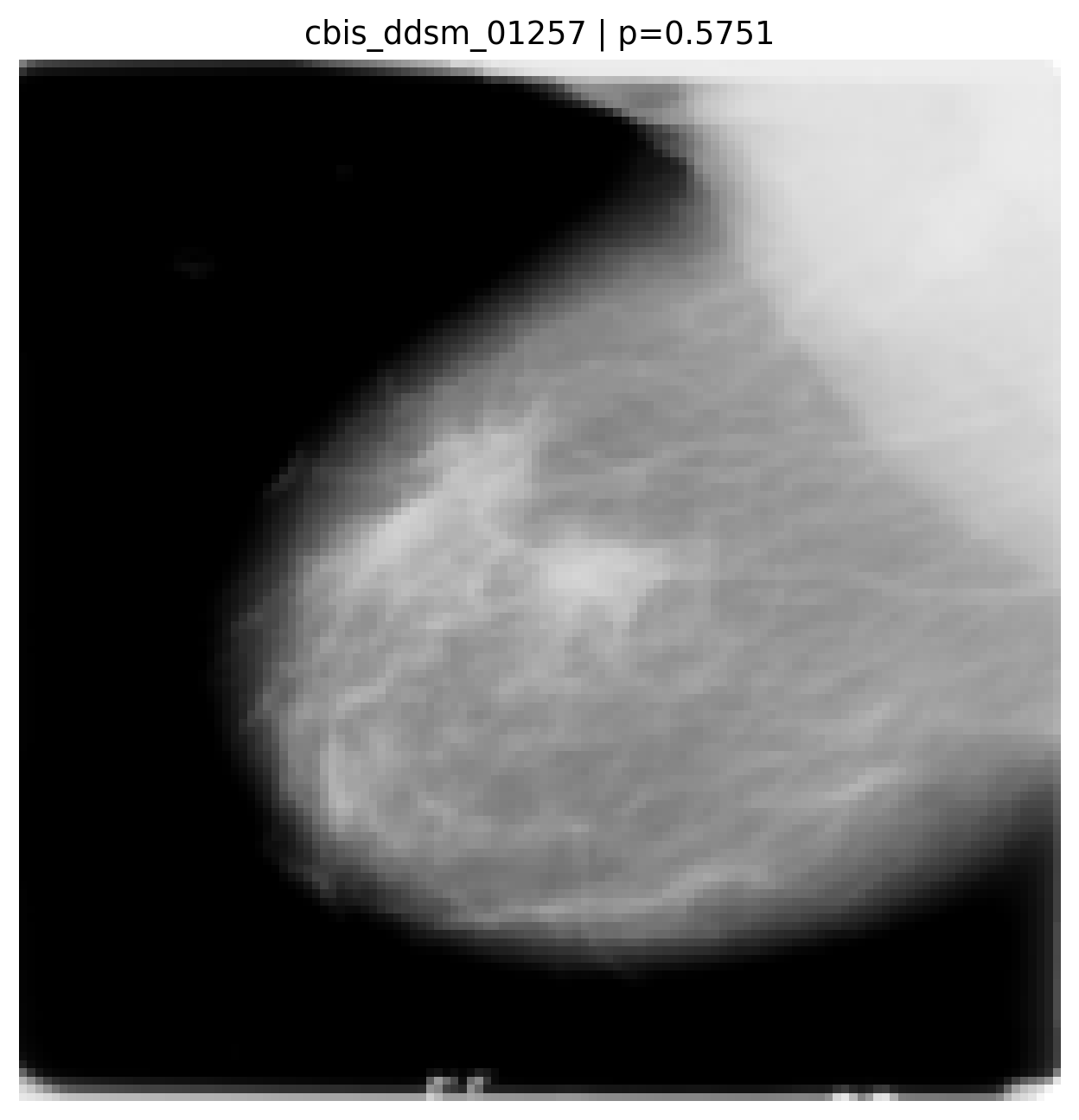}
\captionof{figure}{Original X-Ray Image CBIS-DDSM 01257}
\label{fig:original_cbis-ddsm01257}
\end{center}
\begin{center}
\includegraphics[scale=0.55]{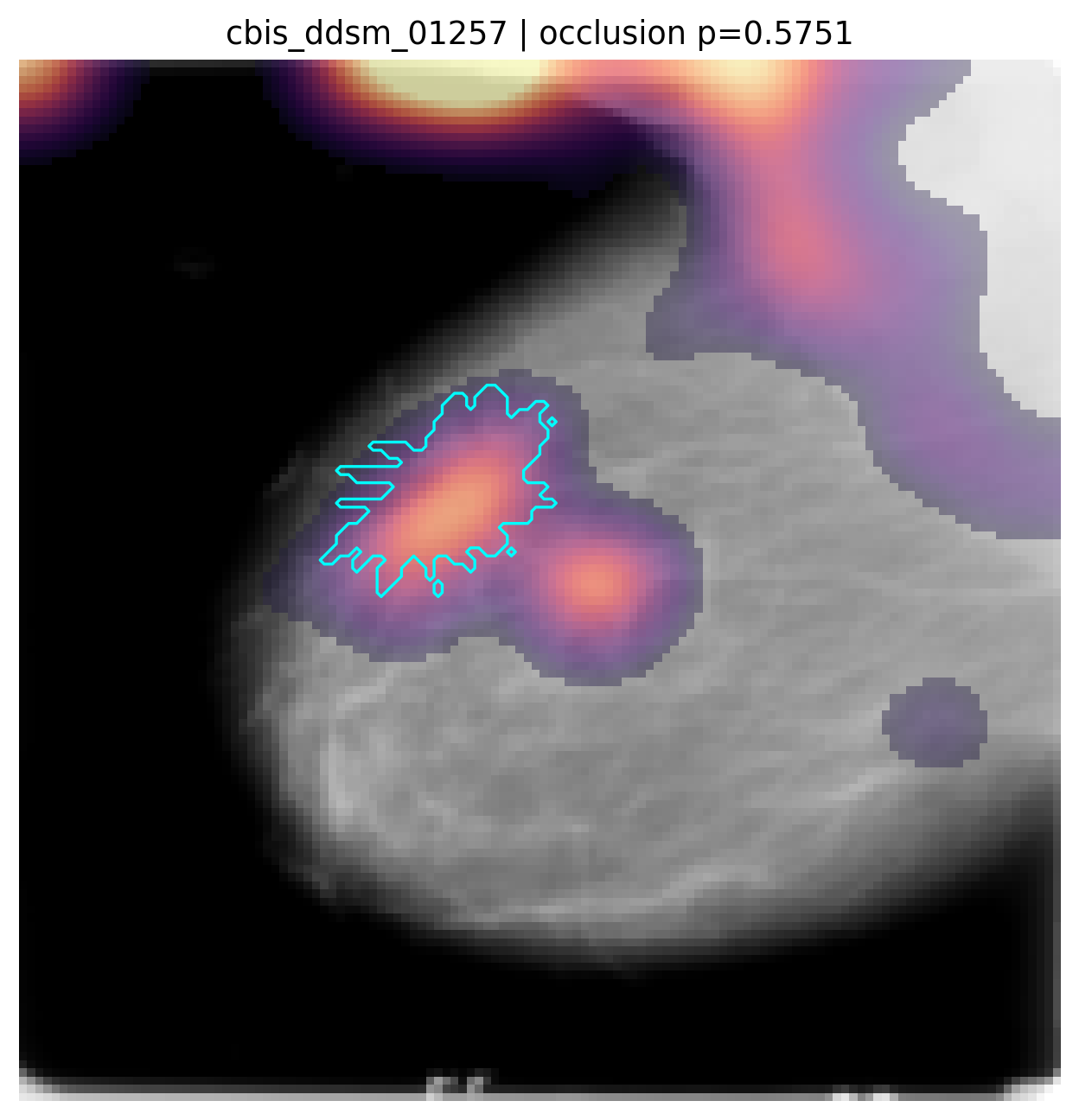}
\captionof{figure}{X-Ray Image CBIS-DDSM 01257 with Occlusion Overlay}
\label{fig:occlusion_cbis-ddsm01257}
\end{center}
\begin{center}
\includegraphics[scale=0.55]{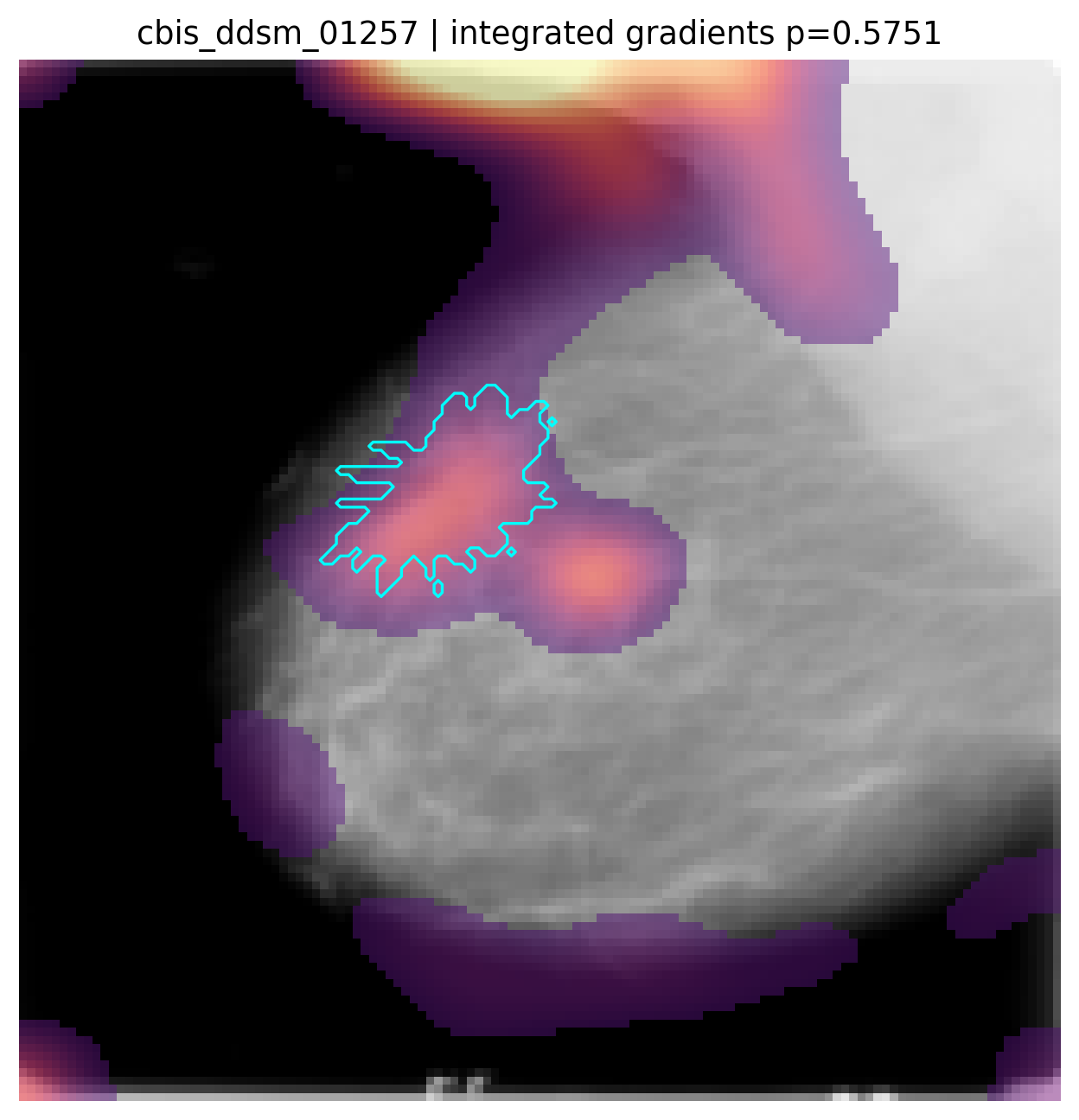}
\captionof{figure}{X-Ray Image CBIS-DDSM 01257 with Integrated Gradients Overlay}
\label{fig:int_grad_cbis-ddsm01257}
\end{center}
\begin{center}
\includegraphics[scale=0.55]{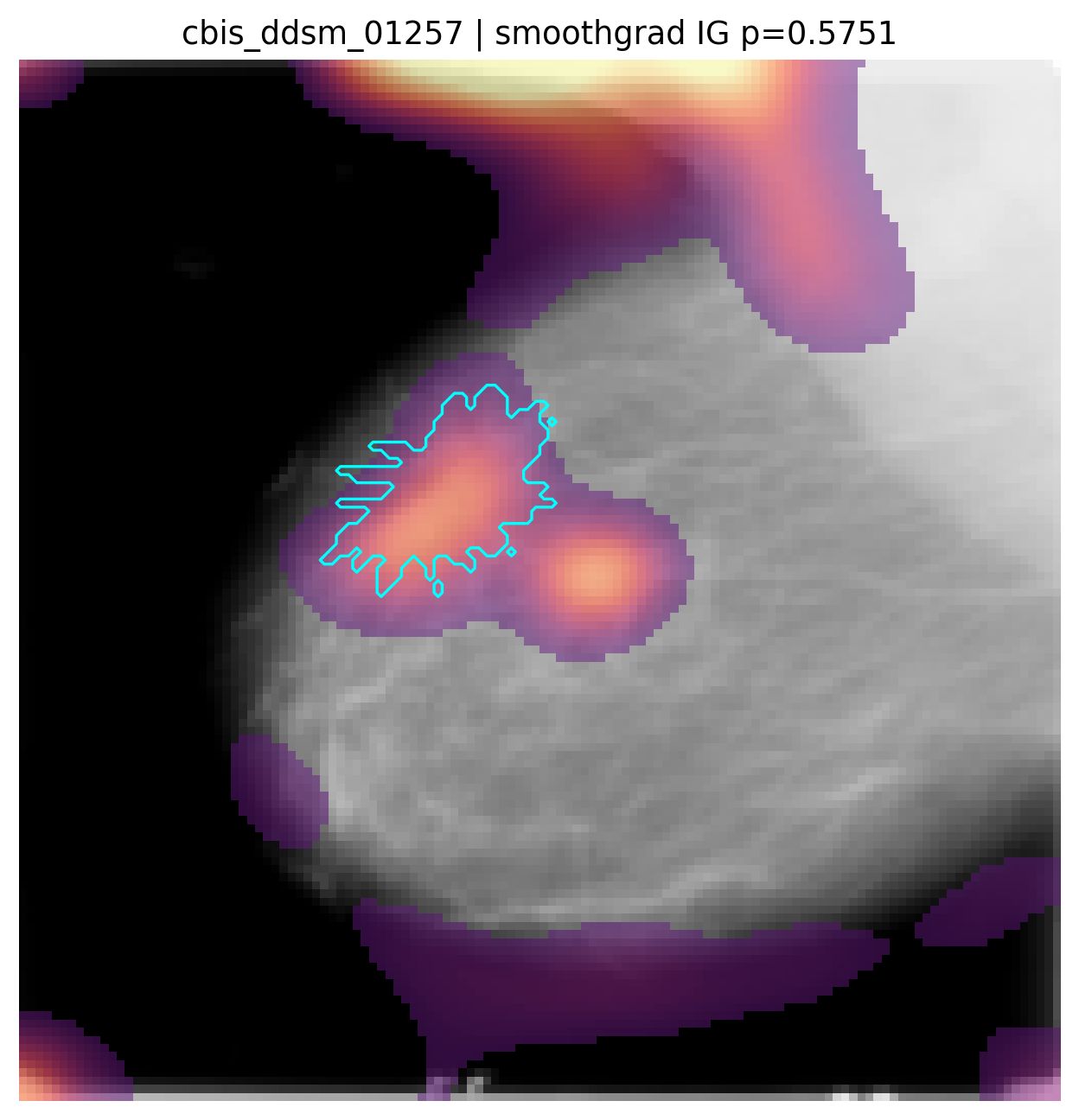}
\captionof{figure}{X-Ray Image CBIS-DDSM 01257 with SmoothGrad Overlay}
\label{fig:smoothgrad_cbis-ddsm01257}
\end{center}
The close agreement among the occlusion, Integrated Gradients, and SmoothGrad overlays is the key interpretive result. Because these methods rely on different perturbation and gradient mechanisms, agreement among them provides stronger evidence than any single heat map would provide alone. In this example, the overlays guide the physician back to a common region of diagnostic relevance on the original X-ray image, rather than dispersing attention across unrelated background artifacts. This permits a high-confidence explanatory analysis in the original image domain while preserving the VEIL trust boundary. The downstream classifier never receives the raw mammography image; it receives only the protected latent representation $\Psi=f_{\theta}\left(\mathbf{x}\right)$. The raw image, the encoder, and the explainer queries remain inside the trusted and tightly controlled Source Environment, where the composite model $h\left(\mathbf{x}\right)=g_{\phi}\left(f_{\theta}\left(\mathbf{x}\right)\right)$ can be queried for clinically meaningful visual explanations without exposing raw patient data outside the protected environment.\\

\section{Conclusions}

ML and AI have become central to enterprise and public-sector decision-making, but their continued adoption depends on whether organizations can use sensitive data without exposing it to downstream systems, cloud platforms, model-serving infrastructure, or logs. This paper has developed Informationally Compressive Anonymization (ICA) and the VEIL architecture as a response to that problem. The central claim is not that sensitive data can be made safe after export through a stronger wrapper around the same risky pipeline. Rather, the pipeline itself must be redesigned so that raw data, the encoder, and all operations requiring raw data remain inside a trusted Source Environment, while external Training and Inference Environments receive only task-aligned, non-invertible latent representations.\\
The technical development began with autoencoders and the SCRAE framework, but the method proposed here departs from conventional autoencoder privacy systems in a crucial way: reconstruction is not the objective of the deployed protection mechanism. The encoder is trained as a supervised, multi-objective, one-way representation learner. For classification, the representation objective organizes latent vectors into task-useful class geometry; for regression, the Graph-Laplacian loss replaces class clustering with a smooth target-aware geometry over continuous labels. The framework is further made practical through $k$-NN sparsification of the Graph-Laplacian loss, through diagnostics for monotonicity, calibration, and latent collapse, and through InfoNCE-style contrastive alternatives when the default Center or Graph-Laplacian losses are insufficient. The resulting methodology is therefore not limited to a single demonstration model; it is a family of supervised representation objectives for producing compressed, task-aligned, non-invertible model inputs.\\
The empirical utility results support this design across classification, regression, tabular credit-risk data, medical imaging, and high-dimensional sparse text-derived features. On MNIST, the SCRAE pipeline compressed each image from 784 pixels to a 2-dimensional latent vector, reducing input size by approximately 98.98\%, while improving test accuracy from $92.42\%$ for the raw baseline to $98.61\%$. On Ames Housing, the regression-adapted SCRAE compressed the engineered design space from 282 dimensions to 14 dimensions, reducing input size by approximately 95.04\%, while improving test $R^2$ from $0.9315$ to $0.9421$ and improving RMSLE from $0.1279$ to $0.1199$. The YearPredictionMSD experiments showed that the sparsified Graph-Laplacian variants remained competitive with, and in several metrics exceeded, the raw baseline while avoiding the quadratic cost of the dense loss. At $B=2,048$, the dense Graph-Laplacian step required $5.55$ seconds and approximately $10.0$ GB of peak allocated memory, whereas sparsified variants required only $18.4$--$42.8$ ms and $47$--$135$ MB, yielding speedups between $129.7\times$ and $301.8\times$. On Fashion-MNIST and E2006, the contrastive objectives preserved the same design pattern: the protected SCRAE pipeline outperformed the raw and DP baselines in Fashion-MNIST and matched the raw E2006 baseline while reducing the 150,360-dimensional representation to 64 dimensions, a storage and transmission reduction of approximately 99.96\%.\\
The later benchmarking and attack-simulation sections extend these results beyond illustrative datasets. In the Home Credit, Credit Card Default, and CBIS-DDSM studies, the VEIL-protected SCRAE pipelines matched or exceeded raw-baseline ROC-AUC while outperforming the corresponding DP and dense-autoencoder pipelines. These results are important because they show that ICA is not merely a privacy argument in isolation. In the evaluated pipelines, the protected representation remains useful for the downstream supervised task, and in several cases it improves generalization by discarding variation that is not useful for prediction. This is the practical advantage of informational compression when the bottleneck is aligned with the supervised objective rather than with input reconstruction.\\
The theoretical guarantees developed in this paper explain why these compressed representations should not be treated as disguised copies of the original records. Under the stated assumptions---a continuous encoder $f_{\theta}:\mathcal{X}\rightarrow\mathcal{Z}$, genuine input variability in dimension $D$, latent dimension $E<D$, no deployed decoder, and no external access to arbitrary encoder queries---the encoder is non-injective and therefore non-invertible. From the perspective of a Source Environment attacker who knows the encoder, the preimage of a latent vector remains underdetermined. From the perspective of a Training or Inference Environment attacker who does not know the encoder, the input dimension, or the input domain, the inverse problem expands into an unbounded family of possible input spaces and encoders. The privacy claim is therefore structural rather than computational: it does not depend on a cryptographic hardness assumption, a privacy budget, or an adversary being unable to afford enough computation. It depends on the fact that the information needed to reconstruct the original record is not present in the exported representation.\\
The empirical attack simulations are consistent with that claim while also clarifying its boundary. Across eight model--dataset pairs, DP pipelines produced statistically significant positive reconstruction advantage in four cases, while VEIL produced no statistically significant positive reconstruction advantage in any case. VEIL also produced zero statistically significant positive label-informed membership-inference results and zero statistically significant positive black-box membership-inference results. Attribute inference, property inference, prediction leakage, and model extraction require a more nuanced interpretation. VEIL reduced significant attribute-inference and property-inference outcomes relative to DP in the evaluated protocols, but it did not eliminate every task-correlated statistical relationship. Prediction leakage remained possible when the downstream output itself was correlated with a sensitive attribute, and model extraction could approximate the exposed downstream predictor $g_{\phi}:\mathcal{Z}\rightarrow\mathcal{Y}$ under repeated query access. These are not failures of non-invertibility; they are reminders that ICA protects raw inputs before they leave the Source Environment, while disclosed predictions, joinable artifacts, logs, and high-volume query interfaces must still be governed as part of the complete system.\\
For this reason, the conclusion of the security analysis is operational as well as mathematical. A correct VEIL deployment must keep the encoder and its parameters inside the Source Environment, prevent arbitrary external encoder queries, avoid exporting latent vectors with identifiers or join handles, govern prediction logs, minimize unnecessary output precision, rate-limit and monitor inference APIs, and test for task-correlated sensitive attributes during validation. Under those controls, the attacker is left with a constrained downstream interface and non-invertible, task-aligned latent vectors rather than raw, reconstructable records. This is the practical meaning of privacy-by-design in the VEIL architecture.\\
The paper also shows that privacy does not have to come at the expense of transparency. Because predictions are returned to the Source Environment, explainability can be performed on the composite model $h\left(\mathbf{x}\right)=g_{\phi}\left(f_{\theta}\left(\mathbf{x}\right)\right)$ while raw data remain protected. The Home Credit SHAP analysis demonstrates that a VEIL-protected credit-risk model can still produce global feature importance, dependence plots, decision plots, waterfall explanations, and applicant-level counterfactual reasoning in the original feature space. The CBIS-DDSM analysis shows the same principle for medical imaging: occlusion, Integrated Gradients, and SmoothGrad can identify diagnostically relevant regions of the original image inside the trusted environment, even though the downstream classifier receives only the protected latent representation. Thus, VEIL does not force a choice between privacy and explainability; it relocates explanation to the environment where raw data are already authorized to exist.\\
Taken together, the results of this paper support ICA and the VEIL architecture as a practical foundation for privacy-preserving supervised machine learning. Compared with DP, VEIL avoids gradient clipping, privacy-budget management, and the associated accuracy trade-offs observed in the evaluated pipelines. Compared with HE, VEIL avoids ciphertext expansion, specialized encrypted-computation stacks, and high-latency inference. Compared with unsupervised autoencoder encodings, VEIL aligns compression with the supervised task rather than with reconstruction. The result is an architecture that supports cloud-scale training, low-latency inference, multi-region and multi-organization deployment patterns, regulatory data-minimization goals, and source-domain explainability, while materially reducing the risk that downstream compromise exposes raw sensitive data.\\
ICA and VEIL should therefore be understood not as a single model architecture, but as a system-level design pattern: learn the smallest task-sufficient representation inside the trust boundary, export only that representation, train and serve downstream models only on that representation, and bring predictions back to the source for explanation and governance. Under the assumptions and operational controls described in this paper, this pattern replaces the traditional trade-off between privacy and performance with a more durable alternative: sensitive data remain where they belong, useful predictive signal moves safely, and the exported artifacts are structurally incapable of serving as reconstructable substitutes for the original records.

\end{multicols}

\pagebreak

\appendix
\section{Additional Figures}

\setcounter{figure}{0}
\renewcommand{\thefigure}{A\arabic{figure}}

\begin{figure}[htbp]
    \centering
    \includegraphics[width=\linewidth]{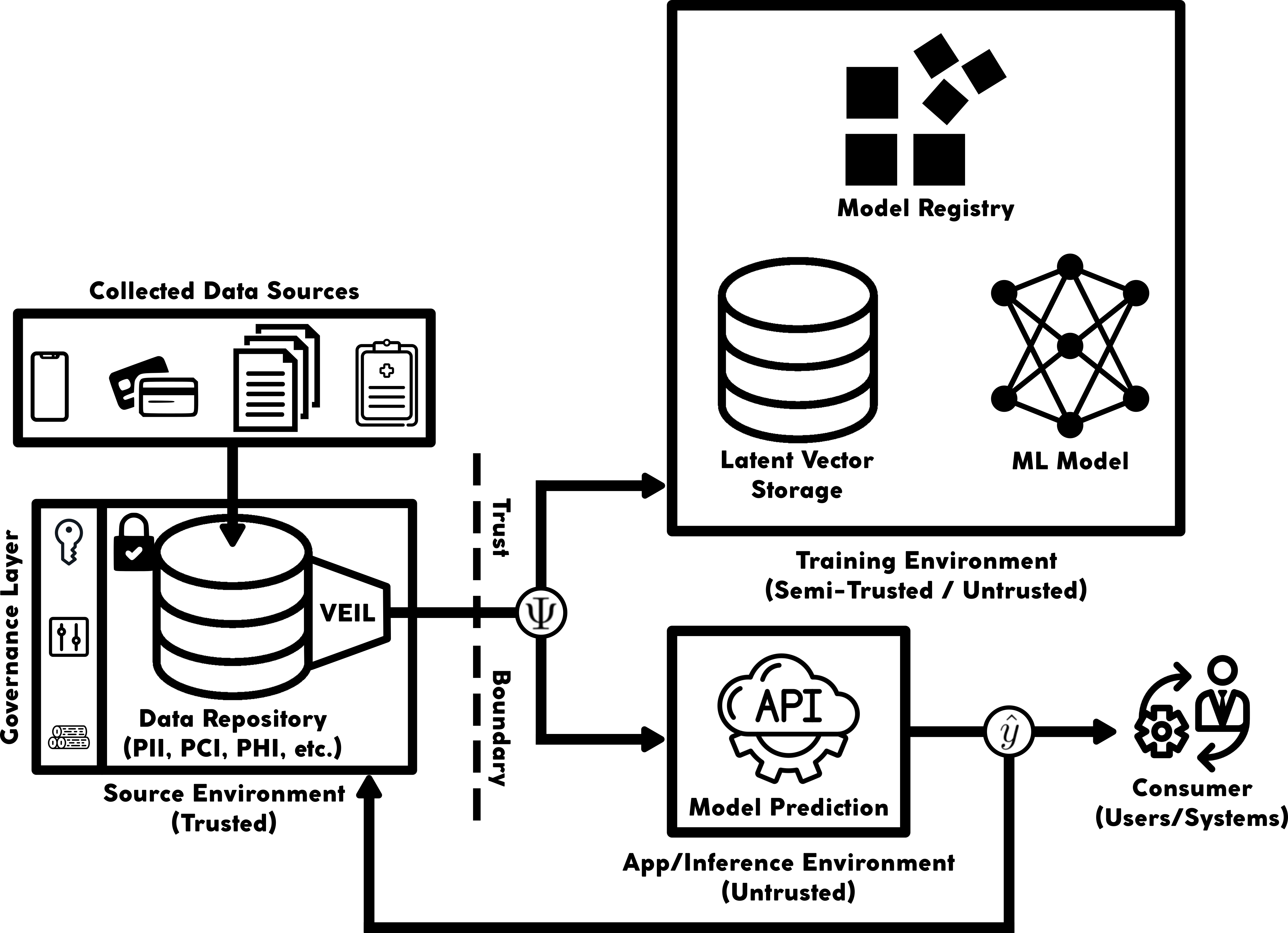}
    \caption{Complete VEIL Architecture}
    \label{fig:VEIL_architecture_full}
\end{figure}

\begin{figure}[htbp]
    \centering
    \includegraphics[width=\linewidth]{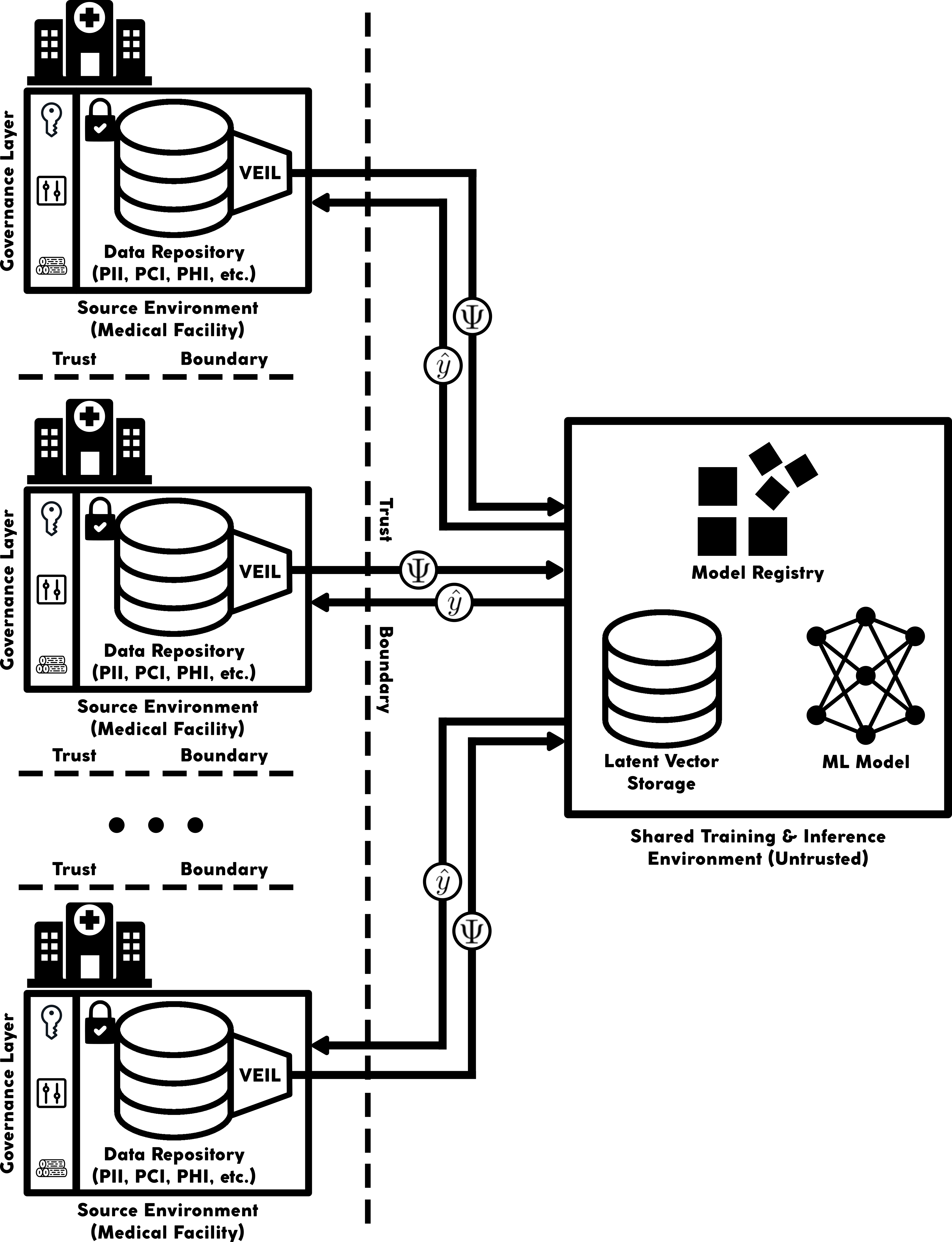}
    \caption{VEIL-enabled Shared Healthcare AI Infrastructure}
    \label{fig:shared_hc_infrastructure}
\end{figure}

\pagebreak

\printbibliography

\end{document}